\documentclass[11pt]{article}

\usepackage{times}
\usepackage{amsmath, amssymb, amsfonts,  amsthm}
\usepackage{bm, braket}
\usepackage{tikz,subfigure}
\usepackage{pgfplots}
\usepackage{adjustbox}
\usepackage{natbib}
\setcitestyle{numbers}
\usepackage[a4paper, total={7in, 9in}]{geometry}
\usepackage{hyperref}
\usepackage[affil-it]{authblk}

\title{Random Features Model with General Convex Regularization: \\A Fine Grained Analysis with Precise Asymptotic Learning Curves}

\author[1]{David Bosch\footnote{davidbos@chalmers.se}}
\author[1]{Ashkan Panahi\footnote{ashkan.panahi@chalmers.se}}
\author[2]{Ayca Özcelikkale \footnote{ayca.ozcelikkale@angstrom.uu.se}}
\author[1]{Devdatt Dubhashi\footnote{dubhashi@chalmers.se}}

\affil[1]{Department of Data Science and AI, Computer Science and Engineering, Chalmers University of Technology}
\affil[2]{Signals and Systems, Department of Electrical Engineering, Uppsala University}

\pgfplotsset{compat=1.17}
\begin{document}

\newcommand{\bDelta}{\bm{\Delta}}
\newcommand{\bLambda}{\bm{\Lambda}}
\newcommand{\bGamma}{\bm{\Gamma}}
\newcommand{\bSigma}{\bm{\Sigma}}
\newcommand{\bOmega}{\bm{\Omega}}
\newcommand{\bPsi}{\bm{\Psi}}

\newcommand{\balpha}{\bm{\alpha}}
\newcommand{\bdelta}{\bm{\delta}}
\newcommand{\bomega}{\bm{\omega}}
\newcommand{\bgamma}{\bm{\gamma}}
\newcommand{\bepsilon}{\bm{\epsilon}}
\newcommand{\blambda}{\bm{\lambda}}
\newcommand{\btheta}{\bm{\theta}}
\newcommand{\bphi}{\bm{\phi}}
\newcommand{\bpsi}{\bm{\psi}}
\newcommand{\bmeta}{\bm{\eta}}
\newcommand{\bzeta}{\bm{\zeta}}
\newcommand{\bmu}{\bm{\mu}}
\newcommand{\bnu}{\bm{\nu}}
\newcommand{\bpi}{\bm{\pi}}
\newcommand{\bsigma}{\bm{\sigma}}

\newcommand{\bargamma}{\bar{\gamma}}
\newcommand{\bartheta}{\bar{\theta}}

\newcommand{\tilDelta}{\tilde{\Delta}}
\newcommand{\tlDelta}{\tilde{\Delta}}

\newcommand{\tlepsilon}{\tilde{\epsilon}}
\newcommand{\tltheta}{\tilde{\theta}}
\newcommand{\tlgamma}{\tilde{\gamma}}

\newcommand{\tlblambda}{\tilde{\blambda}}

\newcommand{\barblambda}{\bar{\blambda}}

\newcommand{\bA}{\mathbf{A}}
\newcommand{\bC}{\mathbf{C}}
\newcommand{\bD}{\mathbf{D}}
\newcommand{\bE}{\mathbf{E}}
\newcommand{\bG}{\mathbf{G}}
\newcommand{\bH}{\mathbf{H}}
\newcommand{\bI}{\mathbf{I}}
\newcommand{\bJ}{\mathbf{J}}
\newcommand{\bL}{\mathbf{L}}
\newcommand{\bM}{\mathbf{M}}
\newcommand{\bN}{\mathbf{N}}
\newcommand{\bP}{\mathbf{P}}
\newcommand{\bQ}{\mathbf{Q}}
\newcommand{\bR}{\mathbf{R}}
\newcommand{\bS}{\mathbf{S}}
\newcommand{\bT}{\mathbf{T}}
\newcommand{\bU}{\mathbf{U}}
\newcommand{\bW}{\mathbf{W}}
\newcommand{\bX}{\mathbf{X}}
\newcommand{\bY}{\mathbf{Y}}
\newcommand{\bZ}{\mathbf{Z}}

\newcommand{\ba}{\mathbf{a}}
\newcommand{\bb}{\mathbf{b}}
\newcommand{\bc}{\mathbf{c}}
\newcommand{\bd}{\mathbf{d}}
\newcommand{\be}{\mathbf{e}}
\newcommand{\mbf}{\mathbf{f}}
\newcommand{\bg}{\mathbf{g}}
\newcommand{\bh}{\mathbf{h}}
\newcommand{\bl}{\mathbf{l}}
\newcommand{\bn}{\mathbf{n}}
\newcommand{\bp}{\mathbf{p}}
\newcommand{\bq}{\mathbf{q}}
\newcommand{\br}{\mathbf{r}}
\newcommand{\bs}{\mathbf{s}}
\newcommand{\bu}{\mathbf{u}}
\newcommand{\bv}{\mathbf{v}}
\newcommand{\bw}{\mathbf{w}}
\newcommand{\bx}{\mathbf{x}}
\newcommand{\by}{\mathbf{y}}
\newcommand{\bz}{\mathbf{z}}

\newcommand{\hbeta}{\hat{\beta}}
\newcommand{\htheta}{\hat{\theta}}
\newcommand{\hsigma}{\hat{\sigma}}
\newcommand{\hmu}{\hat{\mu}}

\newcommand{\bbeta}{\bm{\beta}}

\newcommand{\hp}{\hat{p}}
\newcommand{\hr}{\hat{r}}
\newcommand{\hs}{\hat{s}}
\newcommand{\hw}{\hat{w}}
\newcommand{\hx}{\hat{x}}

\newcommand{\hN}{\hat{N}}

\newcommand{\hbSigma}{\hat{\bm{\Sigma}}}

\newcommand{\hba}{\hat{\mathbf{a}}}
\newcommand{\hbs}{\hat{\mathbf{s}}}
\newcommand{\hbx}{\hat{\mathbf{x}}}
\newcommand{\hbv}{\hat{\mathbf{v}}}
\newcommand{\hbw}{\hat{\mathbf{w}}}

\newcommand{\hbW}{\hat{\mathbf{W}}}

\newcommand{\dif}{\text{d}}

\newcommand{\bbC}{\mathbb{C}}
\newcommand{\bbE}{\mathbb{E}}
\newcommand{\bbR}{\mathbb{R}}
\newcommand{\bbN}{\mathbb{N}}
\newcommand{\bbZ}{\mathbb{Z}}

\newcommand{\calA}{\mathcal{A}}
\newcommand{\calB}{\mathcal{B}}
\newcommand{\calC}{\mathcal{C}}
\newcommand{\calD}{\mathcal{D}}
\newcommand{\calE}{\mathcal{E}}
\newcommand{\calF}{\mathcal{F}}
\newcommand{\calG}{\mathcal{G}}
\newcommand{\calH}{\mathcal{H}}
\newcommand{\calL}{\mathcal{L}}
\newcommand{\calN}{\mathcal{N}}
\newcommand{\calM}{\mathcal{M}}
\newcommand{\calP}{\mathcal{P}}
\newcommand{\calS}{\mathcal{S}}
\newcommand{\calT}{\mathcal{T}}
\newcommand{\calV}{\mathcal{V}}
\newcommand{\calW}{\mathcal{W}}
\newcommand{\calX}{\mathcal{X}}
\newcommand{\calY}{\mathcal{Y}}

\newcommand{\calhL}{\mathcal{\hat{L}}}

\newcommand{\tlA}{\tilde{A}}
\newcommand{\tlC}{\tilde{C}}
\newcommand{\tlD}{\tilde{D}}
\newcommand{\tlP}{\tilde{P}}

\newcommand{\tlf}{\tilde{f}}
\newcommand{\tlv}{\tilde{v}}
\newcommand{\tls}{\tilde{s}}
\newcommand{\tlw}{\tilde{w}}
\newcommand{\tlx}{\tilde{x}}

\newcommand{\barb}{\bar{b}}
\newcommand{\barm}{\bar{m}}
\newcommand{\barn}{\bar{n}}
\newcommand{\barr}{\bar{r}}
\newcommand{\bary}{\bar{y}}

\newcommand{\barC}{\bar{C}}
\newcommand{\barD}{\bar{D}}
\newcommand{\barH}{\bar{H}}
\newcommand{\barK}{\bar{K}}
\newcommand{\barL}{\bar{L}}
\newcommand{\barP}{\bar{P}}
\newcommand{\barW}{\bar{W}}

\newcommand{\barba}{\bar{\ba}}
\newcommand{\barbe}{\bar{\be}}
\newcommand{\barbg}{\bar{\bg}}
\newcommand{\barbh}{\bar{\bh}}
\newcommand{\barbx}{\bar{\bx}}
\newcommand{\barby}{\bar{\by}}
\newcommand{\barbz}{\bar{\bz}}

\newcommand{\barbA}{\bar{\bA}}

\newcommand{\tlbA}{\tilde{\bA}}
\newcommand{\tlbE}{\tilde{\bE}}
\newcommand{\tlbG}{\tilde{\bG}}
\newcommand{\tlbR}{\tilde{\bR}}
\newcommand{\tlbW}{\tilde{\bW}}
\newcommand{\tlbX}{\tilde{\bX}}
\newcommand{\tlbY}{\tilde{\bY}}

\newcommand{\tlbe}{\tilde{\be}}
\newcommand{\tlbf}{\tilde{\mbf}}
\newcommand{\tlbg}{\tilde{\bg}}
\newcommand{\tlbv}{\tilde{\bv}}
\newcommand{\tlbw}{\tilde{\bw}}
\newcommand{\tlbx}{\tilde{\bx}}
\newcommand{\tlby}{\tilde{\by}}

\newcommand{\tc}{\text{c}}
\newcommand{\td}{{\text{d}}}
\newcommand{\ter}{{\text{r}}}
\newcommand{\ts}{{\text{s}}}
\newcommand{\tw}{{\text{w}}}

\newcommand{\bzero}{\mathbf{0}}
\newcommand{\bone}{\mathbf{1}}

\newcommand{\suml}{\sum\limits}
\newcommand{\minl}{\min\limits}
\newcommand{\maxl}{\max\limits}
\newcommand{\infl}{\inf\limits}
\newcommand{\supl}{\sup\limits}
\newcommand{\liml}{\lim\limits}
\newcommand{\intl}{\int\limits}
\newcommand{\ointl}{\oint\limits}
\newcommand{\bigcupl}{\bigcup\limits}
\newcommand{\bigcapl}{\bigcap\limits}

\newcommand{\opconv}{\text{conv}}

\newcommand{\eref}[1]{(\ref{#1})}

\newcommand{\sinc}{\text{sinc}}
\newcommand{\tr}{\text{Tr}}
\newcommand{\var}{\text{Var}}
\newcommand{\cov}{\text{Cov}}
\newcommand{\tth}{\text{th}}
\newcommand{\proj}{\text{proj}}

\newcommand{\nwl}{\nonumber\\}

\newenvironment{vect}{\left[\begin{array}{c}}{\end{array}\right]}
\newtheorem{theorem}{Theorem}
\newtheorem{proposition}{Proposition}
\newtheorem{lemma}{Lemma}
\newtheorem{assum}{Assumption}
\newtheorem{cor}{Corollary}

\theoremstyle{definition}
\newtheorem{definition}{Definition}
\newtheorem{assump}{Assumption}
\theoremstyle{remark}
\newtheorem{remark}{Remark}

\newcommand{\polylog}{\mathrm{polylog}\ }
\newcommand{\bPhi}{\bm{\Phi}}
\newcommand{\bV}{\bm{V}}

\newcommand{\diag}[1]{\mathrm{diag}\lbrace #1 \rbrace}
\newcommand{\bxi}{\bm{\xi}}
\newcommand{\bTheta}{\bm{\Theta}}
\newcommand{\uk}{^{(k)}}
\newcommand{\ukT}{^{(k)T}}
\newcommand{\norm}[1]{\left\|#1\right\|}
\newcommand{\isp}{\frac{1}{\sqrt{p}}}
\newcommand{\bXi}{\bm{\Xi}} 
\newcommand{\bPi}{\bm{\Pi}}
\newcommand{\bUpsilon}{\bm{\Upsilon}}
\newcommand{\bvarphi}{\bm{\varphi}}
\newcommand{\barphi}{\bar{\phi}}
\newcommand{\barbR}{\bar{\bR}}
\newcommand{\bt}{\bm{t}}
\newcommand{\bK}{\bm{K}}
\newcommand{\op}{\mathrm{op}}
\newcommand{\hbd}{\hat{\bd}}
\newcommand{\whp}{\quad\mathrm{w.h.p}}
\newcommand{\poly}{\mathrm{poly}}

\maketitle

\begin{abstract}
We compute precise asymptotic expressions for the learning curves of least squares random feature (RF) models with either a separable strongly convex regularization or the $\ell_1$ regularization. We propose a novel multi-level application of the convex Gaussian min max theorem (CGMT) to overcome the traditional difficulty of finding computable expressions for random features models with correlated data. Our result takes the form of a computable 4-dimensional scalar optimization. In contrast to previous results,  our approach does not require solving an often intractable proximal operator, which scales with the number of model parameters. Furthermore, we extend the universality results for the training and generalization errors for RF models to $\ell_1$ regularization. In particular, we demonstrate that under mild conditions,  random feature models with elastic net or $\ell_1$ regularization are asymptotically equivalent to a surrogate Gaussian model with the same first and second moments. We numerically demonstrate the predictive capacity of our results, and show experimentally that the predicted test error is accurate even in the non-asymptotic regime.
\end{abstract}

\section{Introduction}
It has been recently understood that classical statistical theory requires revisiting to describe the behavior of overparameterized models \citep{zhang2021understanding, belkin2019reconciling}. Since then, studying the asymptotic regime of a machine learning (ML) model, in which the number of data points and model parameters grow infinite at a constant ratio, has become a popular method of analysis \citep{belkin2020two, hastie2019surprises, Bartlett30063, tsigler2020benign}. The asymptotic analysis of regularized Random Feature (RF) models \citep{rahimi2007random} has been of particular interest as they can capture a large range of other interesting models \citep{mei2019generalization, goldt2019modeling, d2020double, LuPreciseRandom}. Despite remarkable progress in the analysis of RF models, existing asymptotic results 
are not directly computable for the majority of regularization functions, and in this generic scenario, precise asymptotic learning curves are still lacking. In this paper, we address this limitation and provide a novel technique that provides computable, exact asymptotic learning curves under a large family of separable, strongly convex regularization, as well as the $\ell_1$ regularization (also known as LASSO).

Similar to many recent papers, we make use of the convex Gaussian Min Max theorem (CGMT) \citep{thrampoulidis2015gaussian, hastie2019surprises, montanari2019generalization, LuPreciseRandom, goldt2020gaussian}, where there are generally two steps.  The RFs are non-Gaussian due to nonlinear activation functions, but it is shown that they can be  equivalently replaced by a surrogate Gaussian model with matching first two statistical moments \citep{panahi2017universal,oymak2018universality, HuUniversalityLaws}. Establishing this equivalence between the RF model and surrogate Gaussian model is generally referred to as universality \citep{panahi2017universal,oymak2018universality, HuUniversalityLaws}. Next, the CGMT is applied, which provides an alternative optimization problem whose analysis is provably tied to the original problem. This alternative optimization formulation has been a great tool for computing precise asymptotic learning curves in the case of uncorrelated features. However, for the general RF formulation, the surrogate features are inevitably heavily correlated. As a result, the alternative optimization has been generally as difficult to analyze as the original RF model. More precisely,  solution of this alternative optimization typically involves solving a proximal operator of a non separable $m$-dimensional vector that scales with the number of model parameters \citep{Loureiro2021Learning}, even if the regularization function is separable. Only in the case of ridge ($\ell_2^2$) regularization, where a specific rotational symmetry holds true, can this difficulty be overcome \citep{chang2020provable, montanari2019generalization, d2020double}.


{\bf Contributions}:
The first main contribution of this paper is a novel multilevel application of the CGMT to the correlated surrogate model that overcomes the difficulties with the analysis of the alternative optimization and substantially simplifies the final results. With this method, we provide a computable technique for obtaining learning curves of surrogate Gaussian model with arbitrary separable, strongly convex; or $\ell_1$ regularization. Our next contribution is to establish universality, i.e. to show that our analysis also applies to the original, non-Gaussian random features. This result has been previously established for regularization functions that are thrice differentiable and strongly convex \citep{HuUniversalityLaws}. We extend this result in two steps. First we show that a wider variety of potentially nondifferentiable, strongly convex functions satisfy universality. In particular, we show that a combination of $\ell_1$ and $\ell_2^2$, known as elastic net \citep{zou2005regularization}, is universal. Furthermore, under the assumptions that the activation function is continuous and Lipschitz, and solution vector that is sufficiently sparse, we show that the $\ell_2^2$ part of the elastic net regularization can be removed and the universality of pure $\ell_1$ (which is not strongly convex) is established. 

\section{Related Works}
The asymptotic analysis of RF models is recently culminated in the study of the so-called double descent phenomenon, where increasing the model size beyond the interpolation threshold, surprisingly improves the learning performance, leading to a learning curve with two descent regions. The double descent phenomenon has a long history \citep{loog2020brief}, but was first discussed in its modern form by \citep{belkin2019reconciling} (see also \citep{geiger2020scaling}). Overparameterized systems have since been studied extensively, for an incomplete list see  \citep{tsigler2020benign, hastie2019surprises, mei2019generalization, bartlettLongLugosiTsigler_2020, belkin2020two, MuthukumarHarmless2019, kobak2020optimal, deng2019model, taheri2021fundamental, lolas2020regularization, mignacco2020role, kini2020analytic, liang2020precise, montanari2019generalization, taheri2020sharp, salehi2019impact}

 Gaussian comparison theorems have played a central role in obtaining exact learning curves, which go back to \citep{gordon1985some, gordon1988milman}. They show an asymptotic equivalence between certain optimization problems over Gaussian random variables. \citep{thrampoulidis2015asymptotically, thrampoulidis2015gaussian} showed that in the presence of convexity, the bounds provided by Gordon could be refined. The applications of comparison theorems to the study of the asymptotic regime are numerous \citep{bosch2021double, Loureiro2021Learning, LuPreciseRandom, ThrampoulidisMestimators, chang2020provable}. A principal difficulty with the CGMT is in the case of correlated covariates, as in the RF model. This results in the alternative optimization problem of the CGMT to be no more tractable than the original problem.
In the case of $\ell^2_2$ regularization, rotational symmetry may be applied to study correlated models. In the papers such as \citep{chang2020provable, mei2019generalization, LuPreciseRandom} this symmetry is exploited to derive analytic expressions. We are not aware of any analytic expressions derived by means of the CGMT considering RFs with more generic regularization. As a contribution of this paper, we resolved the issue of correlated covariates with a novel approach involving multiple applications of CGMT and extend the analysis of regularized least squares into RF features with a larger set of regularization functions. 

  The Gaussian Equivalence Principle (GEP) expresses that there exists an asymptotic equivalence between RF models and Gaussian models with identical first and second moments. This universality was shown for (regularized) least squares by \cite{panahi2017universal}, extended to generic convex regularization by \cite{HuUniversalityLaws} and for generative models by \cite{goldt2022gaussian}. More recent results by \cite{MontanariUniversalityERM} extends universality to empirical risk minimization with regularization. \cite{baGradientStep2022} has also extended universality results to RF models after a single step of gradient descent with small step sizes. The results of \cite{HuUniversalityLaws} and \cite{MontanariUniversalityERM} however do not hold in the case of $\ell_1$ regularization, while those of \cite{panahi2017universal} do not apply to the random features case. We extend the universality results of \cite{HuUniversalityLaws} to the case of of $\ell_1$ and elastic net regularization. \cite{liang2020precise} also demonstrate the universality of $\ell_1$ regularization but for a different setup of max-margin classifiers. Their results cannot simply be translated to that of ours. Firstly, they only consider universality of the objective value, while we additionally demonstrate universality for strongly convex functions of the solution vector. Secondarily they require that the activation function is restricted to a compact set, which we do not require here.

\section{Random Features Model}
\label{sec:randomfeatures}
We consider a dataset $\lbrace (\bz_i, y_i)\in\mathbb{R}^d\times \mathbb{R} \rbrace_{i=1}^n$  and wish to determine the relationship between the data vector $\bz_i$ and the labels $y_i$ by means of a function of the following form:
\begin{equation}
    f(\bz_i;\btheta, \bvarphi) = \frac{1}{\sqrt{m}}\btheta^T\bvarphi(\bz_i) \quad \btheta \in \mathbb{R}^m.
\end{equation}
 Here $\bvarphi:\mathbb{R}^d \rightarrow \mathbb{R}^m$ is a fixed nonlinear feature map, whose relation to the labels $y_i$ is characterized by a variable weight vector $\btheta$. We determine  $\btheta$ by the following optimization problem:
\begin{equation}
    \label{eq:keyOptimization}
    \hat{\btheta} = \arg\min_{\btheta} \sum_{i=1}^n l(f(\bz_i;\btheta, \bvarphi), y_i) + r(\btheta),
\end{equation}
where $l(x, y)=\frac{1}{2}(x - y)^2$ is the square-loss function and $r(x)$ is a regularization function. We consider a wide-range of regularization functions which are explained in Section~\ref{sec:mainResult}. 
%
We restrict ourselves to the feature map
\begin{equation}
    \label{eq:FeatureMap}
    \bvarphi(\bz_i) = \sigma\left(\frac{1}{\sqrt{d}}\bW\bz_i\right),
\end{equation}
where $\sigma:\mathbb{R}\rightarrow\mathbb{R}$ is a non linear, odd activation function applied element wise (eg. $\tanh(x)$), and $\bW\in\mathbb{R}^{m\times d}$ is a random weight matrix whose elements are i.i.d standard Gaussians, independent of $\bz_i$. 
We note that this choice of the random feature map can be interpreted as a Neural Network (NN) with one hidden layer. We let the matrix $\bX$ be given such that $\bX_{ij} = \varphi_j(\bz_i) = \sigma\left(\frac{1}{\sqrt{d}}\bw_j^T\bz_i\right)$, where $\bw_j^T$ is the $j^\tth$ row of $\bW$.
We consider two metrics of the performance of the solution $\hat{\btheta}$ of \eqref{eq:keyOptimization}, the training error, expressed in matrix notation as
\begin{equation}
    \label{eq:trainerror}
   \mathcal{E}_{train}(\btheta)  = \frac{1}{2n}||\by - \frac{1}{\sqrt{m}}\bX\btheta||_2^2 + \frac{1}{m}r(\btheta)
\end{equation}
and the generalization error
\begin{equation}
    \label{eq:generror}
\mathcal{E}_{gen}(\btheta) = \mathbb{E}\left[\frac{1}{2}(y_{new} - f(\bz_{new}; \btheta, \bvarphi))^2 \right],
\end{equation}
where $(\bz_{new}, y_{new})$ is a new sample pair independent of, but identically distributed to the training data.

Analysis of this problem requires making assumptions on the distribution of the the dataset. We assume that $\bz_i \overset{i.i.d.}{\sim}\mathcal{N}(0, \bI_{d})$ and that the labels $y_i$ are generated according to
\begin{equation}
    \label{eq:ylabelDef}
    y_i = \frac{1}{\sqrt{m}}\btheta^{*T}\bvarphi(\bz_i) + \epsilon_i,
\end{equation}
where $\btheta^*$ is a fixed weight vector that may be deterministic or random and $\epsilon_i$ is i.i.d. noise with $\mathbf{E}[\epsilon_i] = 0$, $\mathbb{E}[\epsilon_i^2] = \sigma_{\bepsilon}^2$ and $\mathbb{E}[\epsilon_i^4]<\infty$, and $\bvarphi$ is given in \eqref{eq:FeatureMap}. We note that this method of label generation is different that that of \cite{HuUniversalityLaws}, we note that their results still apply in this context. For a discussion of this fact see remark \ref{remark:HuLuLossfunctionDifference} in the appendix.

Under these assumptions, the main goal of this paper is to predict the values of $\mathcal{E}_{gen}(\hat\btheta), \mathcal{E}_{train}(\hat\btheta)$, where $\hat\btheta$ is given by \eqref{eq:keyOptimization}. Further, we provide the asymptotic value of $h(\hat\btheta)$ where $h$ is an arbitrary test function from a wide range of choices, as we elaborate.

\section{Main Results}\label{sec:mainResult}

\subsection{Overview of Main Results}
%
Before delving into details, we provide an overview of our main results. A more detailed and rigorous treatment is provided in the subsequent sections.

The key optimization problem in \eqref{eq:keyOptimization} can be written as 
\begin{eqnarray}
\label{eq:keyNonGaussianOptimization}
P_1 = \min_{\btheta} \frac{1}{2n}||\by - \frac{1}{\sqrt{m}}\bX\btheta||_2^2 + \frac{1}{m}r(\btheta).
\end{eqnarray}
Hence, the optimal solution of $P_1$ is given by \eqref{eq:keyOptimization}. However we consider a slightly more general problem of the following form:  
\begin{eqnarray}
    \label{eq:pertNonGaussianOptimization}
    \tilde P_1(\tau_1,\tau_2) = \min_{\btheta} \frac{1}{2n}||\by - \frac{1}{\sqrt{m}}\bX\btheta||_2^2 + \frac{1}{m}r(\btheta) 
    + \frac{\tau_1}{m}(\btheta - \btheta^*)^T\bR(\btheta - \btheta^*) + \frac{\tau_2}{m} h(\btheta),
\end{eqnarray}
where $\tau_1,\tau_2$ are real numbers and $h(\be)$  is a test function such that $r+\tau_2 h$ is convex. Moreover, $\bR$ is the feature covariance matrix $\mathbb{E}_{\bz}[\bvarphi(\bz)\bvarphi(\bz)^T]$.  We refer to the solution of \eqref{eq:pertNonGaussianOptimization} as $\tilde{\btheta}_1(\tau_1, \tau_2)$.

We note that setting $\tau_1 =\tau_2 =0$, we obtain the original problem \eqref{eq:keyNonGaussianOptimization}, i.e. $P_1=\tilde P_1(0,0)$ and $\hat\btheta_1=\tilde\btheta_1(0,0)$. These additional ``$\tau$'' are added to the problem definition to prove the universality of generalization error and of generic strongly convex functions. We note that the $\tau_1$ term corresponds to a component of the generalization function and $\tau_2$ is attached to the generic function $h(\btheta)$. Taking the derivative with respect to $\tau_1,$ or $\tau_2$ allows these terms to be recovered, this property is made use of in the proof of the universality, see proof of theorem \ref{thm:universality:sequence}. 

We analyze the problem in \eqref{eq:pertNonGaussianOptimization} by considering two alternative problem formulations, and demonstrating that they are asymptotically equivalent to one another. 

Consider the linear feature map
\begin{eqnarray}\label{eqn:linearfeaturemap}
\tilde{\bvarphi}(\bz) = \frac{\rho_1}{\sqrt{d}}\bW\bz + \rho_*\bg ,
\end{eqnarray}
where $\rho_1 = \mathbb{E}_a[a\sigma(a)]$ and $\rho_*^2 = \mathbb{E}_a[\sigma^2(a)] - \rho_1^2$, with $a\sim\mathcal{N}(0, 1)$, and $\bg\sim\mathcal{N}(0, I_m)$. This feature map is obtained by means of a truncated Hermite polynomial expansion of the original feature map \eqref{eq:FeatureMap}, as discussed in \cite{mei2019generalization}, and unlike the original feature maps in \eqref{eq:FeatureMap} these feature are Gaussian (for fixed weights $\bW$). Let $(\tilde{\bX})_{ij} = \tilde{\bvarphi}_j(\bz_i) = \frac{\rho_1}{\sqrt{d}}\bw_j^T\bz_i + \rho_*g_{ij}$, where $g_{ij}$ are i.i.d Gaussian, and consider the problem
\begin{eqnarray}
\label{eq:GaussianKeyOptimization}
\tilde P_2(\tau_1, \tau_2) = \min_{\btheta} \frac{1}{2n}||\by - \frac{1}{\sqrt{m}}\tilde{\bX}\btheta||_2^2 + \frac{1}{m}r(\btheta) 
    + \frac{\tau_1}{m}(\btheta - \btheta^*)^T\tilde{\bR}(\btheta - \btheta^*) + \frac{\tau_2}{m} h(\btheta),
\end{eqnarray}
where $y_i = \btheta^{*T}\tilde{\bvarphi}(\bz_i) + \epsilon_i$ and $\tilde{\bR} = \mathbb{E}_{\bz}[\tilde{\bvarphi}(\bz)\tilde{\bvarphi}(\bz)^T]$. The optimal solution of $\tilde P_2(\tau_1,\tau_2)$ is referred to as  $\tilde{\btheta}_2(\tau_1, \tau_2)$. In particular, we denote $\hat{\btheta}_2 = \tilde{\btheta}_2(0, 0)$.

 Now, we define $\psi(\beta,q,\xi,t,\tau_1, \tau_2)$ as follows
\begin{eqnarray}
\psi(\beta,q,\xi,t, \tau_1,\tau_2)=
\frac{1}{m}
\mathbb{E}\left[
\mathcal{M}_{\frac{1}{2c_1} (r +\tau_2h)}\left(\btheta^* - \frac{c_2\sqrt{\gamma}}{2c_1}\bphi \right)\right]
\nwl -\frac{c_2^2\gamma}{4c_1} + \frac{\xi t}{2} + \frac{\beta q}{2} + \frac{\beta\sigma_{\bepsilon}^2}{2q} + \frac{\xi\beta^2}{2t\eta}  - \frac{(\beta + 2\tau_1q)\xi^2}{2q}  - \frac{q\beta^2}{2(\beta + 2q\tau_1)\eta} - \frac{\beta^2}{2},
\end{eqnarray}

where $\calM_{\frac{1}{2c_1}(r+\tau_2h)}$ is the Moreau envelope of $r + \tau_2h$ 
with the step size $\frac 1{2c_1}$ (see supplement definition \ref{def:MoreauProx}),
$\bphi$ is a standard Gaussian vector, $c_1$ and $c_2$ are functions of $\beta, q, \xi, t, \tau_1, \tau_2$ given by
\begin{eqnarray}
     c_1 = \frac{(\beta+2\tau_1q)^2\rho_1^2\xi}{2{q}^2 t} + \frac{(\beta + 2q\tau_1)\rho_*^2}{2{q}} \\c_2 = \sqrt{ \frac{(\beta+2\tau_1q)^2\rho_1^2\xi^2\eta}{ q^2} + \beta^2\rho_*^2}.
\end{eqnarray}
The expectation is taken with respect to $\bphi$ and hence the function $\psi$ is not random. Accordingly, we define the key alternative optimization problem, i.e. a four-dimensional scalar optimization problem,  in our development:
\begin{eqnarray}
\label{eq:4doptproblem}
\tilde P_3(\tau_1, \tau_2) =\max_{\beta > 0} \min_{q>0}\max_{\xi > 0}\min_{t > 0} \psi(\beta,q,\xi,t, \tau_1,\tau_2).
\end{eqnarray}
Let $\tilde{\beta}, \tilde{q}, \tilde{\xi}, \tilde{t}$ be the optimal point of $\tilde{P}_3$ and let $\tilde{c}_1 = c_1(\tilde{\beta}, \tilde{q}, \tilde{\xi}, \tilde{t})$ and $\tilde{c}_2= c_2(\tilde{\beta},\tilde{q},\tilde{\xi})$. Accordingly, we define $\tilde{\btheta}_3(\tau_1, \tau_2)$ as follows 
\begin{eqnarray}\label{eqn:soln:4doptproblem}
\tilde{\btheta}_3(\tau_1, \tau_2) := \mathrm{prox}_{\frac{1}{2\tilde{c}_1} (r+\tau_2h)}\left(\btheta^* - \frac{\tilde{c}_2\sqrt{\gamma}}{2\tilde{c_1}}\bphi \right),
\end{eqnarray}
where $\mathrm{prox}_{\frac{1}{2\tilde{c}_1} (r +\tau_2h)}$ denotes the proximal operator of $r+\tau_2h$ with the step size $\frac 1{2\tilde{c}}$. 
Similar to the two previous cases, we define
$\hat{\btheta}_3 = \tilde{\btheta}_3(0, 0)$.
The training and generalization error corresponding to problem $\tilde{P}_3$ are not given by \eqref{eq:trainerror} and \eqref{eq:generror}, instead we have that
\begin{eqnarray}
\tilde{\mathcal{E}}_{train} = \tilde{P}_3(0, 0) \quad 
\tilde{\mathcal{E}}_{gen} = \sigma_{\bepsilon}^2 + \left. \frac{\partial \tilde{P}_3(\tau_1, 0)}{\partial\tau_1}\right|_{\tau_1 = 0}.
\end{eqnarray}
Now, we provide a summary of our main results: 
\begin{theorem}\label{thm:mainresult:informal}
{\bf{Informal statement of the main results}} \\
There exist symmetric intervals $\tau_1\in [-\tau_1^*, \tau_1^*]$ and $\tau_2\in [-\tau_2^*, \tau_2^*]$ with sufficiently small universal constants $\tau_1^*,\tau_2^*$, a wide family of strongly convex, separable functions $r$ and potentially non-convex, separable test functions $h$, for which in the asymptotic limit,
\begin{eqnarray}\label{eqn:mainresult:scalarization}
\tilde P_1(\tau_1,\tau_2)\approx
\tilde P_2(\tau_1,\tau_2),\quad \tilde P_2(\tau_1,\tau_2)\approx
\tilde P_3(\tau_1,\tau_2)
\end{eqnarray}
and hence
\begin{equation}
    \tilde P_1(\tau_1,\tau_2)\approx
\tilde P_3(\tau_1,\tau_2).
\end{equation}
By the above result, we may conclude for such scenarios that
\begin{eqnarray}
\label{eqn:mainresult:universality1}
 \mathcal{E}_{train}(\hat{\btheta}_1)\approx&\mathcal{E}_{train}(\hat{\btheta}_2) &\approx \tilde{\mathcal{E}}_{train},
\\
\label{eqn:mainresult:universality2}
 \mathcal{E}_{gen}(\hat{\btheta}_1) \approx& \mathcal{E}_{gen}(\hat{\btheta}_2)&\approx \tilde{\mathcal{E}}_{gen},
\end{eqnarray}
and
\begin{eqnarray}
\label{eqn:mainresult:universality3}
h(\hat\btheta_1)\approx h(\hat\btheta_2)\approx h(\hat\btheta_3).
\end{eqnarray}
The above result also holds for $\ell_1$ regularization under some considerations about the true model $\btheta^*$ and the activation function.
\end{theorem}

{\bf{ Discussion of Main Result:}}
By Theorem \ref{thm:mainresult:informal}
, the generalization/training error and other properties of the original problem $P_1$, represented by a test function $h$, can be found using the solution of $P_3=\tilde P_3(\tau_1=0,\tau_2=0)$. See Theorem~\ref{thm:mainresult:formal} for a precise statement. 
Note that $P_3$ is scalar and since $r$ is separable, calculating $\mathbb{E}\left[
\mathcal{M}_{\frac{1}{2c_1} r}\left(\btheta^* - \frac{c_2\sqrt{\gamma}}{2c_1}\bphi \right)\right]$ is straightforward ($\tau_1,\tau_2$ are set to zero). Hence, $P_3$ is simple to evaluate using standard computation techniques.  

We note that Theorem \ref{thm:mainresult:informal} is,  at first sight, similar to Theorem 1 in \cite{Loureiro2021Learning}, which is also based on the Moreau envelope and the proximal operator of the regularization function. However, we note that the argument of the Moreau envelope in their expression is more complex and cannot be generally evaluated even if $r$ is separable. Hence, our result is novel and not the same as \cite[Thm.~1]{Loureiro2021Learning} and allows significantly easier calculation of the generalization error compared to other existing methods in the literature for the correlated RF model.  

Our proof has two building blocks: 
Using a novel multi-level application of CGMT, we show in Theorem~\ref{thm:GaussianAsymptotics}, the convergence of $\tilde P_2$ to the scalar optimization problem $\tilde P_3$ in the left hand side of \eqref{eqn:mainresult:scalarization}.   
The universality result, i.e. the asymptotic convergence between  $\tilde P_1$ and $\tilde P_2$ in the right hand side of \eqref{eqn:mainresult:scalarization} is presented in Section~\ref{UnversalityResults}. The other claims i.e \eqref{eqn:mainresult:universality1},\eqref{eqn:mainresult:universality2} and \eqref{eqn:mainresult:universality3} are subsequently obtained by an individual argument.

For strongly convex and thrice differentiable regularization functions, the universality relation in the right hand side of \eqref{eqn:mainresult:scalarization} has already been demonstrated in \cite{HuUniversalityLaws}. Here, we extend these results to  the case of a sequence of strongly convex, thrice differentiable functions with bounded third derivatives that converge uniformly to the regularization function (Theorem~\ref{thm:universality:sequence}). Such functions may not be even differentiable. Moreover, while \cite{HuUniversalityLaws} also shows the universality of the generalization/test errors, we extend this result and show that the entire discussion holds true for an arbitrary test function $h$ obtained as the uniform limit of a sequence of thrice differentiable functions with bounded third derivatives. Exact assumptions will be shortly presented.
The above approach also allows us to extend the universality results to elastic net (Corollary~\ref{lem:universality:elasticNet}) and $\ell_1$ regularization (Theorem~\ref{thm:l1Universality}), which have not been provided in the literature before.

\subsection{Assumptions}
Below, we provide a list of all assumptions considered in our study. The specific assumptions that are used for each result is provided under the statement of the associated result. 
\begin{itemize}
    \item[A1]  The regularization function satisfies one of the below: 
    \begin{itemize}
        \item Case A: For positive constants $\mu, L > 0$, there exists a sequence of functions $r^{(k)}$ that are separable, $\mu$-strongly convex and thrice differentiable with $L-$uniformly bounded third derivatives\footnote{Note that for a generic multi-variable function, the  derivatives are  tensors and we refer to their operator norm for bounds. However, as the functions are separable, i.e. a scalar function is applied element-wise, the bounds are simply on the derivatives of the scalar function.}. The sequence $r^{(k)}$ converges uniformly in the limit of $k\rightarrow\infty$ to the regularization function $r$.
        \item Case B: The regularization function is $r(\btheta) = \lambda||\btheta||_1$.
    \end{itemize}
    Note that it is sufficient that one of these assumptions, either Case A or Case B, holds true. 
    \item[A2] For positive constants $l, L > 0$, there exists a sequence of thrice differentiable functions $h^{(k)}$  with $l-$uniformly bounded second derivatives and $L-$uniformly bounded third derivatives. The sequence $h^{(k)}$ converges uniformly in the limit of $k\rightarrow\infty$ to the test function $h$.
    
    \item[A3] The noise vector $\bepsilon$ has elements $\epsilon_i$ which are i.i.d with $\mathbb{E}[\epsilon_i] = 0$, $\mathbb{E}[\epsilon_i^2] = \sigma_{\bepsilon}^2 < \infty$ and $\mathbb{E}[\epsilon_i^4] < \infty$.
    \item[A4] The dimensions $n, m, d$ remain at constant ratio when they are increased to infinity. These ratios are given by $\gamma = \frac{n}{m}$, $\eta = \frac{n}{d}$ and $\delta = \gamma\eta = \frac{m}{d}$
    \item[A5]  The true model ${\btheta^*}$ is independent of $\bX$. We assume that for some constants $c, c', C>0$, \\$\mathbb{P}\left(\frac{1}{\sqrt{m}}\max(\|\nabla r(\btheta^*)\|_2, \|\nabla h(\btheta^*)\|_2 ) > c \right)\rightarrow 0$, and 
    $\mathbb{P}(\max_i|(\nabla h(\btheta^*))_i| \geq c\log m) \leq Ce^{-c'(\log m)^2}$.
      \item[A6] The activation function $\sigma(\cdot)$ is odd, with bounded first, second, and third derivatives.
\end{itemize}

Given the assumptions, we state the values of the bounds on $\tau_1$ and $\tau_2$:
\begin{eqnarray}
|\tau_1| \leq \tau_1^* = \frac{\mu/8}{\rho_1^2(1+2\sqrt{\delta})^2  +\rho_*^2}\quad |\tau_2|\leq \tau_2^* = \frac{\mu}{4l},
\end{eqnarray}
where the values of $\mu$ and $l$ are given in A1 and A2, respectively. Both of these bounds are chosen to ensure that the sum of the regularization function and the two ``$\tau$ terms'' remains strongly convex with high probability.

\subsection{Asymptotic Gaussian Results}
\label{subsec:AsymptoticGaussianResults}

In this section, we state our main result connecting $P_2$ in \eqref{eq:GaussianKeyOptimization} and $P_3$ in \eqref{eq:4doptproblem}. 

\begin{theorem}
\label{thm:GaussianAsymptotics}
Let Assumptions A3-A5 hold and $r+\tau_2 h$ is $\frac\mu 2-$strongly convex for $\tau_2 \in [-\tau_2^*, \tau_2^*]$. 
Then for all $\tau_1 \in [-\tau_1^*, \tau_1^*]$ and $\tau_2 \in [-\tau_2^*, \tau_2^*]$,
\begin{equation}
\left|\tilde P_2(\tau_1,\tau_2) - \tilde P_3(\tau_1,\tau_2)\right| \xrightarrow[n, m, d\rightarrow\infty]{P} \bm{0}
\end{equation}
Moreover,
\begin{align}
    \left|\left(\mathcal{E}_{train}(\hat{\btheta}_2), \mathcal{E}_{train}(\hat{\btheta}_2), \frac{1}{m}h(\hat{\btheta}_2) \right) \right. \left.- \left(\tilde{\mathcal{E}}_{train}, \tilde{\mathcal{E}}_{gen}, \frac{1}{m}h(\hat{\btheta}_3) \right)\right| \xrightarrow[n, m, d\rightarrow\infty]{P} \bm{0}
\end{align}
where $\hat{\btheta}_2$ is the solution to problem \eqref{eq:GaussianKeyOptimization} and $\hat{\btheta}_3$ is the solution presented in \eqref{eqn:soln:4doptproblem} associated with $P_3$ in \eqref{eq:4doptproblem}. 
\end{theorem}

This result makes the statement in the second equation of \eqref{eqn:mainresult:scalarization} precise. Note that we do not need A6 and the assumption for $r,h$ is weaker than the combination of A1-Case A and A2. A6, A1-Case A and A2 are required for the next step concerning $P_1$. The results for A1-Case B will be obtained from the study of A1-Case A, in a suitable limit. Note that for this result,  $\rho_1$ and $\rho^*$ in  \eqref{eqn:linearfeaturemap} can be arbitrary, but we will set them to the values discussed in text following \eqref{eqn:linearfeaturemap} for the subsequent results. 

\subsubsection{Proof Sketch of Theorem \ref{thm:GaussianAsymptotics}}
The proof of this statement makes use of the Convex Gaussian Min Max Theorem (CGMT), which establishes an asymptotic equivalence between a primary ($P$) and an alternative ($A$) optimization problem of the following form:
\begin{eqnarray}
\label{eq:primary}
P(\bA)= & \minl_{\bx\in S_x}\maxl_{\by \in S_y} \bx^T\bA\by + \psi(\bx, \by)\\
\label{eq:alternative}
A(\bg, \bh)= & \minl_{\bx \in S_x}\maxl_{\by \in S_y} ||\by||_2\bx^T\bg + ||\bx||_2\by^T\bh + \psi(\bx, \by)
\end{eqnarray}
Here, $\bA\in\mathbb{R}^{m\times n}, \bg\in \mathbb{R}^{m}, \bh\in\mathbb{R}^n$ have i.i.d standard Gaussian elements, $\psi(\bx, \by)$ is an arbitrary convex-concave function, and $S_x\subset \mathbb{R}^m, S_y\subset\mathbb{R}^n$ are compact and convex sets. For more details, see supplement \ref{app:CGMT}. To prove Theorem \ref{thm:GaussianAsymptotics}, we first fix $\bW$ and $\btheta^*$ and translate the original minimization problem into a min-max problem of the form in \eqref{eq:primary} by suitable transformations and change of variables.  Then, we invoke the CGMT which eliminates the randomness (in $\bX$ ) due to the data set $\bz$ \footnote{This means that the terms including the random matrix $\bA$ in $P(\bA)$ will be removed and replaced by terms including random vectors $\bg,\bh$ in $A(\bg,\bh)$.} and re-express the problem in terms of \eqref{eq:alternative}. The resulting expression is given in \citep{LuPreciseRandom, Loureiro2021Learning}, but it is well-known to be intractable as it depends on the covariance matrix of the Gaussian feature map. Here, we introduce a key novel step. We show that assuming random weights $\bW$, under further suitable, non-trivial transformations, the resulting equivalent form in \eqref{eq:alternative} itself can be transformed into the form of \eqref{eq:primary} with a new random matrix $\bA$ representing the randomness of the weights. This allows us to apply the CGMT again, resulting in the elimination of the random matrix $\bW$. Finally, we simplify the expressions obtained by the second CGMT application, which leads to the results in Theorem \ref{thm:GaussianAsymptotics}. The full proof is given in the Appendix \ref{sec:GuassianTheoremProof}.

\subsection{Universality}
\label{UnversalityResults}

Next, we demonstrate universality. Here we show that  the solution vectors problems $P_1$ given in \eqref{eq:keyNonGaussianOptimization} and problem $P_2$ given in \eqref{eq:GaussianKeyOptimization} result in asymptotically equivalent values, not only in the training and generalization error, but also in a wide family of other test functions $h$. We provide two novel theorems, in this sections, that extend the existing results for the universality of random feature models. For completeness we first state the existing results by \cite{HuUniversalityLaws}. 

\begin{theorem}[\citep{HuUniversalityLaws} Theorem 1, Proposition 1]
\label{thm:KnownHUUnivLaws}
Let assumptions  A3-A5 hold. Set $\tau_2=0$ and  let $r(\btheta)$ be a regularization function that is strongly convex and thrice differentiable with uniformly bounded third derivatives. Let $\hat{\btheta}_1, \hat{\btheta}_2$ be the optimal solution to the problems given in \eqref{eq:keyNonGaussianOptimization} and \eqref{eq:GaussianKeyOptimization}, respectively. Then for all $\tau_1 \in [-\tau_1^*, \tau_1^*]$, 
\begin{equation}
    \tilde P_1(\tau_1,0)\to \tilde P_2(\tau_1,0) 
\end{equation}
As a result,
\begin{align}
\label{eqn:thm:universality:train}
    \left|\left(\mathcal{E}_{train}(\hat{\btheta}_1), \mathcal{E}_{gen}(\hat{\btheta}_1) \right) - \left(
    \mathcal{E}_{train}(\hat{\btheta}_2),
  \mathcal{E}_{gen}(\hat{\btheta}_2)
  \right) \right| \xrightarrow[n, m, d\rightarrow\infty]{P}\bm{0}
\end{align}
\end{theorem}
\begin{remark}
The statement of the Theorem \ref{thm:KnownHUUnivLaws} is adapted to the particular setup that we consider here. For completeness the original theorem is given in appendix \ref{App:UniversalityTheorems} as Theorem \ref{thm:UnivHuLu}. 

\end{remark}

We are now ready to present our contribution. Firstly, we demonstrate the following theorem, relaxing the condition on the regularizer in Theorem \ref{thm:KnownHUUnivLaws}, to A1-Case~A and extending the result to an arbitrary test function $h$: 

\begin{theorem}\label{thm:universality:sequence}
Let A2-A6 hold and the regularization function $r$ satisfies  A1-Case~A.  Then for all $\tau_1 \in [-\tau_1^*, \tau_1^*]$ and $\tau_2 \in [-\tau_2^*, \tau_2^*]$,
\begin{equation}
    \left|\tilde P_1(\tau_1,\tau_2) - \tilde P_2(\tau_1,\tau_2) \right|\xrightarrow[n, m, d\rightarrow\infty]{P} 0 
\end{equation}
As a result,
\begin{align}
\label{eqn:thm:universality:train:novel}
    \left|\left(\mathcal{E}_{train}(\hat{\btheta}_1), \mathcal{E}_{gen}(\hat{\btheta}_1), \frac{1}{m}h(\hat{\btheta}_1) \right) -
    \left(
    \mathcal{E}_{train}(\hat{\btheta}_2),
  \mathcal{E}_{gen}(\hat{\btheta}_2), \frac{1}{m}h(\hat{\btheta}_2)
  \right)\right|\xrightarrow[n, m, d\rightarrow\infty]{P} \bm{0}
\end{align}
\end{theorem}

The next result illustrates that universality can also be applied to elastic net regularization
\begin{cor}\label{lem:universality:elasticNet}
Let A2-A5 hold.  Let $r(\btheta) = \lambda ||\btheta||_1 + \frac{\mu}{2}||\btheta||_2^2 $. Then, the claims of Theorem \ref{thm:universality:sequence} hold true. 
\end{cor}

\subsubsection{Proof Sketch of Theorem \ref{thm:universality:sequence}}
The original proof given by \cite{HuUniversalityLaws} is valid only for regularization functions that are strongly convex and thrice differentiable with uniformly bounded third derivatives. We first extend these results to sequence of regularization functions $r^{(k)}$ that converge uniformly to a function $r$. Noting that this theorem holds for all $r^{(k)}$ with $k<\infty$ the proof consists of demonstrating that the relations hold in the limit. Second, \cite{HuUniversalityLaws} does not consider the term $\tau_2h(\btheta)$.
We adopt the original proof of \cite{HuUniversalityLaws} and modify it to demonstrate that the results similarly hold with a more generic test function $h(\btheta)$. The proof of these results are given in the Appendix~\ref{app:subsec:univSequenceTheorem}. 

For the specific case of elastic net, we construct a valid sequence $r^{(k)}(\btheta)$ that uniformly converges to the elastic net regularization function, see Appendix~\ref{sec:pf:universality:elasticNet}. 

\subsection{Random Features and Scalar Optimization Problem}

We now connect the original problem $P_1$ to the scalar optimization problem $P_3$ by combining the results in Section~\ref{subsec:AsymptoticGaussianResults} and Section~\ref{UnversalityResults}. This leads to the following precise statement of the main result in  Theorem~\ref{thm:mainresult:informal}:  

\begin{theorem}\label{thm:mainresult:formal}
Let Assumptions A2 - A6 and A1.Case A hold. 
Then for all $\tau_1 \in [-\tau_1^*, \tau_1^*]$ and $\tau_2 \in [-\tau_2^*, \tau_2^*]$, 
\begin{equation}
\left|\tilde P_1(\tau_1,\tau_2) - \tilde P_3(\tau_1,\tau_2)\right| \xrightarrow[n, m, d\rightarrow\infty]{P} 0
\end{equation}
Moreover,
\begin{align}
    \left|\left(\mathcal{E}_{train}(\hat{\btheta}_1), \mathcal{E}_{train}(\hat{\btheta}_1), \frac{1}{m}h(\hat{\btheta}_1) \right)- \left(\tilde{\mathcal{E}}_{train}, \tilde{\mathcal{E}}_{gen}, \frac{1}{m}h(\hat{\btheta}_3)\right) \right| \xrightarrow[n, m, d\rightarrow\infty]{P} \bm{0}
\end{align}
\end{theorem}

\subsection{Results for $\ell_1$ regularization}
We further extend these results to the case of $\ell_1$ regularization. For this case, additional assumptions are needed. In particular, we may only consider scenarios, where problem \eqref{eq:keyOptimization} is sufficiently sparse. This is defined by the following:
\begin{equation}\label{eq:eff}
    M_0=\frac 1m\suml_{i=1}^m\Pr\left(\hat\theta_{i,3}\neq 0\right)
\end{equation}
where $\hat\theta_{i,3}$  denotes the $i^\tth$ element of  $\hat\btheta_3$ for the regularization function $r(\btheta)=\lambda\|\btheta\|_1$. 
We prove the following theorem:
\begin{theorem}
\label{thm:l1Universality}
Let Assumptions A2 - A6 hold and $r(\btheta) = \lambda||\btheta||_1$.  The exists a constant $\rho$ only depending on the activation function $\sigma$ and the parameters of the problem ($\lambda,\sigma^2_{\bepsilon},\gamma,\eta$) such that for $M_0<\rho$,
the results of Theorem \ref{thm:mainresult:formal} holds for $r(\btheta) = \lambda||\btheta||_1$.
\end{theorem}

\subsubsection{Proof Sketch of Theorem \ref{thm:l1Universality}}
We adopt the proof in \cite{panahi2017universal}, which performs this procedure for i.i.d. sub-Gaussian features, and modify it for the random feature model. Extending the results for the $\ell_1$ regularization involves the results for the elastic net optimization in corollary 1. In  \cite{panahi2017universal} (section 3.3 of supplement), it is shown that for a small value of $\mu$ in the elastic net regularization $\lambda\|\ldotp\|_1+\frac\mu 2\|\ldotp\|_2^2$, the $\ell_2$ term can be removed and the change of the solution is negligible, if the matrix $\bX$ satisfies a proper restricted isometry property (RIP).   In  \cite{panahi2017universal} (lemma 8 in supplement), the RIP is shown for i.i.d. sub-Gaussian features. We extend this result and show that a similar RIP condition holds for random features model. The condition on $M_0$ ensures that the optimal solution is sufficiently stable, which otherwise is not guaranteed with the lack of strong convexity.
The full proof is presented in Appendix~\ref{app:subsec:pf:l1Universality}.

\section{Elastic Net Regularization}
In this section, we apply our results to the case of elastic net regularization, for which asymptotic learning curves has not been previously proposed. We consider the regularization function
\begin{eqnarray}
    r(\btheta)  = \lambda||\btheta||_1 + \frac{\alpha}{2}||\btheta||_2^2,
\end{eqnarray}
where $\lambda$ and $\alpha$ are two regularization parameters. We note that in the case of $\lambda = 0$ we obtain ridge regularization and in the case of $\alpha = 0$ we obtain $\ell_1$ regularization (LASSO). Due to the continuity of the asymptotic expressions, the analysis of elastic net may be directly used for the study of ridge or LASSO regression simply by setting either $\lambda=0$ or $\alpha=0$. Our interest in studying elastic net stems from the sparsity-promoting effect of the $\ell_1$ regularizer on the solution vector. 
When viewing the RF model as a shallow neural network, the effect of a sparse solution is to disable a number of nodes in the hidden layer. As a result, elastic net finds a subnetwork of the original NN with a minimal degradation in performance, in effect a form of network compression. For similar attempts, see for example \citep{tang2022survey,oyedotun2021deep, yu2014click}.

The asymptotic equivalent solution to the elastic net regularized problem is given by
\begin{eqnarray}
\label{eq:elasticnetSolutionVector}
(\hat{\btheta}_3)_i =
    \begin{cases}
    \frac{2\tilde{c}_1\theta^*_i}{2\tilde{c}_1+\alpha} + \frac{\tilde{c}_2\sqrt{\gamma}}{(2\tilde{c}_1 + \alpha)}\phi_{i} - \frac{\lambda}{2\tilde{c}_1 + \alpha} & \phi_{i} < -\zeta_{1i} \\
    \frac{2\tilde{c}_1\theta^*_i}{2\tilde{c}_1+\alpha} + \frac{\tilde{c}_2\sqrt{\gamma}}{(2\tilde{c}_1 + \alpha)}\phi_{i}  + \frac{\lambda}{2\tilde{c}_1 + \alpha} & \phi_{i}  > \zeta_{2i}\\
    0  \qquad \qquad -\zeta_{1i} \leq \phi_{i} \leq \zeta_{21}
    \end{cases},
\end{eqnarray}
in which $\zeta_{1i}$ and $\zeta_{2i}$ are given by
\begin{eqnarray}
\zeta_{1i} = \frac{(\lambda - 2\tilde{c}_1\theta^*_i)}{\sqrt{\gamma}\tilde{c}_2} \quad \zeta_{2i} = \frac{(\lambda + 2\tilde{c}_1\theta^*_i)}{\sqrt{\gamma}\tilde{c}_2}
\end{eqnarray}
and $\tilde{c}_1, \tilde{c}_2$ are the constants described in Theorem \ref{thm:GaussianAsymptotics}. The solution may also be expressed more succinctly by means of a soft thresholding operator. A full derivation of this solution may be found in the supplement section \ref{app:sec:ElasticsNetExampleCase}. We note that in the limit of $\lambda \rightarrow 0$, we obtain $-\zeta_{1i} = \zeta_{2i}$ and the solution collapses into a single case, that being the result for ridge regression.

According to theorem  \ref{thm:mainresult:formal} and \ref{thm:l1Universality}, the characteristics of the solution vector $\hat{\btheta}_1$, reflected by a suitable function $h$, asymptotically becomes close to that of $\hat{\btheta}_3$ (the $\ell_1$ case is under sparsity condition). 
Here, we consider the  sparsity of the solution. For this reason, we take a separable function $h_\epsilon(\btheta)=\suml_i \bar h_\epsilon(\theta_i)$, where $\bar  h_\epsilon(\theta)$ is  a positive $C^\infty$ bump function such that $\bar h_\epsilon(0)=1$ and $\bar h_\epsilon(\theta)=0$ for $|\theta|>\epsilon$. Our results apply to this function and we note that
\begin{equation}
    n_0(\btheta)\leq h_\epsilon(\btheta)\leq n_\epsilon(\btheta),
\end{equation}
where $n_\epsilon(\btheta)$ is the number of the elements $\theta_i$ in $\btheta$ with $|\theta_i|\leq \epsilon$. In particular, $n_0$ is the number of zeros. We may show that by theorem  \ref{thm:mainresult:formal} and the law of large numbers, the value of $\frac 1m h_\epsilon(\btheta)$ converges in probability to a constant $s_\epsilon$ calculated by analyzing $\hat\btheta_3$. We refer to $s:=\liml_{\epsilon\to 0}s_\epsilon$ as the "effective sparsity" of $\hat\btheta_1$. Roughly speaking, $s$ counts not only the zero entries of $\hat\btheta_1$, but also the vanishing entries as the problem size grows. 

By direct calculation, we shown in the supplement section \ref{app:sec:ElasticsNetExampleCase} that
\begin{eqnarray}
s \to\frac{2\tilde{c}_1 + \alpha}{\sqrt{\gamma}\tilde{c}_2}\frac{1}{m}\suml_i\mathbb{E}\left[\left(\hat{\btheta}_3\right)_i\phi_i\right],
\end{eqnarray}  
where $\phi_i$ and $\left(\hat{\btheta}_3\right)_i$ are defined in  \eqref{eq:elasticnetSolutionVector}.
We note that for pure $\ell_1$ regularization this formula may still be used by setting $\alpha = 0$, although we can theoretically support it for small values of sparsity. In this case, $s=1-M_0$ where $M_0$ is given in \eqref{eq:eff}. Experimental results for effective sparsity maybe be found in the supplement \ref{App:simulationDetails:Sparsity}

\section{Experiments}

\subsection{Experimental setup}
\label{ExperimentalSetup}
Using the expressions derived in the previous section, we examine the case of elastic net regularization experimentally. We choose the $\tanh$ activation for the non linearity of the feature map. We consider a deterministic vector $\btheta^*$ that consists of half ones and half zeros, 
. We set the noise power $\sigma_{\bepsilon}^2 = 0.1$ and let $\delta = 1$. We consider multiple cases, where for each case we solve the problem $P_3$ (equation \eqref{eq:4doptproblem}) using an iterative refining grid search algorithm. We compare the results to an experimental simulation in which $n + m = 1000$, with the relative ratio varied for different values of $\gamma = m/n$. Each empirical data point was averaged over 100 random realizations of the weights $\bW$, and the data $\bz$. More details maybe found in appendix \ref{App:simulationDetails}.

\subsection{Elastic net model} 
We compare the experimental and theoretically derived values for training and generalization error of the elastic net model for two cases. Firstly we vary the ratio $\gamma = \frac{m}{n}$ for fixed values of the regularization parameters, and secondly we vary the regularization parameter $\lambda$ for all other parameters being fixed.

The case of varying $\gamma$ is shown in figure \ref{fig:elasticNetReg}. Here, we fix $\lambda = 10^{-3}$ and choose several values of $\alpha$ including $0$, the case of pure $\ell_1$ regularization. Our expressions accurately predict the expected behavior of a network, the small deviation explained by the fact that $n, m$ are finite. However, the discrepancy is only notable in a small range near the interpolation peak, suggesting the validity of our expressions in a wide range of networks of a non asymptotic size.
We observe that small values of $\alpha$ result in a spike in the generalization error at the interpolation threshold, which in this model, is slightly more than $\gamma = 1$. We note that as the regularization parameter increases in strength, the interpolation peak diminishes. This is consistent with other results on the study of the double descent phenomenon \citep{d2020double}.

In figure \ref{fig:elasticNetRegLamVary}, we choose $\alpha = 10^{-3}$ and vary the value of the regularization parameter $\lambda$ at constant $\gamma$. We note that that the generalization error suggests that at each ratio of $\gamma = \frac{m}{n}$ there is an optimal value of $\lambda$ that minimizes the expected error.


\begin{figure}[!ht]
    \centering
    
    \subfigure [Training Error] {
    \label{subfig:trainerrorGamma}
    \resizebox{0.45\textwidth}{!}{%
        \begin{tikzpicture}
        \begin{axis}[
          xlabel={$\gamma = \frac{m}{n}$},
          ylabel=Training Error,
          grid = both,
            minor tick num = 1,
            major grid style = {lightgray},
            minor grid style = {lightgray!25},
          ]
        \addplot[ smooth, thin, red] table[ y=train0, x=gamma]{Data/datatheory1.dat};
        \addlegendentry{$\alpha = 0$}
        \addplot[ smooth, thin, blue] table[y=train1e-4, x=gamma]{Data/datatheory1.dat};
        \addlegendentry{$\alpha = 10^{-4}$}
        \addplot[ smooth, thin, purple] table[y=train1e-3, x=gamma]{Data/datatheory1.dat};
        \addlegendentry{$\alpha = 10^{-3}$}
        \addplot[ smooth, thin, brown] table[y=train1e-2, x=gamma]{Data/datatheory1.dat};
        \addlegendentry{$\alpha = 10^{-2}$}
        \addplot[ smooth, thin, teal] table[y=train1e-1, x=gamma]{Data/datatheory1.dat};
        \addlegendentry{$\alpha = 10^{-1}$}
        \addplot[color = red, mark = square, mark size = 1pt, only marks] table[ y=train0, x=gamma]{Data/dataExp1.dat};
        \addplot[blue, mark = otimes, mark size = 1pt, only marks] table[y=train1e-4, x=gamma]{Data/dataExp1.dat};
        \addplot[purple, mark = otimes, mark size = 1pt, only marks] table[y=train1e-3, x=gamma]{Data/dataExp1.dat};
        \addplot[brown, mark = otimes, mark size = 1pt, only marks] table[y=train1e-2, x=gamma]{Data/dataExp1.dat};
        \addplot[teal, mark = triangle, mark size = 1pt, only marks] table[y=train1e-1, x=gamma]{Data/dataExp1.dat};
        \end{axis}
    \end{tikzpicture}
        } 
    }
\subfigure [Generalization Error] {
    \label{subfig:trainerrorGamma}
    \resizebox{0.45\textwidth}{!}{%
        \begin{tikzpicture}
        \begin{axis}[
          xlabel={$\gamma = \frac{m}{n}$},
          ylabel=Generalization Error,
          legend pos = north west,
          grid = both,
            minor tick num = 1,
            major grid style = {lightgray},
            minor grid style = {lightgray!25},
          ]
        \addplot[ smooth, thin, red] table[ y=gap0, x=gamma]{Data/datatheory1.dat};
        \addlegendentry{$\alpha = 0$}
        \addplot[ smooth, thin, blue] table[y=gap1e-4, x=gamma]{Data/datatheory1.dat};
        \addlegendentry{$\alpha = 10^{-4}$}
        \addplot[ smooth, thin, purple] table[y=gap1e-3, x=gamma]{Data/datatheory1.dat};
        \addlegendentry{$\alpha = 10^{-3}$}
        \addplot[ smooth, thin, brown] table[y=gap1e-2, x=gamma]{Data/datatheory1.dat};
        \addlegendentry{$\alpha = 10^{-2}$}
        \addplot[ smooth, thin, teal] table[y=gap1e-1, x=gamma]{Data/datatheory1.dat};
        \addlegendentry{$\alpha = 10^{-1}$}
        \addplot[color = red, mark = square, mark size = 1pt, only marks] table[ y=gap0, x=gamma]{Data/dataExp1.dat};
        \addplot[blue, mark = otimes, mark size = 1pt, only marks] table[y=gap1e-4, x=gamma]{Data/dataExp1.dat};
        \addplot[purple, mark = otimes, mark size = 1pt, only marks] table[y=gap1e-3, x=gamma]{Data/dataExp1.dat};
        \addplot[brown, mark = otimes, mark size = 1pt, only marks] table[y=gap1e-2, x=gamma]{Data/dataExp1.dat};
        \addplot[teal, mark = triangle, mark size = 1pt, only marks] table[y=gap1e-1, x=gamma]{Data/dataExp1.dat};
        \end{axis}
    \end{tikzpicture}
        } 
    }
       \caption{Theoretically predicted (solid line) and numerically determined (markers) values of the training error (a) and generalization error (b) for the random features model with $\ell_1 + \ell_2$ regularization as a function of $\gamma =\frac{m}{n}$, for varying values of regularization strengths of $\alpha$ at constant value of $\lambda = 10^{-3}$.}
    \label{fig:elasticNetReg}
\end{figure}
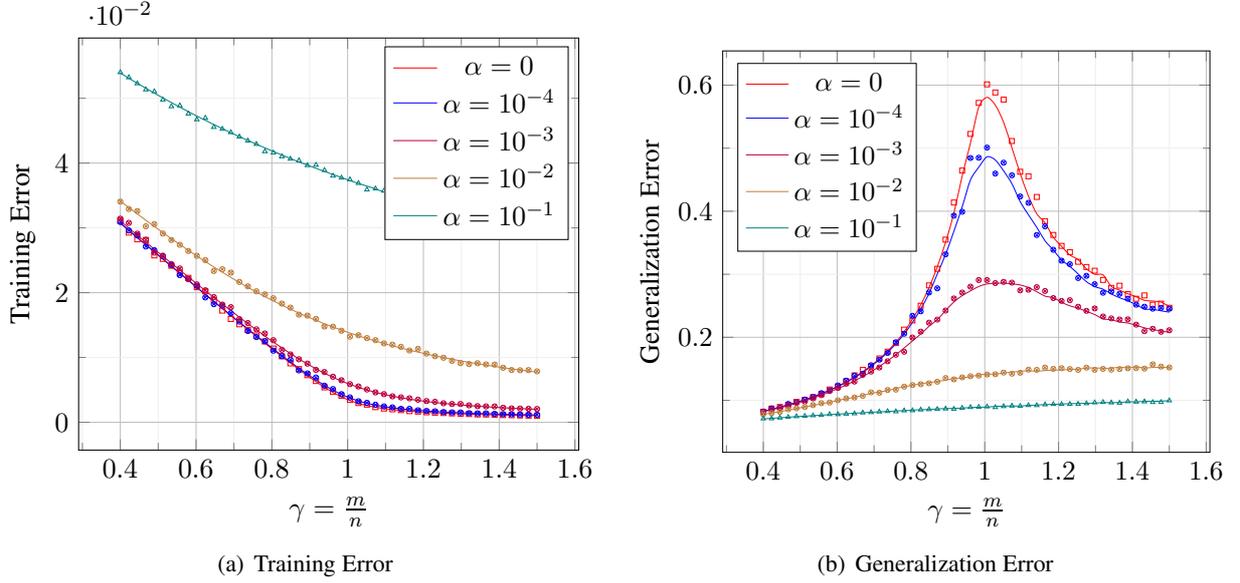

\begin{figure}[!ht]
    \centering
    
\subfigure [Training Error] {
    \label{subfig:trainErrorLam}
    \resizebox{0.45\textwidth}{!}{%
        \begin{tikzpicture}
        \begin{axis}[
          xlabel={$\lambda$},
          ylabel=Training Error,
          legend pos = north west,
          xmode=log,
          grid ,
            minor tick num = 1,
            major grid style = {lightgray},
            minor grid style = {lightgray!25},
          ]
          \addplot[ smooth, thin, red] table[y=train0.3, x=lam]{Data/datatheory2.dat};
        \addlegendentry{$\gamma = 0.3$}
        \addplot[ smooth, thin, green] table[y=train0.6, x=lam]{Data/datatheory2.dat};
        \addlegendentry{$\gamma = 0.6$}
        \addplot[ smooth, thin, blue] table[y=train0.9, x=lam]{Data/datatheory2.dat};
        \addlegendentry{$\gamma = 0.9$}
        \addplot[ smooth, thin, purple] table[y=train1.2, x=lam]{Data/datatheory2.dat};
        \addlegendentry{$\gamma = 1.2$}
        \addplot[ smooth, thin, brown] table[y=train1.5, x=lam]{Data/datatheory2.dat};
        \addlegendentry{$\gamma = 1.5$}
         \addplot[color = red, mark = triangle, mark size = 1pt, only marks] table[ y=train0.3, x=lam]{Data/dataExp2.dat};
        \addplot[color = green, mark = square, mark size = 1pt, only marks] table[ y=train0.6, x=lam]{Data/dataExp2.dat};
        \addplot[blue, mark = otimes, mark size = 1pt, only marks] table[y=train0.9, x=lam]{Data/dataExp2.dat};
        \addplot[purple, mark = triangle, mark size = 1pt, only marks] table[y=train1.2, x=lam]{Data/dataExp2.dat};
         \addplot[color = brown, mark = square, mark size = 1pt, only marks] table[ y=train1.5, x=lam]{Data/dataExp2.dat};
        \end{axis}

    \end{tikzpicture}
        } 
    }
\subfigure [Generalization Error] {
    \label{subfig:genErrorLam}
    \resizebox{0.45\textwidth}{!}{%
        \begin{tikzpicture}
        \begin{axis}[
          xlabel={$\lambda$},
          ylabel=Generalization Error,
          legend pos = north east,
          xmode=log,
          grid,
            minor tick num = 1,
            major grid style = {lightgray},
            minor grid style = {lightgray!25},
          ]
          \addplot[ smooth, thin, red] table[y=gap0.3, x=lam]{Data/datatheory2.dat};
        \addlegendentry{$\gamma = 0.3$}
        \addplot[ smooth, thin, green] table[y=gap0.6, x=lam]{Data/datatheory2.dat};
        \addlegendentry{$\gamma = 0.6$}
        \addplot[ smooth, thin, blue] table[y=gap0.9, x=lam]{Data/datatheory2.dat};
        \addlegendentry{$\gamma = 0.9$}
        \addplot[ smooth, thin, purple] table[y=gap1.2, x=lam]{Data/datatheory2.dat};
        \addlegendentry{$\gamma = 1.2$}
        \addplot[ smooth, thin, brown] table[y=gap1.5, x=lam]{Data/datatheory2.dat};
        \addlegendentry{$\gamma = 1.5$}
        \addplot[red, mark = triangle, mark size = 1pt, only marks] table[y=gap0.3, x=lam]{Data/dataExp2.dat};
      \addplot[color = green, mark = square, mark size = 1pt, only marks] table[ y=gap0.6, x=lam]{Data/dataExp2.dat};
        \addplot[blue, mark = otimes, mark size = 1pt, only marks] table[y=gap0.9, x=lam]{Data/dataExp2.dat};
        \addplot[purple, mark = triangle, mark size = 1pt, only marks] table[y=gap1.2, x=lam]{Data/dataExp2.dat};
        \addplot[color = brown, mark = square, mark size = 1pt, only marks] table[ y=gap1.5, x=lam]{Data/dataExp2.dat};
        \end{axis}

    \end{tikzpicture}
        } 
    }
      \caption{Theoretically predicted (solid line) and numerically determined (markers) values of the training error (a) and generalization error (b) for the random features model with $\ell_1 + \ell_2$ regularization as a function of the regularization parameter $\lambda$, for varying values of the ratio $\gamma = \frac{m}{n}$ constant value of $\alpha = 10^{-3}$.}
    \label{fig:elasticNetRegLamVary}
\end{figure}
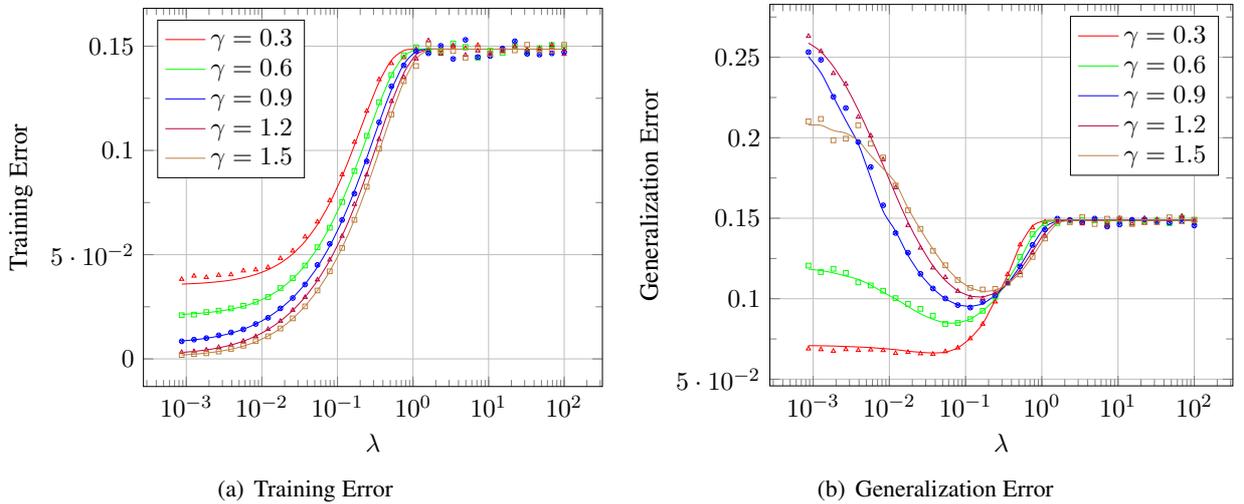

\section{CONCLUSION}
We derived expressions to determine the exact asymptotic learning curves for square loss random feature models, subject to strongly convex regularization, or $\ell_1$ regularization. These expressions consist of a 4-dimensional scalar optimization with two min-max pairs that is computable using standard techniques. We proved in two steps that these expressions coincide with the asymptotic learning curves: First, we demonstate that the scalar optimization is asymptotically equivalent to a surrogate Gaussian model whose first two moments match that of the RF models. For this, we proposed a novel multi-stage application of the CGMT. Then, we extended the results of the universality of RF models to a broader family, including elastic net and $\ell_1$ regularization, thereby demonstrating an asymptotic equivalence between the Gaussian model and the non linear RF model. Our results for universality hold not only for the cases of training and generalization error, but also for test functions $h$ from a wide family. 

There are several potential directions to extend our study. A particularly interesting direction is to use  our methodology to obtain refined expressions for a more generic loss functions, extending the existing studies,  e.g \cite{Loureiro2021Learning}. 

\bibliography{bib}

\appendix
\onecolumn

\section{Gaussian Min Max Theorems}
\renewcommand{\Pr}{\mathbb{P}}
\label{app:CGMT}
We make use of the Gaussian Min max theorem as well as the Convex Gaussian min max theorem in the proof of theorem 1. The Gaussian min max theorem was originally proven by Gordon \citep{gordon1985some, gordon1988milman}. The CGMT was developed by \citep{thrampoulidis2015gaussian}, we state the theorem here for completeness.

The Gaussian Min Max theorem states the following:
\begin{theorem}
\label{thm:GMT}
Let $\bA\in\mathbb{R}^{m\times n}, g\in\mathbb{R}, \bg \in \mathbb{R}^{m}$ and $\bh\in\mathbb{R}^{n}$ be independent of each other and have entries distributed i.i.d according to $\mathcal{N}(0, 1)$. Let $\mathcal{S}_1\subset \mathbb{R}^n$ and $\mathcal{S}_2\subset\mathbb{R}^m$ be nonempty compact sets. Let $f(\cdot, \cdot)$ we a continuous function on $\mathcal{S}_1\times\mathcal{S}_2$. We define
\begin{eqnarray}
\bP_1(\bA) :=& \minl_{\bx\in\mathcal{S}_1}\maxl_{\by\in\mathcal{S}_2} \by^T\bA\bx + g||\bx||_2||\by||_2+ f(\bx, \by),\\
\bP_2(\bg, \bh) :=&\minl_{\bx\in\mathcal{S}_1}\maxl_{\by\in\mathcal{S}_2} ||\bx||_2\bg^T\by + ||\by||_2\bh^T\bx + f(\bx, \by).
\end{eqnarray} 
Then for any $c\in\mathbb{R}$:
\begin{eqnarray}
\Pr(\bP_1(\bA, g) \leq c) \leq \Pr(\bP_2(\bg, \bh) \leq c)
\end{eqnarray}
\end{theorem}
The Convex Gaussian Min Max theorem extends these results to the following:
\begin{theorem}
\label{thm:CGMT}
Let $\bA\in\mathbb{R}^{m\times n}, \bg \in \mathbb{R}^{m}$ and $\bh\in\mathbb{R}^{n}$ be independent of each other and have entries distributed i.i.d according to $\mathcal{N}(0, 1)$. Let $\mathcal{S}_1\subset \mathbb{R}^n$ and $\mathcal{S}_2\subset\mathbb{R}^m$ be nonempty compact sets. Let $f(\cdot, \cdot)$ we a continuous function on $\mathcal{S}_1\times\mathcal{S}_2$. We define
\begin{eqnarray}
\bP_1(\bA) :=& \minl_{\bx\in\mathcal{S}_1}\maxl_{\by\in\mathcal{S}_2} \by^T\bA\bx+ f(\bx, \by),\\
\bP_2(\bg, \bh) :=&\minl_{\bx\in\mathcal{S}_1}\maxl_{\by\in\mathcal{S}_2} ||\bx||_2\bg^T\by + ||\by||_2\bh^T\bx + f(\bx, \by).
\end{eqnarray} 
Then for any $c_1\in\mathbb{R}$ we have that
\begin{eqnarray}
\mathbb{P}(\bP_1(\bA)<c_1) \leq 2\mathbb{P}(\bP_2(\bg, \bh)\leq c_1),
\end{eqnarray}
Under the further assumptions that $S_1$ and $S_2$ are convex sets and $f$ is concave-convex on $S_1\times S_2$ then for all $c_2\in\mathbb{R}$
\begin{eqnarray}
\mathbb{P}(\bP_1(\bA)>c_2) \leq 2\mathbb{P}(\bP_2(\bg, \bh)\geq c_2).
\end{eqnarray}
\end{theorem}

We note that if in the limit of $n, m \rightarrow \infty$ the value of $\bP_2(\bg, \bh)$ concentrates on a value $a$ then similarly $\bP_1(\bA)$ converges to the same value.

\section{Proof of Theorem \ref{thm:GaussianAsymptotics}}
\label{sec:GuassianTheoremProof}

 To prove, theorem \ref{thm:GaussianAsymptotics}, we shall apply the CGMT (supplement theorem \ref{thm:CGMT}) to obtain an alternative problem formulation for \eqref{eq:GaussianKeyOptimization}. Subsequently, we will simplify the alternative problem, and then express it once again in the form that is suitable for a second CGMT application. Applying the CGMT for a second time, we obtain a second alternative problem. After simplifying this second alternative problem, we will demonstrate the results in Theorem \ref{thm:GaussianAsymptotics}. To begin with the first application of the CGMT, we fix $\bW$ and change the variable $\btheta$ in 
 \eqref{eq:GaussianKeyOptimization} to $\be=\btheta-\btheta^*$ to obtain
 \begin{eqnarray}
\label{eq:GaussianKeyOptimization_e}
\tilde P_2(\tau_1, \tau_2) = \min_{\be} \frac{1}{2n}||\bepsilon - \frac{1}{\sqrt{m}}\tilde{\bX}\be||_2^2 + \frac{1}{m}r(\be+\btheta^*)
    + \frac{\tau_1}{m}\be^T\tilde{\bR}\be + \frac{\tau_2}{m} h(\be+\btheta^*),
\end{eqnarray}
 note that the rows $\tlbx_i$ of $\tlbX$ are i.i.d, centered and Gaussian with the covariance matrix $\tlbR=\frac{\rho^2_1}{d}\bW^T\bW+\rho_*^2 I$. Hence, we may write $\tlbX=\bU\tlbR^{\frac 12}$ where $\bU$ has i.i.d. standard Gaussian entries. Next, using the Legendre transform of the square function, we may
 write \eqref{eq:GaussianKeyOptimization} as
\begin{eqnarray}
    \tlP_2(\tau_1,\tau_2) = \min_{\be}\max_{\blambda} \frac{1}{n}\blambda^T\bepsilon - \frac{1}{n\sqrt{m}}\blambda^T\bU\tilde{\bR}^{1/2}\be -\frac{1}{2n}||\blambda||_2^2 + \frac{1}{m}r\left(\be + \btheta^*\right) + \frac{1}{m}\tau_1\be^T\tilde{\bR}\be + \frac{1}{m}\tau_2h(\be+\btheta^*)
\end{eqnarray}
In here $\tau_1\in\mathbb{R}$ and $\tau_2\in\mathbb{R}$ are constants and by the assumption, $r(\be+\btheta^*)+\tau_2 h(\be+\btheta^*)$ is $\frac{\mu} 2-$strongly convex. We require that  $\tau_1$ is chosen  sufficiently small, to ensure that the entire optimization problem remains strongly convex in $\be$.  In particular, we ensure that the term $B(\be) := r(\be + \btheta^*) + \tau_1\be^T\bR\be + \tau_2h(\be+\btheta^*)$ is $\frac{\mu}{4}$-strongly convex.
First, we show that
\begin{eqnarray}\label{eq:tau_1}
    |\tau_1| \leq \tau_1^* = \frac{\mu/8}{\rho_1^2(1+2\sqrt{\delta})^2 + \rho_*^2}
\end{eqnarray}
will satisfy this condition with high probability\footnote{Throughout this paper, the term "high probability" means a probability converging to $1$ as the problem size grows.}. For this reason, we introduce the following lemma:

\begin{lemma}
\label{lem:Rboundlemma}
    Define $C_{\tilde{\bR}}=\rho_1^2\left(1+2\sqrt{\delta}\right)^2+\rho_*^2$.
    For a random matrix $\bW$ with i.i.d. standard Gaussian entries and $\tlbR = \frac{\rho_1^2}{d}\bW\bW^T + \rho_*^2\bI$ , the following relation holds:
    \begin{eqnarray}
        \Pr\left[\|\tlbR\|_2 > C_{\tilde{\bR}}\right]<2e^{-cm}
    \end{eqnarray}
\end{lemma}
for a universal constant $c>0$, where $||\cdot||_2$ denotes the spectral norm.
\begin{proof}
We note that by the definition of $\tlbR$, we have that
\begin{eqnarray}
    \|\tlbR\|_2 = \left\|\frac{\rho_1^2}{d}\bW\bW^T + \rho_*^2\bI\right\|_2 \leq \frac{\rho_1^2}{d}\|\bW\|_2^2 + \rho_*^2.
\end{eqnarray}
The elements of $\bW\in\mathbb{R}^{m\times d}$ are i.i.d normally distributed. From a standard result in matrix theory \citep{papaspiliopoulos2020high}[Corollary 7.3.3] we obtain
\begin{eqnarray}
\mathbb{P}(\frac{1}{\sqrt{d}}||\bW||_2 \geq 1 + \sqrt{m/d} + t) \leq 2e^{-cdt^2}.
\end{eqnarray}
 Choosing $t = \sqrt{m/d}$ yields
\begin{eqnarray}
\Pr(\frac{1}{\sqrt{d}}||\bW||_2 \geq 1+ 2\sqrt{\delta})\leq 2e^{-cm},
\end{eqnarray}
where we recall that $\delta = \frac{m}{d}$. This provides the desired result.
\end{proof}
    According to lemma \ref{lem:Rboundlemma}, the term $\be^T\tlbR\be$ is $2C_\bR-$smooth, and hence for $\tau_1\leq \frac{\frac\mu 4}{2C_\bR}$, the term $B$ is $\frac\mu 2-\frac\mu 4=\frac\mu 4-$convex. This is the same as the condition in \eqref{eq:tau_1}. Hence, in the rest of this proof we assume that $B$ is strongly convex.
    

Next, we note that applying the CGMT requires that both $\blambda$ and $\be$ are in compact feasibility sets. Here we employ a similar strategy to \citep{ThrampoulidisMestimators, LuPreciseRandom, Loureiro2021Learning} by showing that with high probability, the solutions of both the original problem and the alternative problem can be bound in fixed compact sets, hence restricting the optimizations to these sets will not affect the result. As a result, we may apply the CGMT.

\begin{lemma}\label{lemma:bounded_CGMT1}
    Consider the following two optimization problems that correspond to the primary optimization and to the first alternative optimization in CGMT.
    \begin{eqnarray}
    \label{eq:compactsetproblem1}
    \tlP_{2,1} = \min_{\be}\max_{\blambda}\frac{1}{n}\blambda^T\bepsilon - \frac{1}{n\sqrt{m}}\blambda^T\tilde{\bX}\be -\frac{1}{2n}||\blambda||_2^2 + \frac{1}{m}B(\be) \\
    \label{eq:compactsetproblem2}
    \tlP_{2,2} = \min_{\be}\max_{\blambda}\frac{1}{n}\blambda^T\bepsilon - \frac{1}{n\sqrt{m}}||\tilde{\bR}^{1/2}\be||_2\blambda^T\bg - \frac{1}{n\sqrt{m}}||\blambda||_2\bh^T\tilde{\bR}^{1/2}\be -\frac{1}{2n}||\blambda||_2^2 + \frac{1}{m}B(\be)
    \end{eqnarray}
    In these equations $\bg, \bh$ are  standard normal vectors of size $n, m$, respectively. Denote by $\tlbe_{2,1}, \tlbe_{2,2}$ the optimal solutions of $\tlP_{2,1}$ and $\barP_{2,2}$, respectively. Furthermore, respectively denote by $\tlblambda_1(\be), \tlblambda_2(\be)$ the solution of their inner optimization (over $\blambda$) for a given vector $\be$. 
    Let $B$ be strongly convex with constant $\frac\mu 4$ and $\max\left\{\|\nabla r(\btheta^*)\|, \|\nabla h(\btheta^*)\|\right\}=O(\sqrt{m})$. Then, there exist positive constants $C_{\be}, C_{\blambda}$ only depending on $\mu$ such that the following hold true:
    \begin{itemize}
        \item The solutions $\tlbe_{2,i}$ for $i=1,2$ satisfy
        \begin{eqnarray}
    \liml_{m\to \infty}\Pr\left(\max\left\{||\tlbe_{2,1}||_2,||\barbe_{2,2}||_2\right\} \leq C_{\be}\sqrt{m}\right).
    \end{eqnarray}
    \item It also holds that
    \begin{equation}
        \liml_{m\to\infty}\Pr\left(\supl_{\be\mid\|\be\|\leq C_{\be}\sqrt{m}}\max\left\{\|\tlblambda_1(\be)\|,\|\tlblambda_2(\be)\|\right\}\leq C_{\blambda}\sqrt{m}\right)=1
    \end{equation}
    \end{itemize}
\end{lemma}
\begin{proof}
We note that $B$ is $\frac\mu 4$ strongly convex. Solving for $\blambda$ in both optimization, we may write the optimization over $\be$ as
\begin{eqnarray}
    \min_{\be} F_i(\be) \qquad i = 1,2
\end{eqnarray}
Where $F_i(\be)$ is the optimal value over $\blambda$. We note that setting $\blambda = \bm{0}$ in both optimizations, we obtain that $F(\be) \geq \frac 1mB(\be)$. Then we see that
\begin{eqnarray}
    B(\be) \geq B(\bm{0}) + \bd^T\be + \frac{\mu}{4}||\be||_2^2
\end{eqnarray}
where $\bd = \nabla B(0)=\nabla r(\btheta^*)+\tau_2 \nabla h(\btheta^*)$ and by the assumption $\|\bd\|=O(\sqrt{m})$.

For optimization $P_1$, we note that
\begin{eqnarray}
    F(\bm{0}) = \frac{1}{m}B(\bm{0}) + \frac{1}{2n}\left\|\bepsilon  \right\|_2^2.
\end{eqnarray}
This implies that for the optimal solution $\hat{\be}$ we have
\begin{eqnarray}
\label{eq:midd}
 \frac{1}{m}B(\bm{0}) + \frac{1}{2n}\left\|\bepsilon  \right\|_2^2 = F(\bm{0}) \geq F(\tlbe_1) \geq \frac{1}{m}B(\bm{0}) + \frac{1}{m}\bd^T\tlbe_1 + \frac{\mu}{4m}||\tlbe_1||_2^2,
\end{eqnarray}
which yields
\begin{eqnarray}
\frac{\mu}{4m}\left\|\tlbe_1 + \frac{1}{\mu}\bd\right\|^2 \leq \frac{1}{2n}||\bepsilon||_2^2 + \frac{1}{4\mu m}||\bd||_2^2.
\end{eqnarray}
Then, we obtain
\begin{eqnarray}
||\tlbe_1||_2 \leq \left\|\frac{1}{\mu}\bd\right\|_2  + \sqrt{\frac{2m}{n\mu}||\bepsilon||_2^2 + \frac{1}{\mu^2}||\bd||_2^2}.
\end{eqnarray}
From the standard matrix theory \citep{papaspiliopoulos2020high}[Theorem 2.8.1] we know that $||\bepsilon||_2^2 < c n$ for some $c$, with high probability. We observe that there must exist some constant $C_{\be_1}$ such that
\begin{eqnarray}
    \lim_{m\rightarrow\infty}\mathbb{P}(||\tlbe_1||_2 \geq C_{\be_1}\sqrt{m}) = 0.
\end{eqnarray}
Now we consider \eqref{eq:compactsetproblem2}. Our strategy is similar to the previous case. We note that if we let $\beta = ||\blambda||_2$ we can solve the optimization over $\blambda$ to obtain:
\begin{eqnarray}
    F(\be) = \max_{\beta \geq 0} \frac{\beta}{n}\left\|\bepsilon - \frac{1}{\sqrt{m}}\|\tilde{\bR}^{1/2}\be\|_2\bg \right\| - \frac{\beta}{n\sqrt{m}}\bh^T\tilde{\bR}^{1/2}\be - \frac{\beta^2}{2n} + \frac{1}{m}B(\be).
\end{eqnarray}
The optimization is limited to $\beta \geq 0$. Hence, its optimal value will be increased when the constant is lifted, leading to a quadratic optimization and the following result
\begin{eqnarray}
    F(\be) \leq \frac{1}{m}B(\be) + \frac{1}{2n}\left(\left\|\bepsilon - \frac{1}{\sqrt{m}}\|\tilde{\bR}^{1/2}\be\|_2\bg \right\| - \frac{\beta}{\sqrt{m}}\bh^T\tilde{\bR}^{1/2}\be \right)^2,
\end{eqnarray}
and in particular
\begin{eqnarray}
F(\bm{0}) \leq \frac{1}{m}B(\bm{0}) + \frac{1}{2n}||\bepsilon||_2^2.
\end{eqnarray}
By applying the same inequality as in \eqref{eq:midd} we obtain that
\begin{eqnarray}
    ||\tlbe_2||_2 \leq \left\| \frac{1}{\mu}\bd\right\| + \sqrt{ \frac{2m}{n\mu}\|\bepsilon\|_2^2 + \frac{1}{\mu^2}\|\bd\|_2^2  }.
\end{eqnarray}
As such, we obtain 
\begin{eqnarray}
    \lim_{m\rightarrow\infty}\mathbb{P}(||\tlbe_2|| \geq C_{\be_2}\sqrt{m}) = 0.
\end{eqnarray}
Now, let $C_{\be} = \max(C_{\be_1}, C_{\be_2})$, and use this to define the set $A_{\be} = \lbrace \be \in \mathbb{R}^m|\ ||\be||_2 \leq C_{\be}\sqrt{m} \rbrace$.

Next, we note from the optimality condition of the inner optimization in Eq. \eqref{eq:compactsetproblem1} that   
\begin{eqnarray}
    \tlblambda_1(\be) = \bepsilon - \frac{1}{\sqrt{m}}\bU\tilde{\bR}^{1/2}\be.
\end{eqnarray}
As such, for all $\be \in A_{\be}$ we have
\begin{eqnarray}
    ||\tlblambda_1(\be)||_2 \leq ||\bepsilon||_2 + \left\|\frac{1}{\sqrt{m}}\bU\tilde{\bR}^{1/2}\right\|\|\be\|_2 \leq \|\bepsilon\|_2 + \left\|\frac{1}{\sqrt{m}}\bU \right\|_2\|\tilde{\bR}^{1/2}\|_2\|\be\|_2
\end{eqnarray}
We note from lemma \ref{lem:Rboundlemma} that $\|\bR^{1/2}\|_2$ is bounded, and we make use of standard random matrix theory to conclude $\|\frac{1}{\sqrt{m}}\bU\|_2 < C$ with high probability. Then, making use of the same arguments as before we can see that there must exist a constant $C_{\blambda_1}$ such that for all $\be \in A_{\be}$
\begin{eqnarray}
    \lim_{n\rightarrow\infty}\mathbb{P}\left(\supl_{\be \in A_{\be}}||\tlblambda_1(\be)||_2 \geq C_{\blambda_1}\sqrt{n}\right) = 0
\end{eqnarray}

Finally we note that the optimality condition over $\beta$ of problem \ref{eq:compactsetproblem2} gives that for all $\be \in A_{\be}$
\begin{eqnarray}
    \hat{\beta} = \|\tlblambda_1(\be)\|_2 = \left\|\bepsilon - \frac{1}{\sqrt{m}}\|\tilde{\bR}^{1/2}\be\|_2\bg \right\|_2 - \frac{1}{\sqrt{m}}\tilde{\bR}^{1/2}\bh \nwl
    \leq \|\bepsilon\|_2 + \frac{1}{\sqrt{m}}\|\bg\|_2\|\tilde{\bR}^{1/2}\|_2\|\be\|_2 +  \frac{1}{\sqrt{m}}\|\tilde{\bR}^{1/2}\|_2\|\bh\|_2
\end{eqnarray}
We note that with high probability $\|\bepsilon\|_2 < C\sqrt{n}$, $\|\bg\|_2 < C\sqrt{n}$ and $||\bh||_2 < C\sqrt{m}$. From this we can see that there exists a constant $C_{\blambda_2}$ such that
\begin{eqnarray}
    \lim_{n\rightarrow\infty}\mathbb{P}\left(\supl_{\be \in A_{\be}}||\tlblambda_2(\be)||_2 \geq C_{\blambda_2}\sqrt{n}\right) = 0
\end{eqnarray}
Taking $C_{\blambda} = \max(C_{\blambda_1}, C_{\blambda_2})$ completes the proof.
\end{proof}

We use the definition of the sets $A_{\be} = \lbrace\be\in\mathbb{R}^m|\ ||\be||_2 \leq C_{\be}\sqrt{m} \rbrace$ and $A_{\blambda} = \lbrace\blambda\in\mathbb{R}^n|\ ||\blambda||_2 \leq C_{\blambda}\sqrt{n} \rbrace$ in the rest of this study. By the lemma above, we can with high probability, restrict ourselves to the following problem
\begin{eqnarray}
    \tilde{P}_{2,1}'(\tau_1, \tau_2) = \min_{\be\in A_{\be}}\max_{\blambda \in A_{\blambda}} \frac{1}{n}\blambda^T\bepsilon - \frac{1}{n\sqrt{m}}\blambda^T\bU\tilde{\bR}^{1/2}\be -\frac{1}{2n}||\blambda||_2^2 + \frac{1}{m}B(\be),
\end{eqnarray}
and be certain that the solution vector and the optimal value to the problem $\tilde{P}_2$ will be equal to those of the problem $\tilde{P}_{2,1}'$. We now make use of the CGMT, (Thm. \ref{thm:CGMT}). From which we obtain the following optimization problem
\begin{eqnarray}
    \tlP_{2,2}' = \min_{\be\in A_{\be}}\max_{\blambda \in A_{\blambda}}\frac{1}{n}\blambda^T\bepsilon - \frac{1}{n\sqrt{m}}||\tilde{\bR}^{1/2}\be||_2\bg^T\blambda - \frac{1}{n\sqrt{m}}||\blambda||_2\bh^T\tilde{\bR}^{1/2}\be  - \frac{1}{2n}||\blambda||_2^2 + \frac{1}{m}B(\be)
\end{eqnarray}
In which $\bg\sim\mathcal{N}(0, I_n)$ and $\bh\sim\mathcal{N}(0, I_m)$. Note that by lemma \eqref{lemma:bounded_CGMT1}, $\tilde{P}_{2,1}'$ is also identical to $\tilde{P}_{2,2}$. We now let $\beta = \frac{1}{\sqrt{n}}||\blambda||_2$. We further note that $0 \leq \beta \leq \beta_{max}$ in which $\beta_{max}$ can be arbitrarily larger than $C_{\blambda}$. We can solve the optimization over $\blambda$ to obtain

\begin{eqnarray}
    A_2 = \min_{\be \in A_{\be}}\max_{0\leq \beta \leq \beta_{max}} \beta\left\|\frac{1}{\sqrt{n}}\bepsilon - \frac{1}{\sqrt{nm}}\|\tilde{\bR}^{1/2}\be\|_2\bg \right\|_2 - \frac{\beta}{\sqrt{nm}}\bh^T\tilde{\bR}^{1/2}\be - \frac{\beta^2}{2} + \frac{1}{m}B(\be)
\end{eqnarray}

We now note that the first term of this problem concetrates. We prove this in the following lemma

\begin{lemma}
    \label{lem:A2concentrationLemma}
    Consider the term 
    \begin{eqnarray}
F(\be, \beta) =\beta\left\|\frac{1}{\sqrt{n}}\bepsilon - \frac{1}{\sqrt{nm}}\|\tilde{\bR}^{1/2}\be\|_2\bg \right\|_2 - \frac{\beta}{\sqrt{nm}}\bh^T\tilde{\bR}^{1/2}\be - \frac{\beta^2}{2} + \frac{1}{m}B(\be)
    \end{eqnarray}
and let $\bar{F}$ be given by
\begin{eqnarray}
\bar{F}(\be, \beta) = \beta\sqrt{\sigma_{\bepsilon}^2 - \frac{1}{m}\|\tilde{\bR}^{1/2}\be\|_2^2}  - \frac{\beta}{\sqrt{nm}}\bh^T\tilde{\bR}^{1/2}\be - \frac{\beta^2}{2} + \frac{1}{m}B(\be)
\end{eqnarray}
Then there exist positive constants $C, c$ such that for any $\epsilon > 0$.
\begin{eqnarray}
\Pr\left(\sup_{\be\in A_{\be}, 0\leq\beta\leq\beta_{max}}|F(\be, \beta) - \bar{F}(\be,\beta)| \geq \epsilon \right) \leq Ce^{-cn\epsilon}
\end{eqnarray}

\end{lemma}
\begin{proof}
We see that $F$ can be expressed as
\begin{eqnarray}
    F = \beta\sqrt{ \frac{1}{n}\|\bepsilon\|_2^2 + \frac{1}{nm}\|\tilde{\bR}^{1/2}\be\|_2^2\|\bg\|_2^2 - \frac{1}{n\sqrt{m}}\|\tilde{\bR}^{1/2}\be\|_2\bepsilon^T\bg }\nwl- \frac{\beta}{\sqrt{nm}}\bh^T\tilde{\bR}^{1/2}\be - \frac{\beta^2}{2} + \frac{1}{m}B(\be)\nwl
\end{eqnarray}
Or equivalently
\begin{eqnarray}
    F = \beta\sqrt{ \left(\frac{1}{n}\|\bepsilon\|_2^2 - \sigma_{\bepsilon}^2\right) + \sigma_{\bepsilon}^2 + \frac{1}{m}\|\tilde{\bR}^{1/2}\be\|_2^2\left(\frac{1}{n}\|\bg\|_2^2 - 1\right) + \frac{1}{m}||\tilde{\bR}^{1/2}\be||_2^2 - \frac{2}{\sqrt{m}}\|\tilde{\bR}^{1/2}\be\|_2\frac{\bepsilon^T\bg}{n} }\nwl- \frac{\beta}{\sqrt{nm}}\bh^T\tilde{\bR}^{1/2}\be - \frac{\beta^2}{2} + \frac{1}{m}B(\be)\nwl
    \leq \bar{F} + \beta\sqrt{\delta} \leq \bar{F} + \beta_{max}\sqrt{\bar\delta}
\end{eqnarray}
in which
\begin{eqnarray}
\delta =  \left(\frac{1}{n}\|\bepsilon\|_2^2 - \sigma_{\bepsilon}^2\right) + \frac{1}{m}\|\tilde{\bR}^{1/2}\be\|_2^2\left(\frac{1}{n}\|\bg\|_2^2 - 1\right)- \frac{2}{\sqrt{m}}\|\tilde{\bR}^{1/2}\be\|_2\frac{\bepsilon^T\bg}{n} \nwl
\leq \left(\frac{1}{n}\|\bepsilon\|_2^2 - \sigma_{\bepsilon}^2\right) + C_{\be}^2C_{\tilde{\bR}}\left(\frac{1}{n}\|\bg\|_2^2 - 1\right)+ 2\sqrt{C_{\tilde{\bR}}}C_{\be}\left|\frac{\bepsilon^T\bg}{n}\right| \overset{def}{=} \bar{\delta }
\end{eqnarray}

We know that $C_{\tilde{\bR}}$ and $C_{\be}$ are universal constants. It is also clear that the probability that $\Pr(|\bar{\delta}| \geq \epsilon) \leq Ce^{-cn\epsilon}$ for some constants $C, c>0$. From this we can see that

\begin{eqnarray}
\Pr\left(\sup_{\be\in A_{\be}, 0\leq\beta\leq\beta_{max}}|F(\be, \beta) - \bar{F}(\be,\beta)| \geq \epsilon \right) \leq \Pr\left(\sup_{\be\in A_{\be}, 0\leq\beta\leq\beta_{max}}|\delta\beta| \geq \epsilon \right) \leq \Pr\left(|\beta_{max}\bar{\delta}| \geq \epsilon \right) \leq Ce^{-cn\epsilon}
\end{eqnarray}
For some constants $C, c > 0$.

\end{proof}

Because of this we can with high probability, examine instead the problem

\begin{eqnarray}
    \bar{P}_2 = \min_{\be \in A_{\be}}\max_{0\leq \beta \leq \beta_{max}} \beta\sqrt {\sigma_{\bepsilon}^2 + \frac{1}{m}\|\tilde{\bR}^{1/2}\be\|_2^2} - \frac{\beta}{\sqrt{nm}}\bh^T\tilde{\bR}^{1/2}\be - \frac{\beta^2}{2} + \frac{1}{m}B(\be)
\end{eqnarray}

We now note that the optimization problem is convex in $\be$ and concave in $\beta$, and both optimizations are over convex sets, as such we can interchange the order of $\min$ and $\max$
\begin{eqnarray}
    \bar{P}_2 = \max_{0\leq \beta \leq \beta_{max}}\min_{\be \in A_{\be}} \beta\sqrt {\sigma_{\bepsilon}^2 + \frac{1}{m}\|\tilde{\bR}^{1/2}\be\|_2^2} - \frac{\beta}{\sqrt{nm}}\bh^T\tilde{\bR}^{1/2}\be - \frac{\beta^2}{2} + \frac{1}{m}B(\be)
\end{eqnarray}

We now note that for any scalar value $a$, we can express $\sqrt{a} = \min_{q>0} \frac{q}{2} + \frac{a}{2q}$. Making use of this "square root trick" we can obtain the problem

\begin{eqnarray}
    \bar{P}_2 = \max_{0\leq \beta \leq \beta_{max}}\min_{\be \in A_{\be}}\min_{q_{min}\leq q \leq q_{max}} \frac{\beta q}{2} + \frac{\beta}{2q}\sigma_{\bepsilon}^2 + \frac{\beta}{2qm}\|\tilde{\bR}^{1/2}\be\|_2^2- \frac{\beta}{\sqrt{nm}}\bh^T\tilde{\bR}^{1/2}\be - \frac{\beta^2}{2} + \frac{1}{m}B(\be)
\end{eqnarray}
We note also that $q$ can be bounded between $q_{min} = \sigma_{\bepsilon}$, which is obtained when $\be = \bm{0}$ and $q_{\max} = \sqrt{\sigma_{\bepsilon}^2 + C_{\tilde{\bR}}C_{\be}^2}$. We can also swap the order of the two $\min$s obtaining

\begin{eqnarray}
\label{eq:problemA2bar}
    \bar{P}_2 = \max_{0\leq \beta \leq \beta_{max}}\min_{q_{min}\leq q \leq q_{max}} \min_{\be \in A_{\be}}\frac{\beta q}{2} + \frac{\beta}{2q}\sigma_{\bepsilon}^2 + \frac{\beta}{2qm}\|\tilde{\bR}^{1/2}\be\|_2^2- \frac{\beta}{\sqrt{nm}}\bh^T\tilde{\bR}^{1/2}\be - \frac{\beta^2}{2} + \frac{1}{m}B(\be)
\end{eqnarray}

At this point we will consider only the inner optimization problem over $\be$ and consider $\beta$ and $q$ to be fixed. We shall return to the outer optimization later, and instead only consider

\begin{eqnarray}
\label{eq:D2definition}
D_2 = D_2(\beta,q)= \min_{\be \in A_{\be}}\frac{\beta q}{2} + \frac{\beta}{2q}\sigma_{\bepsilon}^2 + \frac{\beta}{2qm}\|\tilde{\bR}^{1/2}\be\|_2^2- \frac{\beta}{\sqrt{nm}}\bh^T\tilde{\bR}^{1/2}\be - \frac{\beta^2}{2} + \frac{1}{m}B(\be)
\end{eqnarray}

We now make use of the definition of $\tilde{\bR}$. We note specifically that

\begin{eqnarray}
    \tilde{\bR}^{1/2}\bh = \tilde{\bh} \sim \mathcal{N}(0, \tilde{\bR} = \frac{\rho_1^2}{d}\bW\bW^T + \rho_*^2\bI)
\end{eqnarray}
Which by the additivity of Gaussians can be expressed as
\begin{eqnarray}
    \tilde{\bh} = \frac{\rho_1}{\sqrt{d}}\bW\bphi_1 + \rho_*\bphi_2
\end{eqnarray}
In which $\bphi_1\sim\mathcal{N}(0, I_d)$ and $\bphi_2\sim\mathcal{N}(0, I_m)$, we also pull the relevant factor of $\be^T\tilde{\bR}\be$ out of $B(\be)$. We make a new definition $\tilde{B}(\be) = r(\be + \btheta^*) + \tau_2h(\be+\btheta^*)$, we remind that $\tilde{B}(\be)$ is by assumption $\frac{\mu}{2}$ strongly convex. Making the relevant substitutions we obtain

\begin{eqnarray}
    D_2 = \min_{\be\in A_{\be}}\frac{\beta q}{2} + \frac{\beta}{2q}\sigma_{\bepsilon}^2 + \frac{\beta\rho_1^2}{2qmd}||\bW^T\be||_2^2 + \frac{\beta\rho_*^2}{2qm}||\be||_2^2 - \nwl \frac{\beta\rho_1}{\sqrt{nmd}}\be^T\bW\bphi_1 - \frac{\beta\rho_*}{\sqrt{nm}}\bphi_2^T\be - \frac{\beta^2}{2} +\frac{\tau_1\rho_1^2}{md}||\bW^T\be||_2^2 + \frac{\tau_1\rho_*^2}{m}||\be||_2^2 + \frac{1}{m}\tilde{B}(\be)
\end{eqnarray}
We complete the square over the terms that contain $\bW^T\be$, obtaining:

\begin{eqnarray}
    D_2 = \min_{\be \in A_{\be}}\frac{\rho_1^2(\beta + 2q\tau_1)}{2qmd}\left\|\bW^T\be - \frac{\beta\sqrt{md}}{\rho_1(\beta + 2qm\tau_1)\sqrt{n}}\bphi_1 \right\|_2^2 - \frac{\beta^2q}{2n(\beta + 2q\alpha_1)}||\bphi_1||_2^2 \nwl+  \frac{\beta q}{2} + \frac{\beta}{2q}\sigma_{\bepsilon}^2   +  \frac{\rho_*^2(\beta + 2q\tau_1)}{2qm}||\be||_2^2 - \frac{\beta\rho_*}{\sqrt{nm}}\bphi_2^T\be - \frac{\beta^2}{2}  +  \frac{1}{m}\tilde{B}(\be)
\end{eqnarray}

We now introduce another maximization over $\bp$ as the convex conjugate of the $\ell_2^2$ norm. We obtain
\begin{eqnarray}
    D_2 = \min_{\be\in A_{\be}}\max_{\bp}\frac{\rho_1^2(\beta + 2q\tau_1)}{qmd}\bp^T\bW^T\be - \frac{\beta\rho_1}{\sqrt{nmd}}\bp^T\bphi_1 - \frac{\rho_1^2(\beta + 2q\tau_1)}{2qmd}||\bp||_2^2 \nwl- \frac{\beta^2q}{2n(\beta + 2q\tau_1)}||\bphi_1||_2^2 +  \frac{\beta q}{2} + \frac{\beta}{2q}\sigma_{\bepsilon}^2   +  \frac{\rho_*^2(\beta + 2q\tau_1)}{2qm}||\be||_2^2 - \frac{\beta\rho_*}{\sqrt{nm}}\bphi_2^T\be - \frac{\beta^2}{2}  +  \frac{1}{m}\tilde{B}(\be)
\end{eqnarray}

Our goal now is to apply the CGMT again to this problem. We note that the problem in convex in $\be$ and concave in $\bp$. However we need to show that $\be$ and $\bp$ can be bound to compact and convex sets, and that the optimal points of both optimizations fall within these sets. We prove this in the following lemma
\begin{lemma}
\label{lemma:bounded_CGMT2}
        Consider the following two optimization problems that correspond to the first alternative and second alternative optimization by the CGMT
        \begin{eqnarray}
        \label{eq:compactsetproblem3}
         D_2 = \min_{\be\in A_{\be}}\max_{\bp}\frac{\rho_1^2(\beta + 2q\tau_1)}{qmd}\bp^T\bW^T\be - \frac{\beta\rho_1}{\sqrt{nmd}}\bp^T\bphi_1 - \frac{\rho_1^2(\beta + 2q\tau_1)}{2qmd}||\bp||_2^2 \nwl- \frac{\beta^2q}{2n(\beta + 2q\tau_1)}||\bphi_1||_2^2 +  \frac{\beta q}{2} + \frac{\beta}{2q}\sigma_{\bepsilon}^2   +  \frac{\rho_*^2(\beta + 2q\tau_1)}{2qm}||\be||_2^2 - \frac{\beta\rho_*}{\sqrt{nm}}\bphi_2^T\be - \frac{\beta^2}{2}  +  \frac{1}{m}\tilde{B}(\be)\\
         \label{eq:compactsetproblem4}
          D_3 = \min_{\be}\max_{\bp}\frac{\rho_1^2(\beta + 2q\tau_1)}{qmd}||\bp||_2\be^T\bphi_3 + \frac{\rho_1^2(\beta + 2q\tau_1)}{qmd}||\be||_2\bp^T\bphi_4 - \frac{\beta\rho_1}{\sqrt{nmd}}\bp^T\bphi_1 - \frac{\rho_1^2(\beta + 2q\tau_1)}{2qmd}||\bp||_2^2 \nwl- \frac{\beta^2q}{2n(\beta + 2q\tau_1)}||\bphi_1||_2^2 +  \frac{\beta q}{2} + \frac{\beta}{2q}\sigma_{\bepsilon}^2   +  \frac{\rho_*^2(\beta + 2q\tau_1)}{2qm}||\be||_2^2 - \frac{\beta\rho_*}{\sqrt{nm}}\bphi_2^T\be - \frac{\beta^2}{2}  +  \frac{1}{m}\tilde{B}(\be)
        \end{eqnarray}
        where $\bphi_3$ and $\bphi_4$ are standard normals of dimension $m, d$ respectively. Denote $\hat{\be}_2,  \hat{\be}_3$ as optimal points of $D_2$ and $D_3$ respectively and $\hat{\bp}_2(\be), \hat{\bp}_3(\be)$ as their inner optimization solution for a fixed $\be$. Let $\tilde{B}$ be $\frac{\mu}{2}$ strongly convex and $\max\left\{\|\nabla r(\theta^*)\|,\|\nabla h(\theta^*)\|\right\}=O(\sqrt{m})$. Then there exist positive constants $C_{\be},  C_{\bp}$ only depending on $\mu$ such that
        \begin{eqnarray}
        \liml_{m\to\infty}\Pr\left(||\hat{\be_i}||_2 \leq C_{\be}\sqrt{m}\right) =1\quad  i=2,3
        \end{eqnarray}
        and
        \begin{eqnarray}
        \liml_{m\to\infty}\Pr\left(\supl_{\be\mid\|\be\|\leq C_{\be}}||\hat{\bp}_i(\be)||_2 \leq C_{\bp}\sqrt{md}\right) = 1\quad i  =2,3
        \end{eqnarray}
    \end{lemma}
    \begin{proof}
    We know that $C_{\be_2}$ exists from the fact that in $D_2$ $\be$ is already in a bounded set. For both optimizations, we solve the optimization over $\bp$, and write this optimization over $\be$ as
    \begin{eqnarray}
    \min_{\be} F_i(\be)\quad i = 2, 3,
    \end{eqnarray}
    where $F_i(\be)$ is the optimal value over $\bp$. We note that setting $\bp = \bm{0}$ in both optimizations we obtain that
    \begin{eqnarray}
    F(\be) \geq \frac{1}{m}T(\be):=- \frac{\beta^2q}{2n(\beta + 2q\tau_1)}||\bphi_1||_2^2 +  \frac{\beta q}{2} + \frac{\beta}{2q}\sigma_{\bepsilon}^2   +  \frac{\rho_*^2(\beta + 2q\tau_1)}{2qm}||\be||_2^2 - \frac{\beta\rho_*}{\sqrt{nm}}\bphi_2^T\be - \frac{\beta^2}{2}  +  \frac{1}{m}\tilde{B}(\be)
    \end{eqnarray}
    On the other hand,  by taking the second derivative, we observe that $T(\be)$ is  $\nu = \frac{\rho_*\beta}{2q} + \frac\mu2$ strongly convex with respect to $\be$. As such, we find that
    \begin{eqnarray}
    T(\be) \geq T(\bm{0}) + \bd^T\be + \frac{\nu}{2}||\be||_2^2,
    \end{eqnarray}
    where $\bd=\nabla T(\bzero)$. We note that by the assumption, $\bd=\mathcal{O}(\sqrt{m})$. For the optimization $D_3$, we let $\xi = ||\bp||_2$ and solve the optimization over $\bp$ to obtain that
    \begin{eqnarray}
    F_3(\be) = \max_{\xi > 0}\frac{\rho_1^2(\beta + 2q\tau_1)\xi}{qmd}\be^T\bphi_3 + \xi\left\|\frac{\rho_1^2(\beta + 2q\tau_1)}{qmd}||\be||_2\bphi_4 - \frac{\beta\rho_1}{\sqrt{nmd}}\bphi_1\right\|_2 - \frac{\rho_1^2(\beta + 2q\tau_1)\xi^2}{2qmd} 
    + \frac{1}{m}T(\be)
    \end{eqnarray}
    We note that dropping the constraint over $\xi$ will not decrease the optimal value, as such
    \begin{eqnarray}
    F_3(\be) \leq \max_{\xi}\frac{\rho_1^2(\beta + 2q\tau_1)\xi}{qmd}\be^T\bphi_3 + \xi\left\|\frac{\rho_1^2(\beta + 2q\tau_1)}{qmd}||\be||_2\bphi_4 - \frac{\beta\rho_1}{\sqrt{nmd}}\bphi_1\right\|_2 - \frac{\rho_1^2(\beta + 2q\tau_1)\xi^2}{2qmd} 
    + \frac{1}{m}T(\be)
    \end{eqnarray}
    
    From which we see that
    \begin{eqnarray}
    F_3(\bm{0}) \leq \max_{\xi} \frac{\xi\beta\rho_1}{\sqrt{nmd}}||\bphi_1||_2 - \frac{\rho_1^2(\beta + 2q\tau_1)\xi^2}{2qmd} 
    + \frac{1}{m}T(\bm{0})\nwl
     = \frac{\beta\sqrt{md}q ||\bphi_1||_2}{2\sqrt{n}\rho_1(\beta  +2q\tau_1)^2} + \frac{1}{m}T(\bm{0})
    \end{eqnarray}
    
    From this we obtain that
    \begin{eqnarray}
     \frac{\beta\sqrt{md}q ||\bphi_1||_2}{2\sqrt{n}\rho_1(\beta  +2q\tau_1)^2} + \frac{1}{m}T(\bm{0})\geq F(\bm{0}) \geq F_3(\hat{\be}) \geq \frac{1}{m}T(\bm{0})  + \frac{1}{m}\bd^T\be + \frac{\nu}{2m}||\be||_2^2,
    \end{eqnarray}
    and hence
    \begin{eqnarray}
    \frac{\nu}{2m}\left\|\be + \frac{1}{\nu}\bd \right\|_2^2 \leq \frac{1}{\nu m}||\bd||_2^2 +\frac{\beta\sqrt{md}q ||\bphi_1||_2}{2\sqrt{n}\rho_1(\beta  +2q\tau_1)^2}
    \end{eqnarray}
    or
    \begin{eqnarray}
    ||\be||_2 \leq \frac{1}{\nu}||\bd||_2 + \sqrt{\frac{2}{\nu^2}||\bd||_2^2 + \frac{m\beta\sqrt{md}q ||\bphi_1||_2}{\nu\sqrt{n}\rho_1(\beta  +2q\tau_1)^2}}
    \end{eqnarray}
    Noting that with high probability $||\bphi||_2 \leq C\sqrt{d}$ and recalling that $n, m, d$ all grow at constant ratios, we can see that there must exist a constant $C_{\be_3}$ such that
    \begin{eqnarray}
    \Pr(||\hat{\be}_3|| > C_{\be_3}\sqrt{m}) \to 0
    \end{eqnarray}
     We then let $C_{\be} = \max(C_{\be_2}, C_{\be_3})$ and define the set $\tilde{A}_{\be} = \lbrace \be \in \mathbb{R}^m|\ ||\be||_2\leq C_{\be}\sqrt{m}\rbrace$. Then from the optimality condition over $\bp$ for eq \eqref{eq:compactsetproblem3} we know that
     
     \begin{eqnarray}
     \hat\bp_2(\be) = \bW^T\be - \frac{\beta\sqrt{md}}{\rho_1(\beta + 2qm\tau_1)\sqrt{n}}\bphi_1
     \end{eqnarray}
     and as such for all $\be \in A_{\be}$ we must have that
     \begin{eqnarray}
     ||\hat\bp_2(\be)||_2 \leq ||\bW||_2||\be||_2 + \frac{\beta\sqrt{md}}{\rho_1(\beta + 2qm\tau_1)\sqrt{n}}||\bphi_1||_2
     \end{eqnarray}
     We know as a standard result that $||\bW||_2 \leq C\sqrt{d}$ and that $||\bphi_1||_2\leq C\sqrt{d}$ with high probability. As such the constant $C_{\bp_2}$ must exist.
     
     Finally examining the optimality condition over $\xi$ of problem \eqref{eq:compactsetproblem4} we find that for all $\be \in A_{\be}$ we have that
     \begin{eqnarray}
     \hat{\xi} = ||\hat{\bp}_3(\be)||_2 = \be^T\bphi_3 + \left\|||\be||_2\bphi_4 - \frac{q\beta \sqrt{md}}{\rho_1(\beta + 2q\tau_1)\sqrt{n}}||\bphi_1\right\|_2 \nwl
     \leq  ||\be||_2||\bphi_3||_2 + ||\be||_2||\bphi_4||_2 + \frac{q\beta\sqrt{md}}{\rho_1(\beta+2q\tau_1)\sqrt{n}}||\bphi_1||_2
     \end{eqnarray}
     
     We note that with high probability $||\bphi_1||_2 < \sqrt{d}C$, $||\bphi_4||_2 < \sqrt{d}C$ and $||\bphi_3||_2 < \sqrt{m}C$. Recalling that $m, d$ grow at constant ratio we see that the constant $C_{\bp_3}$ exists.
    \end{proof}

    We can therefore define the constants $C_{\be}:= \max(C_{\be_i})$ from $i=1,2,3$ and $C_{\bp} = \max(C_{\bp_2},C_{\bp_3})$, and by doing so define the sets $A_{\be} = \lbrace \be\in\mathbb{R}^m|\ ||\be||_2 < C_{\be}\sqrt{m}\rbrace$ and $A_{\bp} = \lbrace \bp\in\mathbb{R}^d|\ ||\bp||_2 < C_{\bp}\sqrt{md}\rbrace$.
     From this we can see that with high probability the optimal value of the optimization $\bar{P}_2$ will be equal to that of

\begin{eqnarray}
    D_2 = \min_{\be\in A_{\be}}\max_{\bp\in A_{\bp}}\frac{\rho_1^2(\beta + 2q\tau_1)}{qmd}\bp^T\bW^T\be - \frac{\beta\rho_1}{\sqrt{nmd}}\bp^T\bphi_1 - \frac{\rho_1^2(\beta + 2q\tau_1)}{2qmd}||\bp||_2^2 \nwl- \frac{\beta^2q}{2n(\beta + 2q\tau_1)}||\bphi_1||_2^2 +  \frac{\beta q}{2} + \frac{\beta}{2q}\sigma_{\bepsilon}^2   +  \frac{\rho_*^2(\beta + 2q\tau_1)}{2qm}||\be||_2^2 - \frac{\beta\rho_*}{\sqrt{nm}}\bphi_2^T\be - \frac{\beta^2}{2}  +  \frac{1}{m}\tilde{B}(\be)
\end{eqnarray}
We now apply the CGMT to the problem $D_2$ for fixed values of $\beta, q$, we obtain the following problem 

\begin{eqnarray}
D_3 = \min_{\be\in A_{\be}}\max_{\bp\in A_{\bp}}\frac{\rho_1^2(\beta + 2q\tau_1)}{qmd}||\bp||_2\be^T\bphi_3 + \frac{\rho_1^2(\beta + 2q\tau_1)}{qmd}||\be||_2\bp^T\bphi_4 - \frac{\beta\rho_1}{\sqrt{nmd}}\bp^T\bphi_1 - \frac{\rho_1^2(\beta + 2q\tau_1)}{2qmd}||\bp||_2^2 \nwl- \frac{\beta^2q}{2n(\beta + 2q\tau_1)}||\bphi_1||_2^2 +  \frac{\beta q}{2} + \frac{\beta}{2q}\sigma_{\bepsilon}^2   +  \frac{\rho_*^2(\beta + 2q\tau_1)}{2qm}||\be||_2^2 - \frac{\beta\rho_*}{\sqrt{nm}}\bphi_2^T\be - \frac{\beta^2}{2}  +  \frac{1}{m}\tilde{B}(\be)
\end{eqnarray}
Let $\xi = \frac{\rho_1}{\sqrt{md}}||\bp||_2$ and solve the optimization over $\bp$. We note that $\xi\geq0$ and that $\xi\leq \xi_{max} = \frac{\rho_1}{\sqrt{dm}}\sup_{\bp\in A_{\bp}} ||\bp||_2$. From this we obtain the problem,
\begin{eqnarray}
D_3 = \min_{\be\in A_{\be}}\max_{0 \leq \xi \leq \xi_{max}}\frac{\rho_1(\beta + 2q\tau_1)}{q\sqrt{md}}\be^T\bphi_3 + \xi\left\|\frac{\rho_1(\beta + 2q\tau_1)}{q\sqrt{md}}||\be||_2\bphi_4 - \frac{\beta}{\sqrt{n}}\bphi_1\right\|_2 - \frac{(\beta + 2q\tau_1)\xi^2}{2q} \nwl- \frac{\beta^2q}{2n(\beta + 2q\tau_1)}||\bphi_1||_2^2 +  \frac{\beta q}{2} + \frac{\beta}{2q}\sigma_{\bepsilon}^2   +  \frac{\rho_*^2(\beta + 2q\tau_1)}{2qm}||\be||_2^2 - \frac{\beta\rho_*}{\sqrt{nm}}\bphi_2^T\be - \frac{\beta^2}{2}  +  \frac{1}{m}\tilde{B}(\be)
\end{eqnarray}

We now show that this term concentrates in the following lemma

\begin{lemma}
\label{lem:secondConcentrationResult}
Let $F(\be, \xi)$ be given by
\begin{eqnarray}
F(\be, \xi) = \frac{\rho_1(\beta + 2q\tau_1)}{q\sqrt{md}}\be^T\bphi_3 + \xi\left\|\frac{\rho_1(\beta + 2q\tau_1)}{q\sqrt{md}}||\be||_2\bphi_4 - \frac{\beta}{\sqrt{n}}\bphi_1\right\|_2 - \frac{(\beta + 2q\tau_1)\xi^2}{2q} \nwl- \frac{\beta^2q}{2n(\beta + 2q\tau_1)}||\bphi_1||_2^2 +  \frac{\beta q}{2} + \frac{\beta}{2q}\sigma_{\bepsilon}^2   +  \frac{\rho_*^2(\beta + 2q\tau_1)}{2qm}||\be||_2^2 - \frac{\beta\rho_*}{\sqrt{nm}}\bphi_2^T\be - \frac{\beta^2}{2}  +  \frac{1}{m}\tilde{B}(\be)
\end{eqnarray}
and let $\bar{F}(\be, \xi)$
\begin{eqnarray}
\bar{F}(\be, \xi) = \frac{\rho_1(\beta + 2q\tau_1)}{q\sqrt{md}}\be^T\bphi_3 + \xi\sqrt{\frac{\rho_1^2(\beta + 2q\tau_1)^2}{q^2m}||\be||_2^2 + \frac{\beta^2d}{n}} - \frac{(\beta + 2q\tau_1)\xi^2}{2q} \nwl- \frac{\beta^2qd}{2n(\beta + 2q\tau_1)} +  \frac{\beta q}{2} + \frac{\beta}{2q}\sigma_{\bepsilon}^2   +  \frac{\rho_*^2(\beta + 2q\tau_1)}{2qm}||\be||_2^2 - \frac{\beta\rho_*}{\sqrt{nm}}\bphi_2^T\be - \frac{\beta^2}{2}  +  \frac{1}{m}\tilde{B}(\be)
\end{eqnarray}
Then
\begin{eqnarray}
\Pr\left(\sup_{\be \in A_{\be}, 0 \leq \xi\leq \xi_{max}}|F(\be, \xi) - \bar{F}(\be, \xi)| > \epsilon\right) \xrightarrow[m, d\rightarrow\infty]{P} 0
\end{eqnarray}
\end{lemma}
\begin{proof}
The lemma is proven in the same manner as lemma \ref{lem:A2concentrationLemma}.
\end{proof}

By this lemma we can with high probability consider the following problem instead:
\begin{eqnarray}
\label{eq:D3def1}
\bar{D}_3 = \min_{\be\in A_{\be}}\max_{0 \leq \xi \leq \xi_{max}}\frac{\rho_1(\beta + 2q\tau_1)}{q\sqrt{md}}\be^T\bphi_3 + \xi\sqrt{\frac{\rho_1^2(\beta + 2q\tau_1)^2}{q^2m}||\be||_2^2 + \frac{\beta^2d}{n}} - \frac{(\beta + 2q\tau_1)\xi^2}{2q} \nwl- \frac{\beta^2qd}{2n(\beta + 2q\tau_1)} +  \frac{\beta q}{2} + \frac{\beta}{2q}\sigma_{\bepsilon}^2   +  \frac{\rho_*^2(\beta + 2q\tau_1)}{2qm}||\be||_2^2 - \frac{\beta\rho_*}{\sqrt{nm}}\bphi_2^T\be - \frac{\beta^2}{2}  +  \frac{1}{m}\tilde{B}(\be)
\end{eqnarray}
We now interchange the order of the min and max. As the problem is clearly convex in $\be$ and concave in $\xi$ and the problem is over convex sets this interchange is admissible.

\begin{eqnarray}
\bar{D}_3 =\max_{0 \leq \xi \leq \xi_{max}} \min_{\be\in A_{\be}}\frac{\rho_1(\beta + 2q\tau_1)}{q\sqrt{md}}\be^T\bphi_3 + \xi\sqrt{\frac{\rho_1^2(\beta + 2q\tau_1)^2}{q^2m}||\be||_2^2 + \frac{\beta^2d}{n}} - \frac{(\beta + 2q\tau_1)\xi^2}{2q} \nwl- \frac{\beta^2qd}{2n(\beta + 2q\tau_1)} +  \frac{\beta q}{2} + \frac{\beta}{2q}\sigma_{\bepsilon}^2   +  \frac{\rho_*^2(\beta + 2q\tau_1)}{2qm}||\be||_2^2 - \frac{\beta\rho_*}{\sqrt{nm}}\bphi_2^T\be - \frac{\beta^2}{2}  +  \frac{1}{m}\tilde{B}(\be)
\end{eqnarray}

We now make use of the square root trick one more time, introducing  new parameter $t$, we note that $t$ can be bounded by $t_{min} = \frac{\beta \sqrt{d}}{\sqrt{n}}$ and $t_{max} = \sqrt{\frac{\beta^2 d}{n} + \frac{\rho_1^2(\beta +2\tau_1q)^2}{q^2}C_{\be}^2}$.
\begin{eqnarray}
\label{eq:D3def2}
\bar{D}_3 =\max_{0 \leq \xi \leq \xi_{max}}\min_{t_{min}\leq t\leq t_{max}} \min_{\be\in A_{\be}}\frac{\rho_1(\beta + 2q\tau_1)}{q\sqrt{md}}\be^T\bphi_3 + \frac{\xi\rho_1^2(\beta + 2q\tau_1)^2}{2tq^2m}||\be||_2^2 + \frac{\beta^2\xi d}{2tn} - \frac{(\beta + 2q\tau_1)\xi^2}{2q} \nwl- \frac{\beta^2qd}{2n(\beta + 2q\tau_1)} + \frac{\xi t}{2} +  \frac{\beta q}{2} + \frac{\beta}{2q}\sigma_{\bepsilon}^2   +  \frac{\rho_*^2(\beta + 2q\tau_1)}{2qm}||\be||_2^2 - \frac{\beta\rho_*}{\sqrt{nm}}\bphi_2^T\be - \frac{\beta^2}{2}  +  \frac{1}{m}\tilde{B}(\be)
\end{eqnarray}
Where we have changed the order of the two min operations. We can now define the constants,

\begin{eqnarray}
c_1 = \frac{\xi\rho_1^2(\beta  +2q\tau_1)^2}{2tq^2} + \frac{\rho_*^2(\beta + 2q\tau_1)}{2q} \qquad c_2 = \sqrt{\frac{\rho_1^2(\beta + 2q\tau_1)^2\eta}{q^2} + \rho_*^2\beta^2}
\end{eqnarray}

and we note that by the additivity of Gaussians we have that

\begin{eqnarray}
\frac{c_2}{\sqrt{nm}}\bphi = \frac{\rho_1(\beta + 2q\tau_1)}{q\sqrt{md}}\bphi_3- \frac{\beta\rho_*}{\sqrt{nm}}\bphi_2
\end{eqnarray}
We obtain 
\begin{eqnarray}
\bar{D}_3 =\max_{0 \leq \xi \leq \xi_{max}}\min_{t_{min}\leq t\leq t_{max}} \min_{\be\in A_{\be}} \frac{c_1}{m}||\be||_2^2 + \frac{c_2}{\sqrt{nm}}\bphi^T\be + \frac{\beta^2\xi d}{2tn} - \frac{(\beta + 2q\tau_1)\xi^2}{2q} \nwl- \frac{\beta^2qd}{2n(\beta + 2q\tau_1)} + \frac{\xi t}{2} +  \frac{\beta q}{2} + \frac{\beta}{2q}\sigma_{\bepsilon}^2    - \frac{\beta^2}{2}  +  \frac{1}{m}\tilde{B}(\be)
\end{eqnarray}
Completing the square over $\be$ we find
\begin{eqnarray}
\bar{D}_3 =\max_{0 \leq \xi \leq \xi_{max}}\min_{t_{min}\leq t\leq t_{max}} \min_{\be\in A_{\be}} \frac{c_1}{m}\left\|\be  + \frac{c_2\sqrt{m}}{2c_1\sqrt{n}}\bphi\right\|_2^2 - \frac{c_2^2}{4c_1n}||\bphi||_2^2 + \frac{\beta^2\xi d}{2tn} - \frac{(\beta + 2q\tau_1)\xi^2}{2q} \nwl- \frac{\beta^2qd}{2n(\beta + 2q\tau_1)} + \frac{\xi t}{2} +  \frac{\beta q}{2} + \frac{\beta}{2q}\sigma_{\bepsilon}^2    - \frac{\beta^2}{2}  +  \frac{1}{m}\tilde{B}(\be)
\end{eqnarray}
Finally noting that in the aysmptotic limit $||\bphi||_2^2$ concentrates to $m$ with high probability, and then recognizing the Moreau envelope over $\be$ (see definition \ref{def:MoreauProx} below) we obtain the problem
\begin{eqnarray}
\label{eq:D3def3}
\bar{D}_3 =\max_{0 \leq \xi \leq \xi_{max}}\min_{t_{min}\leq t\leq t_{max}} \frac{1}{m}\mathcal{M}_{\frac{1}{2c_1}\tilde{B}}\left(-\frac{c_2^2\sqrt{m}}{2c_1\sqrt{n}}\bphi\right)  - \frac{c_2 m}{4c_1n}+ \frac{\beta^2\xi d}{2tn} - \frac{(\beta + 2q\tau_1)\xi^2}{2q} \nwl- \frac{\beta^2qd}{2n(\beta + 2q\tau_1)} + \frac{\xi t}{2} +  \frac{\beta q}{2} + \frac{\beta}{2q}\sigma_{\bepsilon}^2    - \frac{\beta^2}{2}  
\end{eqnarray}
We can recall that $\tilde{B}(\be) = r(\be + \btheta^*) + \tau_2h(\be + \btheta^*)$, and letting $\btheta = \be + \btheta^*$, we obtain
\begin{eqnarray}
\bar{D}_3 =\max_{0 \leq \xi \leq \xi_{max}}\min_{t_{min}\leq t\leq t_{max}} \frac{1}{m}\mathcal{M}_{\frac{1}{2c_1}(r + \tau_2h)}\left(\btheta^* -\frac{c_2^2\sqrt{m}}{2c_1\sqrt{n}}\bphi\right)  - \frac{c_2 m}{4c_1n}+ \frac{\beta^2\xi d}{2tn} - \frac{(\beta + 2q\tau_1)\xi^2}{2q} \nwl- \frac{\beta^2qd}{2n(\beta + 2q\tau_1)} + \frac{\xi t}{2} +  \frac{\beta q}{2} + \frac{\beta}{2q}\sigma_{\bepsilon}^2    - \frac{\beta^2}{2}  
\end{eqnarray}
Finally we show in Lemma \ref{lem:MenvelopeGausBound} that the Moreau envelope will concentrate in the asymptotic limit on its expected value. As such we finally obtain:
\begin{eqnarray}
\bar{D}_3 =\max_{0 \leq \xi \leq \xi_{max}}\min_{t_{min}\leq t\leq t_{max}} \frac{1}{m}\mathbb{E}\mathcal{M}_{\frac{1}{2c_1}(r + \tau_2h)}\left(\btheta^* -\frac{c_2^2\sqrt{m}}{2c_1\sqrt{n}}\bphi\right)  - \frac{c_2 m}{4c_1n}+ \frac{\beta^2\xi d}{2tn} - \frac{(\beta + 2q\tau_1)\xi^2}{2q} \nwl- \frac{\beta^2qd}{2n(\beta + 2q\tau_1)} + \frac{\xi t}{2} +  \frac{\beta q}{2} + \frac{\beta}{2q}\sigma_{\bepsilon}^2    - \frac{\beta^2}{2}  
\end{eqnarray}

We know by the properties of the CGMT that for any fixed choice of $\beta, q$ that $D_3(\beta,q)$ converges pointwise to $D_2(\beta,q)$. However to determine the properties that we are interested in we require uniform convergence. For this, we simply show that $D_2(\beta,q), \bar D_3(\beta,q)$ are Lipschitz continuous for $\beta\in [0,\ \beta_{\max}]$ and  $q\in [q_{\min},\ q_{\max}]$.

\begin{lemma}
    The problem $D_2$ as given in \eqref{eq:D2definition} and problem $\bar{D_3}$ as given in equations \ref{eq:D3def1}, \eqref{eq:D3def2} and \eqref{eq:D3def3} are $C$-Lipschitz on the compact set $K= [0, \beta_{max}]\times[q_{min}, q_{max}]$ for some constant $C<\infty$, with high probability.
 \end{lemma}
\begin{proof}
We first consider problem $D_2$ given in equation \eqref{eq:D2definition}. 
\begin{eqnarray}
D_2 = \min_{A_\be} \frac{\beta q}{2} + \frac{\beta\sigma_{\bepsilon}^2}{2q} + \frac{\beta}{2qm}||\tlbR^{1/2}\be||_2^2 - \frac{\beta}{\sqrt{nm}}\bh^T\tlbR\be - \frac{\beta^2}{2} + \frac{\tau_1}{m}\be\tlbR\be + \frac{1}{m}\tilde{B}(\be)
\end{eqnarray}

We note that the objective $D(\beta,q,\be)$ is strongly convex, the solution is hence unique, and $D_2$ is continuously differentiable on the compact set $K$. We simply bound its gradient, which is given by
\begin{eqnarray}
\frac{\partial D_2}{\partial\beta}=\frac{\partial D}{\partial\beta}\mid_{\be=\hat\be}=\ \frac{q}{2}  +\frac{\sigma_{\bepsilon}^2}{2q} + \frac{1}{2qm}||\tlbR^{1/2}\hat\be||_2^2  - \frac{1}{\sqrt{nm}}\bh^T\tlbR\hat\be - \beta\\
\frac{\partial D_2}{\partial q}=\frac{\partial D}{\partial q}\mid_{\be=\hat\be}=\ \frac{\beta}{2} - \frac{\beta\sigma_{\bepsilon}^2}{2q^2} - \frac{\beta}{2q^2m}||\tlbR^{1/2}\hat\be||_2^2 
\end{eqnarray}

where $\hat\be$ is the optimal solution. Noting that $\hat\be\in A_{\be}$ and $\beta,q$ are bounded, we obtain the result for $D_2$. 

For problem $\bar{D}_3$ we make use of the same strategy by calculating the gradient. Defining $\hat\be,\hat\xi$ as the optimal solution of \eqref{eq:D3def1} , we observe that
\begin{eqnarray}
\hat{\xi} = \sqrt{\frac{\rho_1^2}{4m}||\hat{\be}||_2^2 + \frac{\beta^2 q^2d}{4(\beta + 2q\tau_1)^2n}}
\end{eqnarray}
Further, we define
\begin{eqnarray}
\hat{t} = \sqrt{\frac{\rho_1^2(\beta + 2q\tau_1)^2}{q^2m}||\hat{\be}||_2^2 + \frac{\beta^2d}{n}}
\end{eqnarray}
Finally we examine the partial derivatives of problem $D_3$ with respect to $\beta$ and $q$,
\begin{eqnarray}
\frac{\partial \bar D_3}{\partial \beta}=\ \frac{\rho_1}{q\sqrt{md}}\be^T\bphi_3 + \frac{\xi\rho_1^2(\beta + 2q\tau_1)}{tq^2m}||\be||_2^2 + \frac{\beta\xi d}{tn} - \frac{\xi^2}{2q} - \frac{\beta qd}{n(\beta +2q\tau_1)} + \frac{\beta^2qd}{2n(\beta +2q\tau_1)^2} \nwl+ \frac{q}{2} + \frac{\sigma_{\bepsilon}^2}{2q} + \frac{\rho_*^2}{2qm}||\be||_2 - \frac{\beta\rho_*}{\sqrt{nm}}\bphi_2^T\be - \beta\\
\frac{\partial \bar D_3}{\partial q}=\ -\frac{\rho_1(\beta + 2q\tau_1)}{q^2\sqrt{md}}\be^T\bphi_3 - \frac{\xi\rho_1^2(\beta + 2q\tau_1)^2}{tq^3m}||\be||_2^2 + \frac{\xi\rho_1^2\tau_1(\beta + 2q\tau_1)}{2tq^2m}||\be||_2^2  + \frac{(\beta + 2q\tau_1)\xi^2}{2q^2} - \frac{2\tau_1\xi^2}{2q} \nwl- \frac{\beta^2d}{2n(\beta + 2q\tau_1)} + \frac{\beta^2qd\tau_1}{n(\beta + 2q\tau_1)^2} +  \frac{\beta }{2} - \frac{\beta}{2q^2}\sigma_{\bepsilon}^2   -  \frac{\rho_*^2(\beta + 2q\tau_1)}{2q^2m}||\be||_2^2 + \frac{\rho_*^2\tau_1}{qm}||\be||_2^2
\end{eqnarray}
Noting the boundedness of the involved terms, we conclude the result.
\end{proof}

We have established that both $D_2$ and $\bar{D}_3$ are Lipschitz, we now create a rectangular $\epsilon$ net $\calN$ on the set $[0, \beta_{max}]\times[q_{min}, q_{max}]$ consisting of  $k=\frac{\beta_{max}(q_{max}-q_{min})}{\epsilon^2}$ points. We can then see that

\begin{eqnarray}
|D_2(\beta, q) - D_3(\beta, q)| \leq |D_2(\beta, q) - D_2(\beta_k, q_k)| + |D_2(\beta_k, q_k) - D_3(\beta_k, q_k)| + |D_{3}(\beta_k, q_k) - D_{3}(\beta, q)|\nwl 
\leq C\epsilon\sqrt{2} + |D_2(\beta_k, q_k) - D_3(\beta_k, q_k)| + C\epsilon\sqrt{2},
\end{eqnarray}
$\beta_k, q_k$ is the closes element of the $\epsilon$-net to $\beta, q$. The second inequality is due to the fact that both $D_2$ and $D_3$ are $C$-Lipschitz with respect to both $\beta$ and $q$ and the distance of between $\beta, q$ and $\beta_k, q_k$ cannot be more than $\epsilon\sqrt{2}$. From this we can see that
\begin{equation}
\sup_{0\leq\beta\leq\beta_{max}, q_{min}\leq q\leq q_{max}}|D_2(\beta, q) - D_3(\beta,q)| \leq 2C\epsilon\sqrt{2} + \sup_{\beta, q \in \mathcal{N}}|D_2(\beta,q) - D_3(\beta, q)|
\end{equation}
As a result, 
\begin{eqnarray}
\Pr\left(\sup_{0\leq\beta\leq\beta_{max}, q_{min}\leq q\leq q_{max}}|D_2(\beta, q) - D_3(\beta,q)|\geq 4C\epsilon\sqrt{2} \right)\leq \Pr\left(\sup_{\beta, q \in \mathcal{N}}|D_2(\beta,q) - D_3(\beta, q)|\geq 2C\epsilon\sqrt{2} \right)
\end{eqnarray}
For a fixed and $k$, the right hand side goes to zero by the union bound and the second CGMT. Therefore the convergence is uniform in the sense that
\begin{equation}
    \Pr\left(\sup_{0\leq\beta\leq\beta_{max}, q_{min}\leq q\leq q_{max}}|D_2(\beta, q) - D_3(\beta,q)|\geq \delta \right)\to 0
\end{equation}
for any $\delta>0$.
Finally we can obtain the following optimization problem:
\begin{eqnarray}
\label{eqn:A3problem}
\tlP_3 = \max_{0\leq \beta \leq \beta_{max}}\min_{q_{min}\leq q \leq q_{max}}\max_{0\leq\xi\leq \xi_{max}}\min_{t_{min}\leq t\leq t_{max}}\mathbb{E} \frac{1}{m}\mathcal{M}_{\frac{1}{2c_1}(r + \tau_2h)}\left(\btheta^*  -\frac{c_2^2\sqrt{m}}{2c_1\sqrt{n}}\bphi\right)\nwl  - \frac{c_2 m}{4c_1n}+ \frac{\beta^2\xi d}{2tn} - \frac{(\beta + 2q\tau_1)\xi^2}{2q}- \frac{\beta^2qd}{2n(\beta + 2q\tau_1)} + \frac{\xi t}{2} +  \frac{\beta q}{2} + \frac{\beta}{2q}\sigma_{\bepsilon}^2    - \frac{\beta^2}{2} 
\end{eqnarray}

We have now demonstrated that $\tlP_3$ converges in probability to $\tlP_2$, which subsequently converges to $\tlP_1$. This establishes the first part of Theorem 2, about the optimal values.
We show the asymptotic equivalence of the generalization error and test functions by  following lemma
\begin{lemma}
\label{lem:gaussianGenErrorConvergence}
Let $\hat{\btheta}_2(\tau_1, \tau_2)$ be the solution of $P_2$ \eqref{eq:GaussianKeyOptimization} and let $\hat{\btheta}_3(\tau_1, \tau_2)$ be the solution of $\tilde{P}_3$ as given in \eqref{eqn:A3problem}, then
\begin{eqnarray}
\mathcal{E}_{gen}(\hat{\btheta}_2(0, 0)) \xrightarrow[n\rightarrow\infty]{P} \tilde{\mathcal{E}}_{gen}\\
\frac{1}{m}h(\hat{\btheta}_2(0, 0)) \xrightarrow[n\rightarrow\infty]{P}\frac{1}{m}h(\hat{\btheta}_3(0, 0))
\end{eqnarray}
\end{lemma}
\begin{proof}
    We note that for any optimization $P(\tau)=\minl_{\be} F(\be)+\tau G(\be)$ with optimal solution $\be_\tau$ it holds that
    \begin{equation}
        P(\tau)\leq F(\be_0)+\tau G(\be_0)
    \end{equation}
    Applying this observation to our problem with $\tau_1=\tau$ and $\tau_2=0$, we obtain
    \begin{eqnarray}
    P_2(\tau,0)\leq P_2(0,0)+\tau \frac{(\hat{\btheta} - \btheta^*)^T\tilde{\bR}(\hat{\btheta} - \btheta^*)}{m} 
    \end{eqnarray}
    From which we obtain that
    \begin{eqnarray}
    \frac{P_2(\tau, 0) - P_2(0, 0)}{\tau} \leq  \frac{(\hat{\btheta} - \btheta^*)^T\tilde{\bR}(\hat{\btheta} - \btheta^*)}{m} \quad \tau>0\nwl
    \frac{(\hat{\btheta} - \btheta^*)^T\tilde{\bR}(\hat{\btheta} - \btheta^*)}{m}
    \leq \frac{P_2(0, 0) - P_2(\tau, 0)}{\tau}\quad \tau<0
    \end{eqnarray}
   Take an arbitrary $\delta>0$. For sufficiently small values of $\tau$ and from the convergence of the optimal value we have that
    \begin{eqnarray}
    \Pr\left(\frac{(\hat{\btheta} - \btheta^*)^T\tilde{\bR}(\hat{\btheta} - \btheta^*)}{m} < \frac{\tlP_3(\tau, 0) - \tlP_3(0, 0)}{\tau} + \frac{\delta}{2} \right) \rightarrow 0,\quad \tau>0\\
     \Pr\left(\frac{(\hat{\btheta} - \btheta^*)^T\tilde{\bR}(\hat{\btheta} - \btheta^*)}{m} > \frac{\tlP_3(0, 0) - \tlP_3(\tau, 0)}{\tau} - \frac{\delta}{2} \right) \quad \tau<0\rightarrow 0
    \end{eqnarray}
    Where this relationship follows form the fact that $\tlP_2(\tau_1,\tau_2)$ converges to $\tlP_3(\tau_1, \tau_2)$ for all $\tau_1\in [-\tau_1^*, \tau_1^*]$ and $\tau_2 \in [-\tau_2^*, \tau_2^*]$. We also know that for sufficiently small values of $|\tau|$ we have that
    \begin{eqnarray}
    \left|\frac{\tlP_3(\tau, 0)- \tlP_3(0, 0) }{\tau} - \left.\frac{\partial \tlP_3(\tau_1,0)}{\partial \tau_1}\right|_{\tau_1 = 0}\right| \leq \frac{\delta}{2}
    \end{eqnarray}
    The uniqueness of the solutions $\hat{t}, \hat{\xi}, \hat{q}, \hat{\beta}$ guarantees that the derivatives exist. We then obtain that
    \begin{eqnarray}
    \Pr\left(\left|\frac{(\hat{\btheta} - \btheta^*)^T\tilde{\bR}(\hat{\btheta} - \btheta^*)}{m} - \left.\frac{\tlP_3(\tau_1, 0)}{\partial\tau_1}\right|_{\tau_1 = 0} \right|  > \delta \right) \rightarrow 0
    \end{eqnarray}
    from which we finally obtain that
    \begin{eqnarray}
    \frac{(\hat{\btheta} - \btheta^*)^T\tilde{\bR}(\hat{\btheta} - \btheta^*)}{m} \xrightarrow[n\rightarrow\infty]{P} \left.\frac{\tlP_3(\tau_1, 0)}{\partial \tau_1}\right|_{\tau_1 = 0}
    \end{eqnarray}
    This provides the first result, but 
    we can also compute that 
    \begin{eqnarray}
    \left.\frac{\tlP_3(\tau_1, 0)}{\partial \tau_1}\right|_{\tau_1 = 0} = \frac{1}{m}\mathbb{E}\left[\left.\left\|\btheta^* -\frac{c_2^2\sqrt{m}}{2c_1\sqrt{n}}\bphi -\mathrm{prox}_{\frac{1}{2c_1}}\left(\btheta^* -\frac{c_2^2\sqrt{m}}{2c_1\sqrt{n}}\bphi \right)\right\|_2^2\frac{\partial c_1}{\partial \tau_1} \right.\right.\nwl \left. \left. 
    + \left(\btheta^* -\frac{c_2^2\sqrt{m}}{2c_1\sqrt{n}}\bphi -\mathrm{prox}_{\frac{1}{2c_1}}\left(\btheta^* -\frac{c_2^2\sqrt{m}}{2c_1\sqrt{n}}\bphi \right)\right)^T\left(\frac{c_2^2\sqrt{m}}{c_1\sqrt{n}}\frac{\partial c_1}{\partial\tau_1} - \frac{c_2\sqrt{m}}{\sqrt{n}}\frac{\partial c_2}{\partial\tau_1} \right)\bphi\right|_{\tau_1 = 0} \right] \nwl - \hat{\xi}^2 - \frac{\hat{q}^2d}{n}
    \end{eqnarray}
    
    where $c_1$ and $c_2$ are evaluated at $\hat{\beta}, \hat{q}, \hat{\xi}, \hat{t}$ and $\tau_1 =0, \tau_2 = 0$. In this computation we have made use of the following rules for the derivatives of Moreau envelopes 
    \begin{eqnarray}
    \nabla_{\bx} \mathcal{M}_{\tau f}(\bx) = \frac{1}{\tau}(\bx - \mathrm{prox}_{\tau f}(\bx))\\
    \frac{\partial}{\partial \tau}\mathcal{M}_{\tau f}(\bx) = - \frac{1}{2\tau^2}\left\|\bx - \mathrm{prox}_{\tau f}(\bx) \right\|_2^2
    \end{eqnarray}
    
    Using the same symmetric logic for the case of $\tau_2$ we find that
    
    \begin{eqnarray}
    \frac{h(\hat{\btheta}(0,0)_2)}{m} \xrightarrow[n\rightarrow\infty]{P} \left.\frac{\partial \tlP_3(0,\tau_2)}{\partial\tau_2}\right|_{\tau_2 =0}
    \end{eqnarray}
    where we find that
    \begin{eqnarray}
\left.\frac{\partial \tlP_3(0,\tau_2)}{\partial\tau_2}\right|_{\tau_2 =0} = h(\hat{\btheta}_3(\hat{\beta}, \hat{q}, \hat{\xi},\hat{t}))
    \end{eqnarray}
    From this we see that
    \begin{eqnarray}
    \mathbb{E}\frac{1}{m}h(\hat{\btheta}_2(0, 0)) \xrightarrow[n\rightarrow\infty]{P}\mathbb{E}\frac{1}{m}h(\hat{\btheta}_3(0, 0))
    \end{eqnarray}
    Finally to demonstrate the generalization error we note that
    
    \begin{eqnarray}
    \mathcal{E}_{gen}(\hat{\theta}_2) = \mathbb{E}\left(y_{new} - 
    \frac{1}{\sqrt{m}}\tilde{\bvarphi}(\bz_{new})^T\hat{\btheta}_2 \right)^2 = \mathbb{E}\left(\epsilon_{new} -\frac{1}{\sqrt{m}}\tilde{\bvarphi}(\bz_{new})^T(\hat{\btheta}_2 - \btheta^*) \right)^2 
    \end{eqnarray}
    in which we have made use of the definition of $y_{new} = \frac{1}{\sqrt{m}}\tilde{\bvarphi}(\bz_{new})\btheta^* + \epsilon_{new}$. We recall that $\mathbb{E}[\tilde{\bvarphi}(\bz_{new})\tilde{\bvarphi}(\bz_{new})] = \tilde{\bR}$. As such we obtain that
    \begin{eqnarray}
    \mathcal{E}_{gen}(\hat{\btheta}_2) = \sigma_{\bepsilon}^2 + \frac{(\hat{\btheta} - \btheta^*)^T\tilde{\bR}(\hat{\btheta} - \btheta^*)}{m} 
    \end{eqnarray}
    By the calculation above we see that
    \begin{eqnarray}
    \mathcal{E}_{gen}(\hat{\theta}_2) \rightarrow \sigma_{\bepsilon}^2 + \left.\frac{\partial \tlP_3(\tau_1, 0)}{\partial\tau_1}\right|_{\tau_1 = 0} = \tilde{\mathcal{E}}_{gen}
    \end{eqnarray}
    
\end{proof}

\subsection{Non Deterministic True Vector}
In the previous analysis we have assumed that the true vector $\btheta^*$ has been deterministic. In the case of $\btheta^*$ being random, we can freeze its value by conditioning on $\btheta^*$. The proof holds for a random $\btheta^*$ with high probability, according to the assumptions. This shows that the results hold for a suitable random $\btheta^*$. 

\subsection{Moreau Envelopes}
We remind the reader of the definition of the Moreau Envelope and the proximal operator.
\begin{definition}\label{def:MoreauProx}
Let $f:\calX \rightarrow (-\infty, \infty]$ be a proper, lower semi-continuous function on a Hilbert space $\calX$. Then the Moreau envelope with step size $\tau$ of the function is given by
\begin{eqnarray}
    \calM_{\tau f}\left(\by\right) = \min_{\bx \in \calX} f(\bx) + \frac{1}{2\tau}\left\|\bx - \by \right\|
\end{eqnarray}
The proximal operator of the function $f$ with step size $\tau$ is given by
\begin{eqnarray}
    \mathrm{prox}_{\tau f}(\by) = \arg\min_{\bx \in \calX} f(\bx) + \frac{1}{2\tau}\left\|\bx - \by \right\|
\end{eqnarray}
\end{definition}

Here we give a lemma concerning the concentration of Moreau envelopes.
\begin{lemma}[Gaussian Concentration of Moreau Envelopes, extension of (\cite{Loureiro2021Learning}, lemma 5)]
\label{lem:MenvelopeGausBound}
Consider a proper convex function $f:\mathbb{R}^n\rightarrow\mathbb{R}$. Furthermore, let $\bg\in\mathbb{R}^n$ be a standard Gaussian random vector and $\ba\in\mathbb{R}^n$ a constant vector with finitely bounded norm. Then for any parameter $\tau>0$ and for any $\epsilon > 0$, there exists a constant $c$ such that
\begin{eqnarray}
    \mathbb{P}\left(\left|\frac{1}{n}\mathcal{M}_{\tau f}(\ba + \bg) - \mathbb{E}\left[\frac{1}{n}\mathcal{M}_{\tau f}(\ba + \bg) \right] \right| \geq \epsilon\right) \leq \frac{c}{n\tau^2\epsilon^2}
\end{eqnarray}
\end{lemma}

The original lemma as given by \citep{Loureiro2021Learning} does not have the constant vector $\ba$ and instead only considers a Moreau envelope over a Gaussian. We give a proof here for this case but note that the original proof may be applied by instead considering the shifted function $\bar{f}(\cdot) = f(\cdot - \ba)$. We give the proof here for completeness.

\begin{proof}
    First, we show that the Moreau envelope of a convex proper function $f$ is integrable with respect to the Gaussian measure. By making use of the convexity of the optimization problem that defines the Moreau envelope, and because $f$ is proper, there exists a $\bz_0\in\mathbb{R}^n$ and finite constant $\kappa$ such that
    \begin{eqnarray}
        \frac{1}{n}\mathcal{M}_{\tau f}(\bg + \ba) \leq \frac{1}{n}f(\bz_0) + \frac{1}{2n\tau}\left\|\bz_0 - \bg -\ba \right\|^2\nwl
        \leq \kappa + \frac{1}{2n\tau}\left\|\bz_0 - \bg -\ba \right\|^2
    \end{eqnarray}
    The second line is integrable with respect to a Gaussian measure. By means of the Gaussian Poincare inequality (see for example, \citep{boucheron2013concentration}).
    \begin{eqnarray}
        \mathrm{Var}\left[\frac{1}{n}\mathcal{M}_{\tau f}(\ba + \bg) \right] \leq \frac{c}{n^2}\mathbb{E}_{\bg}\left[||\nabla_\bg\mathcal{M}_{\tau f}(\ba + \bg)||_2^2 \right] = \frac{c}{n^2}\mathbb{E}_{\bg}\left\|\frac{1}{\tau}\left(\bg + \ba - \mathrm{prox}_{\tau f}(\bg + \ba) \right) \right\|_2^2
    \end{eqnarray}
    
    From \citep{bauschke2011convex}[Proposition 12.28 and Proposition 4.4], the function $f(\bg + \ba) = \bz - \mathrm{prox}_{\tau f}(\bg + \ba)$ is firmly non-expansive and 
    \begin{eqnarray}
        ||\bg + \ba - \mathrm{prox}_{\tau f}(\bg + \ba)||_2^2 \leq \braket{\bg + \ba|\bg- \mathrm{prox}_{\tau f}(\bg + \ba)}
    \end{eqnarray}
    which implies that
    \begin{eqnarray}
        ||\bg + \ba - \mathrm{prox}_{\tau f}(\bg + \ba)||_2^2 \leq ||\bg + \ba|_2^2
    \end{eqnarray}
    by means of the Cauchy Swarchz inequality. 
\end{proof}
This implies that 
\begin{eqnarray}
    \mathrm{var}\left[ \frac{1}{n} \mathcal{M}_{\tau f}(\ba  + \bg) \right] \leq \frac{c}{n^2\tau^2} \mathbb{E}\left\|||\bg + \ba||_2^2\right\| = \frac{c(n + ||\ba||_2^2)}{n^2\tau^2} \leq \frac{C}{n\tau^2}
 \end{eqnarray}
 in which we have used the fact that the norm of $\ba$ is bounded. By making use of Chebyshev's inequality we obtain that
 \begin{eqnarray}
        \mathbb{P}\left(\left|\frac{1}{n}\mathcal{M}_{\tau f}(\ba + \bg) - \mathbb{E}\left[\frac{1}{n}\mathcal{M}_{\tau f}(\ba + \bg) \right] \right| \geq \epsilon\right) \leq \frac{c}{n\tau^2\epsilon^2}
 \end{eqnarray}


\section{Analysis of Universality}
\label{App:UniversalityTheorems}
We recall the definition of the perturbed optimization problem as a function of the feature map
\begin{equation}
    \label{app:eq:pertubedProblem}
    P(\tau_1, \tau_2) = \min_{\be}\frac{1}{2n}\left\|\bepsilon - \frac{1}{\sqrt{m}}\bX\be\right\|_2^2 + \frac{1}{m}r(\be+\btheta^*) + \frac{1}{m}\tau_1\be\bR\be + \frac{1}{m}\tau_2 h(\be+\btheta^*)
\end{equation}
and
\begin{equation}
    \label{app:eq:pertubedProblem_tilde}
    \tlP(\tau_1, \tau_2) = \min_{\be}\frac{1}{2n}\left\|\bepsilon - \frac{1}{\sqrt{m}}\tlbX\be\right\|_2^2 + \frac{1}{m}r(\be+\btheta^*) + \frac{1}{m}\tau_1\be\bR\be + \frac{1}{m}\tau_2 h(\be+\btheta^*),
\end{equation}
where $\bX,\tlbX$ are respectively generated by the following two alternative feature maps
\begin{eqnarray}
\bvarphi(\bz) = \sigma\left(\frac{1}{\sqrt{d}}\bW\bz\right)\\
\tilde\bvarphi(\bz) = \frac{\rho_1}{\sqrt{d}}\bW\bz + \rho_*\bg,
\end{eqnarray}
which lead to the following two covariance matrices 
\begin{eqnarray}    
    \bR = \mathbb{E}_{\bz}[\bvarphi(\bz)\bvarphi^T(\bz)] = \mathbb{E}_{\bz}\left[ \sigma\left(\frac{1}{\sqrt{d}}\bW\bz\right)\sigma^T\left(\frac{1}{\sqrt{d}}\bW\bz\right) \right]\\
    \tlbR = \mathbb{E}_{\bz}[\tilde\bvarphi(\bz)\tilde\bvarphi^T(\bz)] =  \frac{\rho_1^2}{d}\bW\bW^T + \rho_*^2\bI
\end{eqnarray}

Now recall the function $B(\be) = r(\be + \btheta^*) +\tau_1\be\bR\be + \tau_2h(\be+\btheta^*)$. We recall that $r$ is assumed to be $\mu$-strongly convex. The values $\tau_1 \in [-\tau_1^*, \tau_1^*]$ and $\tau_2 \in [-\tau_2^*, \tau_2^*]$, with the bounds $\tau_1^*$ and $\tau_2^*$ chosen to be sufficiently small such that $B$ remains $\frac\mu 4$-strongly convex

We can now state a theorem concerning Universality that is an extension of Theorem 1 in \cite{HuUniversalityLaws}
\begin{theorem}[Extension of \cite{HuUniversalityLaws}]
\label{thm:UnivHuLu}
Assume that assumptions A3-A6 hold. Fix $\tau_1 \in [-\tau_1^*, \tau_1^*]$ and $\tau_2 \in [-\tau_2^*, \tau_2^*]$. Finally assume that the regularization function $r(\btheta)$ is strongly convex, thrice differentiable with bounded third derivative.\\
Then for every $\epsilon \in (0,1)$ and every finite constant $c$, we have that
\begin{eqnarray}
    \mathbb{P}(|P( \tau_1, \tau_2) - c| \geq 2\epsilon) \leq \mathbb{P}(|\tlP(\tau_1,\tau_2) - c| \geq \epsilon) + \frac{\mathrm{polylog} m}{\epsilon\sqrt{m}}
\end{eqnarray}
and
\begin{eqnarray}
    \mathbb{P}(|\tlP(\tau_1, \tau_2) - c| \geq 2\epsilon) \leq \mathbb{P}(|P(\tau_1,\tau_2) - c| \geq \epsilon) + \frac{\mathrm{polylog} m}{\epsilon\sqrt{m}}
\end{eqnarray}
for $m\geq \frac{1}{\epsilon^2}$, in which $\mathrm{polylog} m$ is a function that grows no faster than a polynomial of $\log m$. Consequently,
\begin{equation}
    P(\tau_1, \tau_2)\xrightarrow[n, m, d\rightarrow\infty]{P} c\quad \mathrm{iff}\quad  \tlP(\tau_1, \tau_2) \xrightarrow[n,m,d\rightarrow\infty]{P} c
\end{equation}
\end{theorem}

This theorem is different than the one presented by \citep{HuUniversalityLaws} in two ways. Firstly we have restricted ourselves to the square loss function which simplifies this analysis, we discuss this difference in remark \ref{remark:HuLuLossfunctionDifference}. Secondarily, the term associated with $\tau_2$ is different. We consider generic test functions $h(\btheta)$ satisfying assumptions A2, \citep{HuUniversalityLaws} only consider one particular case of $h(\btheta) = \frac{\rho_1\sqrt{m}}{\sqrt{d}}\bm{\xi}^T\bW\btheta$ in which $\bm{\xi}$ is their teacher vector. The changes required to their proof to apply to generic test functions are minimal, and we give an outline in proof sketch below.

\begin{remark}
\label{remark:HuLuLossfunctionDifference}
We note that the conditions considered by \cite{HuUniversalityLaws} are slightly different than the case considered here. However the proof is sufficiently generic that it applies to the case considered here. Specifically, \citep{HuUniversalityLaws} consider a generic strongly convex and thrice differentiable loss function $l(\frac{1}{\sqrt{m}}\bvarphi(\bz_i)\btheta, y_i)$ for a particular data element $i$. For the labels $y_i$, \citep{HuUniversalityLaws} consider a function $\psi_{teach}(\bz_i^T\bm{\xi})$ in which $\bm{\xi}$ is a teacher vector and $\psi_{teach}$ is a differentiable function (except at a finite number of points) and is bounded by
\begin{eqnarray}
\forall x\in\mathbb{R}\quad \psi_{teach}(x) \leq C(1+|x|^K)
\end{eqnarray}
for some constants $C>0$ and positive integer $K$. They then prove their results for the joint distribution \\$(\frac{1}{\sqrt{m}}\bvarphi(\bz)^T\btheta;\bz^T\bm{\xi})$, which is jointly Gaussian through the variable $\bz$. In the case considered in this paper, we consider $y_i = \frac{1}{\sqrt{m}}\bvarphi(\bz_i)\btheta^* + \epsilon_i$ for some known vector $\btheta^*$ and noise $\epsilon_i$, and specifically choose the square loss. This allows for the definition of the error vector $\be = \btheta - \btheta^*$, and allows us to instead consider the distribution $(\frac{1}{\sqrt{m}}\bvarphi(\bz)\be;\epsilon_i)$ which simplifies the analysis in this case.
\end{remark}

\subsection{Proof sketch}
Here we discuss how to extend the results of \cite{HuUniversalityLaws} to the case of generic test function $h(\btheta)$, instead of their particular choice of $\frac{\rho_1\sqrt{m}}{\sqrt{d}}\bm{\xi}\bW\btheta$. The structure and details of the entire proof remain almost unchanged, except for the following set of minor changes, where the equation numbers refers to \cite{HuUniversalityLaws}:

\begin{itemize}
    \item In equation 172 step (a) and in the proof of \citep{HuUniversalityLaws} lemma 19, the property that $H_{\setminus k} \succeq \frac{\mu}{2}\bI$, where
    \begin{eqnarray}
    H_{\setminus k} = \frac{1}{m}\sum_{i=0}^{k-1}l''(\frac{1}{\sqrt{m}}\tilde{\bvarphi}(\bz_i),\bar{\be})\tilde{\bvarphi}(\bz_i)\tilde{\bvarphi}^T(\bz_i) _+ \frac{1}{m}\sum_{i=k+1}^{n}l''(\frac{1}{\sqrt{m}}\bvarphi(\bz_i)\bar{\be})\bvarphi(\bz_i)\bvarphi^T(\bz_i) \nwl+ \mathrm{diag}\lbrace r''(\bar{\be} + \btheta^*) \rbrace + \nabla^2(\tau_1\bar{\be}^T\bR\bar{\be} + \tau_2h(\bar{\be} + \btheta^*)),
    \end{eqnarray}
    where $\bar{\be}$ is the optimal solution to the problem given in \cite{HuUniversalityLaws} (equation 32), $l$ is the loss function and $l''$ its second derivative, in our case the square loss. For the case of for our choice of $\tau_2\in[-\tau_2^*, \tau_2^*]$ and assumptions A2, and recalling that $r$ is $\mu$ strongly convex, this property holds. 
    \item Similarly they require that $R_{\setminus k}$ given in equation 187, defined as
    \begin{eqnarray}
    R_{\setminus k}(\btheta) = \sum_{i\neq k}l(\frac{1}{\sqrt{m}}\bvarphi(\bz_i)^T\btheta) + \sum_{j=1}^m r(\btheta) + \tau_1\btheta^T\bR\btheta + \tau_2h(\btheta)
    \end{eqnarray}
    to be $\frac{\mu}{2}$-strongly convex. Which obviously holds with our restrictions on $h$.
    \item In equation 210 they require that  that $G(\be) = r(\be + \btheta^*) + \tau_1\be^T\bR\be + \tau_2h(\be + \btheta^*)$ is $\frac{\mu}{2}$ strongly convex and that
    \begin{eqnarray}
    ||\nabla G(\bm{0})|| \leq C\sqrt{m}
    \end{eqnarray}
    Which is clearly satisfied by assumption A2 and assumption A5 (see errata for updated assumption A5).
    \item Finally in equation 252 they require that $c, c', C >0$
    \begin{eqnarray}
    \Pr\left(\max_i|(\nabla h(\btheta^*))_i|>c\log m\right) \leq Ce^{-c'(\log m)^2}
    \end{eqnarray}
    This boundedness is satisfied by assumption A5.
    
\end{itemize}
As these are the only changes necessary to prove \citep{HuUniversalityLaws} results for more generic test functions we do not reproduce the proof here in full.

\subsection{Universality of Generalization Error and Test functions $h$}
\label{app:subsec:generrorUniv}
In this section we demonstrate that the universality of generalization error holds for strongly convex and thrice differentiable regularization functions, making use of the perturbation that we defined above in problem \ref{app:eq:pertubedProblem}. We prove the following result based on results from \cite{HuUniversalityLaws}. For this theorem we require the following definition

\begin{definition}
    \label{def:kappapidef}
    In Theorem \ref{thm:GaussianAsymptotics}, we showed that $\tlP( \tau_1, \tau_2) \xrightarrow{P} \tlP_3(\tau_1, \tau_2)$. Let the partial derivatives of $\tlP(\tau_1, \tau_2)$ at $\tau_1 = \tau_2 = 0$ be denoted by $\frac{\partial}{\partial \tau_1}\tlP_3(0, 0) = \hat{\kappa}$ and $\frac{\partial}{\partial\tau_2}\tlP_3(0, 0) = \hat{\pi}$.
\end{definition}
 Note that derivatives may be readily computed as done in lemma \ref{lem:gaussianGenErrorConvergence}. We now state the following result 
\begin{theorem}[Universality of Generalization Error]
\label{sup:thm:universalityGenerror}
Assume the same assumptions hold as in theorem \ref{thm:UnivHuLu} and let $\hat{\kappa}$ and $\hat{\pi}$ be given in definition \ref{def:kappapidef}. Take the Generalization error for a given feature map as
\begin{equation}
    \mathcal{E}_{gen}(\btheta, \bvarphi) = \mathbb{E}\left(y_{new} - 
\frac{1}{\sqrt{m}}\bvarphi(\bz_{new})^T\btheta\right)^2,
\end{equation}
where $\bz_{new}\sim\mathcal{N}(0, I_d)$ and $y_{new} = \frac{1}{\sqrt{m}}\bvarphi(\bz_{new})^T\btheta^* + \epsilon_{new}$, where $\epsilon_{new}$ is noise. Then

\begin{equation}
    \mathcal{E}_{gen}(\hat{\btheta}_1, \bvarphi) \rightarrow \mathcal{E}^*_{gen}\qquad \mathrm{and}\qquad \mathcal{E}_{gen}(\hat{\btheta}_2, \tilde{\bvarphi}) \rightarrow \mathcal{E}^*_{gen}
\end{equation}
in which
\begin{equation}
    \mathcal{E}^*_{gen} = \sigma_{\bepsilon}^2 + \hat{\kappa}
    \end{equation}
\end{theorem}

\begin{proof}
We let $\tau_2 = 0$ and let Let $\bz_{new}\sim\mathcal{N}(0, I_n)$ be a new Gaussian vector that is independent of all other training samples, and let $y_{new} = \frac{1}{\sqrt{m}}\btheta^{*T}\bvarphi(\bz) + \epsilon_{new}$. We can then express the generalization errors as
\begin{eqnarray}
    \mathcal{E}_{gen}(\hat{\btheta}_1, \bvarphi) = \mathbb{E}_{\epsilon_{new}, \bz_{new}}\left[\epsilon_{new} - \frac{1}{\sqrt{m}}\bvarphi(\bz_{new})\hat{\be}_1\right]^2 = \sigma_{\bepsilon}^2 + \frac{1}{m}\hat{\be}_1\bR\hat{\be}_1\\
    \mathcal{E}_{gen}(\hat{\btheta}_2, \tilde\bvarphi) = \mathbb{E}_{\epsilon_{new}, \bz_{new}}\left[\epsilon_{new} - \frac{1}{\sqrt{m}}\tilde\bvarphi(\bz_{new})\hat{\be}_2\right]^2 = \sigma_{\bepsilon}^2 + \frac{1}{m}\hat{\be}_2\tlbR\hat{\be}_2
\end{eqnarray}

Let $\kappa_2 =\frac{1}{m}\hat{\be}_2^T\tilde{\bR}\hat{\be}_2$, from which we see that $\mathcal{E}_{gen}(\tilde{\bvarphi}) = \sigma_{\bepsilon}^2 + \kappa_2$.
We start by noting that by lemma \ref{lem:gaussianGenErrorConvergence} that $\mathcal{E}_{gen}(\hat{\btheta}_{2}, \bvarphi_2) = \sigma_{\bepsilon}^2 + \kappa_2\rightarrow \sigma_{\bepsilon}^2 + \hat{\kappa}$, which proves the second claim.
Now, we consider the value of $\kappa_1 = \frac{1}{m}\hat{\be}^T_1\bR\hat{\be}_1$.
By the definition of the optimization problem we have 
\begin{eqnarray}
    P(\tau_1, \tau_2 = 0) \leq P(0, 0) + \tau_1\hat{\be}_2\bR_1\hat{\be}_2
\end{eqnarray}
For any $\tau_1$. From this it follows that for any $\tau> 0$ we have
\begin{eqnarray}
\label{eq:kappanineq}
    \frac{P(\tau, 0) - P( 0, 0)}{\tau} \leq \kappa_1 \leq \frac{P( -\tau, 0) - P( 0, 0)}{-\tau}
\end{eqnarray}
We choose an $\epsilon > 0$. By definition \ref{def:kappapidef} the limit function $\tlP_3(\tau_1,\tau_2)$ is differentiable at the origin
and we know from theorem \ref{thm:UnivHuLu} and \ref{thm:GaussianAsymptotics} that $P(\tau_1, \tau_2) \xrightarrow[]{P}\tlP_3(\tau_1, \tau_2)$.On the other hand there exists some $\delta > 0$ such that

\begin{eqnarray}
    \left|\frac{\tlP_3(\delta, 0) - \tlP_3(0, 0)}{\delta}  - \hat{\kappa}\right|\leq \frac{\epsilon}{3}.
\end{eqnarray}
Substituting this into the first inequality of equation \ref{eq:kappanineq} above and letting $\tau = \delta$ we obtain

\begin{eqnarray}
    \label{eq:kappa2bound}
    \mathbb{P}(\kappa_1 -\hat{\kappa} < -\epsilon) \leq \mathbb{P}\left(\frac{P(\delta, 0) - P( 0, 0)}{\delta} - \hat{\kappa} < -\epsilon  \right) \nwl
    \leq \mathbb{P}(|P(\delta, 0) - \tlP_3(\delta, 0)| \geq \delta\epsilon/3 ) + \mathbb{P}(|P(0, 0) - \tlP_3(0, 0))| \geq \delta \epsilon/3)
\end{eqnarray}
Now by assumption we have that $P(\delta, 0) \xrightarrow[]{P}\tlP_3(\delta, 0)$ and  $P(0, 0) \xrightarrow[]{P}\tlP_3(0, 0)$. It then follows from Eq \ref{eq:kappa2bound} that $\lim_{n\rightarrow\infty}\mathbb{P}(\kappa_1 - \hat{\kappa} < - \epsilon) = 0$. The exact same reasoning may be applied to second inequality \ref{eq:kappanineq} to obtain that $\lim_{n\rightarrow\infty}\mathbb{P}(\kappa_1 - \hat{\kappa} > \epsilon) = 0$ as such $\kappa_1 \xrightarrow{P} \hat{\kappa}$.

\end{proof}

We now prove the universality of the test functions $h(\btheta)$.

\begin{theorem}[Universality of Test Functions]
Assume that the same assumptions hold as in theorem \ref{thm:UnivHuLu} and let $\hat{\pi}$ be given in definition \ref{def:kappapidef}. Then 
\begin{eqnarray}
\frac{1}{m}h(\hat{\btheta_2}) \rightarrow \hat{\pi} \quad \mathrm{and}\quad \frac{1}{m}h(\hat{\btheta_1})\rightarrow \hat{\pi} 
\end{eqnarray}
\end{theorem}
\begin{proof}
Our proof takes a similar form to the proof of theorem \ref{sup:thm:universalityGenerror}. We let $\tau_1 = 0$. Then we note that by the definition of the optimization problems
\begin{eqnarray}
P( \tau_1 = 0, \tau_2) \leq P( 0, 0) + \tau_2h(\hat{\btheta}_1)
\end{eqnarray}
for any $\tau_2$. It follows that for any $\tau > 0$ we have that
\begin{eqnarray}
\label{eq:pifuncinequalities}
\frac{P(0,\tau) - P( 0, 0)}{\tau} \leq h(\hat\btheta_1) \leq \frac{P( 0, 0)- P( 0, -\tau)  }{-\tau} 
\end{eqnarray}
We choose $\epsilon > 0$. By definition \ref{def:kappapidef} the limit function $\tlP_3(\tau_1,\tau_2)$ is differentiable at the origin. Therefore there exists some $\delta$ such that
\begin{eqnarray}
\left|\frac{\tlP_3(0, \delta) - \tlP_3(0, 0)}{\delta} - \hat{\pi} \right|\leq \frac{\epsilon}{3}
\end{eqnarray}
we substitute this into the first inequality of equation \eqref{eq:pifuncinequalities} above and let $\tau = \delta$. We obtain
\begin{eqnarray}
\Pr(\frac{1}{m}h(\hat{\btheta}_1) - \hat{\pi}< - \epsilon) \leq \Pr\left(\frac{P(0,\delta) - P( 0, 0)}{\delta} - \hat{\pi}< -\epsilon \right) \nwl
\leq \Pr(|P(0,\delta) - \tlP_3(0, \delta) |>\delta\epsilon/3) + \Pr(|P(0, 0) - \tlP_3(0, 0)| > \delta\epsilon/3)
\end{eqnarray}
Because of the universality laws given in theorem \eqref{thm:UnivHuLu} and Theorem \ref{thm:GaussianAsymptotics}, we know that $P( \tau_1, \tau_2)\xrightarrow{P}\tlP_3(\tau_1, \tau_2)$. It then follows that $\lim_{n\rightarrow\infty}\Pr(\frac{1}{m}h(\hat{\btheta}_1) - \hat{\pi}<-\epsilon) = 0$. The exact same reasoning may be applied to the second inequality of \eqref{eq:pifuncinequalities} to obtain that $\lim_{n\rightarrow\infty}\Pr(\frac{1}{m}h(\hat{\btheta}_1) - \hat{\pi}>\epsilon) = 0$. As such we conclude that $\frac{1}{m}h(\hat{\btheta}_1)\rightarrow \hat{\pi}$. Similar argument proves the result for $\hat\btheta_2$.
\end{proof}

\subsection{Proof of Theorem \ref{thm:universality:sequence}}
\label{app:subsec:univSequenceTheorem}
We now consider our theorem \ref{thm:universality:sequence}. We first show that the thrice differentiability condition of $r(\btheta)$ can be lifted in the case that there exist a sequence of function that are differentiable and converge to $r$.

\begin{lemma}
    Let $r^{(k)}(\btheta)$ be a sequence of functions that are each thrice differentiable and strongly convex. Assume further that $r^{(k)}(\btheta)$ converge uniformly to the regularization function $r(\btheta)$ in the limit of $k\rightarrow\infty$. 
    Then, the results of theorem \ref{thm:UnivHuLu} hold for this regularization function $r(\btheta)$. 
\end{lemma}
\begin{proof}
    We define $P^{(k)}( \tau_1, \tau_2), \tlP^{(k)}( \tau_1, \tau_2)$ to be the optimal cost for the regularization function $r^{(k)}$ respectively with feature map $\bvarphi,\tilde\bvarphi$. We choose $k$ to be sufficently large such that $|r^{(k)}(\btheta) - r(\btheta)| < m\epsilon$ for every $\btheta$. This implies that $|P^{(k)}( \tau_1, \tau_2) - P(\tau_1, \tau_2)| < \epsilon$ and $|\tlP^{(k)}( \tau_1, \tau_2) - \tlP(\tau_1, \tau_2)| < \epsilon$. Furthermore, by theorem \ref{thm:UnivHuLu}, we have
    \begin{equation}
        \mathbb{P}(|P^{(k)}( \tau_1, \tau_2) - c |>4\epsilon) \leq \mathbb{P}(|\tlP( \tau_1, \tau_2) - c| > 2\epsilon)+\frac{\mathrm{polylog}(m)}{\sqrt{m}}
    \end{equation}
    and hence
    \begin{equation}
        \mathbb{P}(|P( \tau_1, \tau_2) - c |>5\epsilon) \leq \mathbb{P}(|\tlP( \tau_1, \tau_2) - c| > \epsilon)+\frac{\mathrm{polylog}(m)}{\sqrt{m}}
    \end{equation}
    
    The other case is similarly proven.
\end{proof}
We note that this proof hold analogously for $r(\btheta) + \tau_2h(\btheta)$ for a test function $h(\btheta)$ that satisfies the conditions of assumption A2. This completes the proof of theorem \ref{thm:universality:sequence}. 

\subsection{Proof of Corollary~\ref{lem:universality:elasticNet}}
\label{sec:pf:universality:elasticNet}

We now consider elastic net regularization
\begin{equation}
    \label{eq:elasticnetregEpsilon}
    r(\btheta) = \lambda||\btheta||_1  + \frac{\epsilon}{2}||\btheta||_2^2
\end{equation}
The following lemma demonstrates that can construct a sequence of regularization function that uniformly converges to Eq. \ref{eq:elasticnetregEpsilon}. This result together with Theorem \ref{thm:universality:sequence} shows that for all $\epsilon > 0$ universality is established.
\begin{lemma}
There exists a sequence of function $r^{(k)}(\btheta)$ that are separable, strongly convex and thrice differentiable that converge uniformly to the elastic net regularization function given in Eq. \ref{eq:elasticnetregEpsilon}.
\end{lemma}
\begin{proof}
    Define $h^{(k)}(x)$ as
    \begin{equation}
        h^{(k)}(x) = x\ \mathrm{erf}\left(\frac{\sqrt{k}x}{\sqrt{2}}\right) + \frac{\sqrt{\frac{2}{\pi}}e^{-\frac{kx^2}{2}}}{\sqrt{k}}
    \end{equation}
    in which $\mathrm{erf}$ is the error function. It is simple to verify that $h^{(k)}(x)$ is thrice differentiable and has bounded third derivative. The maximum difference between $|x|$ and $h^{(k)}(x)$ is at $x = 0$ and is $\sqrt{\frac{2}{k\pi}}$. A such in the limit of $k\rightarrow \infty$, $h^{(k)}(x)$ converges uniformly to the absolute value function. We choose
    \begin{equation}
        r^{(k)}(\btheta) = \frac{\epsilon}{2}||\btheta||_2^2 + \sum_i^{m} h^{(k)}(\theta_i)
    \end{equation}
    This regularization function statisfies the conditions of the lemma.
\end{proof}

\subsection{Proof of Theorem \ref{thm:l1Universality}}
\label{app:subsec:pf:l1Universality}

First consider the universality with elastic net which is proven in Section~\ref{sec:pf:universality:elasticNet}.
To demonstrate universality with respect to the $\ell_1$ norm, we take the case of elastic net in \eqref{eq:elasticnetregEpsilon} with a sufficiently small $\epsilon$ drop the quadratic part of $r(\btheta)$. We note that for any $\epsilon > 0$, Theorem \ref{thm:universality:sequence} holds. Our goal will be to show that for very small values of $\epsilon$ removing $\epsilon$ does not substantially change the value of the training and testing error. We first make the following definitions

\begin{definition}
\label{def:RIP}
Consider an $m\times n$ matrix $\bA$.
\begin{enumerate}
    \item For any $k\in\mathbb{N}$ such that $k<n$, the RIP constant $\delta_k(\bA)$ is the smallest number $\delta$, such that for any index subset $I\subset {1,2,\dots, n}$ with $|I| \leq k$
    \begin{equation}
        1 - \delta \leq \sigma^2_{min}(\bA_I) \leq \sigma^2_{max}(\bA_I) \leq 1+\delta
    \end{equation}
    in which $\sigma_{min}$ and $\sigma_{max}$ are the minimum and maximum singular values.
      \item Let $\theta_k(\bA)$ for any $k<n/2$ be the smallest number $\theta$ such that for any disjoint subsets $I, I' \subset \lbrace 1, 2, \ldots n \rbrace$ with $|I|,\ |I'| \leq k$ it holds that
\begin{equation}
    \sigma_{max}(\bA_{I'}^T\bA_I) \leq \theta
\end{equation}
In which $\sigma_{max}(\bX)$ is the maximum singular value of a matrix $\bX$. It is known that $\theta_k\leq \delta_{2k}$.
    \item We define the admissible sparsity $M_{adm}(\bA)$ as
    \begin{equation}
        M_{adm}(\bA) = \sup_{k} \frac{k[1-\delta_k(\bA)]_+}{2n}
    \end{equation}
    in which $[\cdot]_+$ represents the positive part.
\end{enumerate}
\end{definition}
This admissible sparsity $M_{adm}$ is the constant $\rho$ given in Theorem \ref{thm:l1Universality}. Provided that $M_{0}$, the effective sparsity given in \eqref{eq:eff}, is strictly less than $M_{adm}$ the theorem holds.

For our proof we require the following lemma. The original lemma is given in \cite{panahi2017universal} but we have extended it here to a slightly more general setup.

\begin{lemma}[Extension of \cite{panahi2017universal} lemma 8]
Suppose that $\sigma$ is $1-$Lipschitz. For the feature matrix $\bX=\left(\sigma\left(\frac 1{\sqrt{d}}\bw_j^T\bz_i\right)\right)_{ij}$ and $\bA=\frac 1{\sqrt{n(\rho_*^2+\rho_1^2)}}\bX$ there exist constants $\alpha, \beta > 0$ and $1 > \epsilon > 0$ such that
\begin{eqnarray}
        \lim_{n\rightarrow\infty} \mathbb{P}(\delta_{\alpha n}(\bA) + \theta_{\alpha n}(\bA) > 1-\epsilon) = 0\\
        \lim_{n\rightarrow\infty} \mathbb{P}(\sigma_{max}(\bA) > \beta) = 0
    \end{eqnarray}
    
\end{lemma}
\begin{proof}
    Consider an arbitrary subset $I\subset [m]$ with $|I|=k$. Let $\bu \in \mathbb{S}^{k-1}$, where $\mathbb{S}^{k-1}$ is the surface of the unit sphere in $\mathbb{R}^k$. We note given $\bW$ that $\by = \bA_I\bu$ is an i.i.d, centered  vector and defining
    \begin{equation}
        f(\bx)=\frac 1{\sqrt{(\rho_*^2+\rho_1^2)}}\suml_{i=1}^k\sigma\left(\frac{\bw_j^T\bx}{\sqrt{d}}\right)u_j,
    \end{equation}
    we have $y_i=\frac 1{\sqrt{n}}f(\bx_i)$. We observe that
    \begin{equation}
        \nabla f(\bx)= \frac 1{\sqrt{ d(\rho_*^2+\rho_1^2)}}\bW\bv 
    \end{equation}
    where $\bv=\left(\sigma'\left(\frac{\bw_j^T\bx}{\sqrt{d}}\right)u_j\right)_j$. Note that by Lipschitz continuity $\|\bv\|_2\leq 1$ and hence
    \begin{equation}
        \|\nabla f\|\leq \frac 1{\sqrt{d(\rho_*^2+\rho_1^2)}} \|\bW\|_2
    \end{equation}
    Hence, by the standard random matrix results, $f$ is $\mu'=\mu \sqrt{1+\frac{m}{d}}-$Lipschitz, where $\mu:=\frac 1{\sqrt{(\rho_*^2+\rho_1^2)}}$, with high probability. Hence, $y_i$ is $\frac{\mu'}{\sqrt{n}}-$sub-Gaussian. The rest of the argument is conditioned on the event that $\|\bW\|_2$ is bounded and hence leads to $\mu'-$Sub-Gaussian variables. Define
    \begin{equation}
        \sigma^2:=n\var(y_i)=\mu^2\suml_{j,j'}^ku_ju_{j'}R_{jj'},
    \end{equation}
    where $\bR=(R_{jj'})$ is the  exact covariance matrix of the features. 
     It has been shown in the previous works e.g. \cite{HuUniversalityLaws},that  exists a constant $c$ such that
    \begin{equation}
        \Pr\left(|\sigma-1|>\epsilon\right)\leq\frac 1{c} e^{-nc\epsilon^2}
    \end{equation}
   Now, since $y_i$ are $\mu-$sub-Gaussian, it is standard to show that (\cite{honorio2014tight})
   \begin{eqnarray}
    \forall \lambda:\ |\lambda |<\frac{n}{4\mu^2},\qquad \mathbb{E}\left[e^{\lambda \left(y_i^2-\frac 1n\sigma^2\right)}\mid\bW \right] 
    \leq
    e^{\frac{16\mu'^4\lambda^2}{n^2}}
\end{eqnarray}
and hence, conditioned on $\bW$ (which satisfies $\mu'-$sub-Gaussianity) we have
 \begin{eqnarray}
\mathbb{P}(||\by||_2^2 \geq \sigma^2+\epsilon) = \mathbb{P}\left(\sum_i^n (y_i^2-\frac 1n\sigma^2) \geq \epsilon\right) \leq \min_{0<\lambda<\frac{n}{4\mu^2} }\left(\mathbb{E}\left[e^{\lambda\suml_{i=1}^n (y_i^2-\frac 1n\sigma^2)}\right]e^{-\lambda\epsilon} \right) \leq \min_{0<\lambda<\frac{n}{4\mu^2}}e^{\frac{16\mu'^4\lambda^2}n-\lambda\epsilon}
\end{eqnarray}
Hence, for sufficiently small $\epsilon$, we may choose $\lambda=\frac{n\epsilon}{32\mu'^4}$ and obtain
 \begin{eqnarray}
\mathbb{P}(||\by||_2^2 \geq \sigma^2+\epsilon)  \leq e^{-c n\epsilon^2}
\end{eqnarray}
where $c$ is a suitable constant that may grow in each appearance. We conclude that for a random $\bW$ we have
 \begin{eqnarray}\label{eq:bound_single_point}
\mathbb{P}(||\by||_2^2 \geq 1+2\epsilon)\leq \mathbb{P}(||\by||_2^2 \geq \sigma^2+\epsilon)+ \mathbb{P}(\sigma^2 \geq 1+\epsilon)\leq \frac 1ce^{-c n\epsilon^2}
\end{eqnarray}
We may repeat the above Chernoff bound on the event $\|\by\|^2\leq \sigma^2-\epsilon$, to conclude that
\begin{equation}
    \mathbb{P}\left(\left|||\by||_2^2-1\right| \geq 2\epsilon\right)\leq \frac 1ce^{-c n\epsilon^2}
\end{equation}
The rest of the proof is similar to \cite{panahi2017universal}. We note that for every $\Delta > 0$  there exists a set $G_k\subset\mathbb{S}^{k-1}$ of maximally $\left( \frac{3}{\Delta}\right)^k$ points such that for any $\bu \in\mathbb{S}^{k-1}$ there exists a point $\bu_1\in G_k$ such that $||\bu - \bu_1||_2 \leq \Delta$. We denote $B = \max_{\bu\in G_k}||\bA_I\bu||_2$ and $A = \sigma_{max}(\bA_I) = \max_{x\in\mathbb{S}^{k-1}}||\bA\bu||_2$ with its maximum being at $\bu_0$. From this we see that
\begin{equation}
    A = ||\bA_I\bu_0||_2 \leq ||\bA_I\bu_1||_2 + ||\bA_I(\bu_1 - \bu_0)||_2 \leq B +\Delta A
\end{equation}
in which $\bu_1$ is the point in $G_n$ closest to $\bu_0$. If $\Delta < 1$ we obtain
\begin{equation}
    \label{eq:sigmaMaxBound}
    \sigma_{max}(\bA_I) \leq \frac{\max\limits_{x\in G_n} ||\bA_I\bu||_2}{1-\Delta}
\end{equation}
This argument may be repeated for the minimum singular value to obtain
\begin{equation}
    \sigma_{min}(\bA_I) \geq \min_{\bu\in G_n}||\bA_I\bu||_2 - \sigma_{max}(\bA_I)\Delta
\end{equation}
From equation \ref{eq:sigmaMaxBound} we see that $\epsilon_0=\delta-\Delta-\Delta\delta>0$, we have
\begin{equation}
    \mathbb{P}(\sigma_{max}(\bA_I) > 1+\delta) \leq \mathbb{P}\left(\min_{\bu \in G_n}||\bA\bu||_2 > (1-\Delta)(1+\delta)\right) \leq \frac{1}{c}e^{-cn\epsilon_0^2}\left(\frac{3}{\Delta}\right)^k
\end{equation}
and 
\begin{eqnarray}
    &\mathbb{P}(\sigma_{min}(\bA_I) < 1-\delta) \leq \mathbb{P}\left(\min_{\bu \in G_n}||\bA\bu||_2 < 1-\delta+(1+\delta)\Delta\right)+\Pr\left(\sigma_{max}(\bA_I) > 1+\delta\right) \leq\nwl
    &\frac{2}{c}e^{-cn\epsilon_0^2}\left(\frac{3}{\Delta}\right)^k
\end{eqnarray}
Choose $k=n$, $\Delta=\frac 12$, $\epsilon_0^2>\frac{\log 6}{c}$ and $\delta=1+2\epsilon_0$. We observe that $\mathbb{P}(\sigma_{max}(\bA_I)> 1+\delta)\to 0$, which proves the second part. For the first part, note that by the union bound
\begin{equation}
    \Pr\left(\delta_k(\bA)>\delta\right)\leq \frac{3}{c}e^{-cn\epsilon_0^2}\left(\frac{3}{\Delta}\right)^k{n \choose k}
\end{equation}
Take for example $\Delta=\frac 15,\ \delta=\frac 13$, hence $\epsilon_0=\frac 15$. Furthermore, for $k=2\alpha n$, we have ${n \choose k}\sim e^{n H(2\alpha)}$ where $H(p)=-p\log p-(1-p)\log (1-p)$ is the entropy function.  Choosing $\alpha$ small enough such that $H(2\alpha)+2\alpha\log 15<c\epsilon_0^2$ will lead to $\Pr\left(\delta_k(\bA)>\delta\right)\to 0$. We conclude the first result by noting that $\delta_{\alpha n} + \theta_{\alpha n} \leq 2\delta_{2\alpha n}$. 

\end{proof}

We note that the feature matrix $\bX$ satisfies this lemma for most practical choices of the activation function. $\tanh$ and the error function are both odd activation functions that satisfy the assumptions A6 and produce a suitable matrix $\bX$. 

By the above lemma, we conclude Theorem \ref{thm:l1Universality} as the rest of the proof in  \cite{panahi2017universal} will hold true. For the sake of completeness we repeat these proofs in full.

\begin{theorem}
\label{thm:L1ProblemConverges}
Let assumptions A2-A6 hold and let $r(\btheta) = \lambda||\btheta||_1$. Denote
\begin{equation}
    P_{\lambda} = \min_{\be}\frac{1}{2n} ||\bepsilon + \bX\be||_2^2 + \frac{\lambda}{m}||\be + \btheta^*||_1
\end{equation}
Furthermore assume that there exist constants $\alpha, \beta, \epsilon$, such that
\begin{eqnarray}
    \lim_{n\rightarrow\infty} \mathbb{P}(\delta_{\alpha n}(\bX) + \theta_{\alpha n}(\bX) > 1-\epsilon) = 0\\
    \lim_{n\rightarrow\infty} \mathbb{P}(\sigma_{max}(\bX) > \beta) = 0
\end{eqnarray}
Then
\begin{equation}
    P_{\lambda} \xrightarrow[n, m, d\rightarrow\infty]{P} \tilde{P}_{3, \lambda}(\beta, q, \xi, r)
\end{equation}
In which $\tilde{P}_{3, \lambda}$ is $\tilde{P}_3$ as given in Eq. \eqref{eqn:A3problem} for the case that $r(\btheta) = \lambda||\btheta||_1$
\end{theorem}
\begin{proof}
    Let $\hat{\be}^{(0)}$ be the minimal point of the optimization
    \begin{equation}
           P_{\lambda, \mu} = \min_{\be}\frac{1}{2n} ||\bepsilon + \bX\be||_2^2 + \frac{\lambda}{m}||\be + \btheta^*||_1 + \frac{\mu}{2m }||\be||_2^2
    \end{equation}
    We know that from Lemma \ref{lemma:bounded_CGMT1} and \ref{lemma:bounded_CGMT2} that there exists a number $C_\be$ such that for every $\mu < 1$, $||\hat{\be}^{(0)}||_2^2 \leq C_{\be}^2m$ with high probability. We define
    \begin{equation}
        p(\be) = \frac{1}{2}||\bepsilon - \bX\be||_2^2 + \frac{\lambda}{\gamma}||\be + \btheta^*||_1
\end{equation}  
From the KKT conditions we know that
\begin{equation}
    -\mu\hat{\be}^{(0)} \in \partial p(\hat{\btheta}^{(0)})
\end{equation}
where $\partial$ represents the subdifferential. We define $\zeta^{(0)} = -\mu\hat{\be}^{(0)}$. We let $k = \alpha n$ and select $k$ entries of $\hat{\be}^{(0)}$ with the largest absolute values and collect their indices in $I_0$. We set $\ba_0 = 0 \in\mathbb{R}^k$ and let $t =0$. We now perform the subsequent iterative algorithm.
\begin{enumerate}
    \item Define $\bA_t = \bX_{I_t}$ and let $\bh_t = \bepsilon + \bX_{I^c_t}\hat{\be}^{(t)}_{I^c}$ and solve
    \begin{equation}
        \label{eq:iterativeProblemMin}
        \min_{\bw} \frac{1}{2}||\bh_t + \bA_t\bw||_2^2 + \lambda||\btheta^*_{I_t} + \bw||_1 - \ba_t^T\bw
    \end{equation}
    define its cost function and optimal point by $p_t(\bw)$ and $\bw_t$ respectively
    \item Find $k$ elements in $I_t^c$ with largest absolute values in $\bX_{I_t^c}^T\bX_{I_t}(\bw_t - \be_{I_t}^{(t)})$. We denote the indices by $I_{t+1}$. We set $\ba_{t+1} = \zeta_{I_{t+1}}^{(t)}$
    \item We construct $\be^{(t+1)}$ and $\zeta^{(t+1)}$ such that $\be_{I_t}^{(t+1)} = \bw_t$, $\be_{I_t^c}^{(t+1)} = \be_{I_t^c}^{(t)}, \zeta_{I_t}^{(t+1)} = \ba_t$ and $\zeta_{I_t^c}^{(t+1)} = \zeta^{(t)}_{I_t^c} + \bX_{I_t^c}^T\bX_{I}(\bw_t - \be_I^{(t)})$
    \item we let $t \leftarrow t+1$ and return to step 1.
\end{enumerate}

In Lemma \ref{lemma:iterativeProcessFinishes} below we show that this iterative process results in a point $\be^{(\infty)}$ with subgradient $\zeta^{(\infty)}\in\partial p(\be^{(\infty)})$, such that
\begin{eqnarray}
    \frac{1}{\sqrt{m}}||\be^{(\infty)} - \be^{(0)}||_2 \leq \frac{\mu C_{\be}}{1 - \delta_k - \theta_k} \overset{def}{=} \mu C_2\\
    ||\zeta^{(\infty)}||_\infty \leq \mu C_{\be}\left( \sqrt{\frac{m}{k}} + \frac{\theta_k}{1-\delta_k - \theta_k} \right) \overset{def}{=} \mu C_1
\end{eqnarray}
We note that $\be^{(\infty)}$ is the optimal point of the optimization

\begin{equation}
    \rho_{\mu, \lambda} = \min_{\be} \frac{1}{2}||\bepsilon - \bX\be||_2^2 + \lambda||\btheta^* + \be||_1 + \be^T\zeta^{(\infty)}
\end{equation}

We shall let the subscripts $\lambda, \mu$ denote that a particular value of $\zeta^{(\infty)}$ or $\be^{(\infty)}$ are computed for particular values of $\lambda,\mu$. We now note that $\be^T\bzeta^{(\infty)} \leq ||\be||_1||\bzeta^{(\infty)}||_\infty \leq \mu C_1||\be||_1$, and as such
\begin{equation}
    \rho_{\mu, \lambda} \leq P_{\lambda +C_1\mu}
\end{equation}
Or equivalently we can express this as
\begin{equation}
    P_{\lambda} \geq \rho_{\mu, \lambda - C_1\mu}
\end{equation}
We also note that
\begin{eqnarray}
    m\rho_{\mu, \lambda} = \frac{1}{2}||\bepsilon + \bX\be^{(\infty)}||_2^2 + \lambda||\be^{(\infty)} - \btheta^*||_1 + \be^{(\infty)T}\bzeta^{(\infty)}\nwl
    \geq m P_{\lambda, \mu} + \mbf^T\bX(\be^{(\infty)} - \be^{(0)}) - \lambda||\be^{(\infty)} - \be^{(0)}||_1 + \be^{(\infty)T}\bzeta^{(\infty)}\nwl
    \geq m P_{\lambda, \mu} - (||\mbf^T\bX||_2 + \lambda\sqrt{m})||\be^{(\infty)} - \be^{(0)}||_2 - ||\be^{(\infty)}||_2||\bzeta^{(\infty)}||_2\nwl
    \geq m P_{\lambda, \mu} - (\sigma_{max}(\bX)||\mbf||_2 + \lambda\sqrt{m})\mu C_2\sqrt{m} - \sqrt{m}\mu C_1||\be^{(\infty)}||_2 \nwl
    \geq m P_{\lambda, \mu} - (\sigma_{max}(\bX)\kappa + \lambda)\mu C_2 m - m\mu C_1(C_{\be} + C_2\mu)
\end{eqnarray}
in which $\mbf = \bepsilon - \bX\be^{(0)}$, $\kappa$ is a proper bound that is independent of all other parameters, such that $||\bepsilon||_2 \leq \sqrt{m}\kappa$ with high probability. This holds by the law of large numbers, and we note that $||\mbf||_2 \leq ||\bepsilon||_2$. From this we find htat
\begin{equation}
    \rho_{\mu, \lambda} \geq P_{\lambda, \mu} - (\sigma_{max}(\bX)r +\lambda)\mu C_2 - \mu C_1(C_{\be} + C_2\mu)
\end{equation}
Noting that by Theorem \ref{thm:universality:sequence} that $P_{\lambda, \mu} \rightarrow \tilde{P}_{3\lambda, \mu}(\beta, q, \xi, r)$. We note that by the continuity of $\tilde{P}_{3\lambda, \mu}$ at $\mu = 0$, and for any $\epsilon > 0$, we can select a value of $\mu$ small enough such that
\begin{equation}
    \mathbb{P}(|P_{\lambda} - \tilde{P}_{3\lambda, \mu = 0}| > \epsilon) \xrightarrow[n,m,d\rightarrow\infty]{P} 0
\end{equation}

We also note that for any sufficiently small value of $\delta$ we see that
\begin{equation}
    \frac{P_{\lambda} - P_{\lambda-\delta}}{\delta} \leq \frac{||\hat{\btheta}||_1}{m} \leq \frac{P_{\lambda + \delta} - P_{\lambda}}{\delta} 
\end{equation}
From this we see that
\begin{equation}
    \frac{||\hat{\btheta}||_1}{m} \xrightarrow[n, m, d \rightarrow \infty]{P} \frac{\partial P_\lambda}{\partial \lambda}
\end{equation}

\end{proof}

\begin{lemma}
\label{lemma:iterativeProcessFinishes}
The iterative process defined in Theorem \ref{thm:L1ProblemConverges} produces a point $\be^{(\infty)}$ with subgradient $\bzeta^{(\infty)}$ that are bounded as
    \begin{eqnarray}
    \frac{1}{\sqrt{m}}||\be^{(\infty)} - \be^{(0)}||_2 \leq \frac{\mu C_{\be}}{1 - \delta_k - \theta_k} \overset{def}{=} \mu C_2\\
    ||\zeta^{(\infty)}||_\infty \leq \mu C_{\be}\left( \sqrt{\frac{m}{k}} + \frac{\theta_k}{1-\delta_k - \theta_k} \right) \overset{def}{=} \mu C_1
\end{eqnarray}
\end{lemma}
\begin{proof}
    Firstly we show that $\bzeta_t \in \partial p(\be^{(t)})$. We prove this by means of induction. We note that by definition $\bzeta_0 \in \partial p(\be^{(0)})$. For the iteration step we assume that $\bzeta_t \in \partial p(\be^{(t)})$. We  note that by the KKT conditions of the problem \ref{eq:iterativeProblemMin} we see that
    \begin{equation}
        (\bzeta_{t+1})_{I_t} = \ba_t \in \bX_{I_t}^T(\bepsilon + \bX\be^{(t+1)}) + \partial||\btheta^*_{I_t} + \be^{(t+1)_{I_t}}||_1
    \end{equation}
    Furthermore we have that
    \begin{equation}
        (\bzeta_{t})_{I_t^c} \in \bX^T_{I_t^c}(\bepsilon + \bX\be^{(t)}) + \partial||\btheta^*_{I_t^c} + \be^{(t)}_{I_t^c}||_1 
    \end{equation}
    From which we can see that
    \begin{equation}
        (\bzeta_{t+1})_{I_t^c} \in -\bX^T_{I_t^c}(\bepsilon + \bX\be^{(t+1)}) + \partial||\btheta^*_{I_t^c} + \be^{(t+1)}_{I_t^c}||_1 
    \end{equation}
    This shows that $\bzeta_{t+1} \in \partial p(\be^{(t+1)})$. This completes the induction.

    Next we will show by induction that
    \begin{eqnarray}
        \label{eq:induction1}
        \frac{1}{\sqrt{m}}||\be^{(t+1)} - \be^{(t)}||_2 \leq \frac{\mu C_{\be}}{1-\delta_k}\left( \frac{\theta_k}{1- \delta_k } \right)^t\\
        \label{eq:induction2}\zeta_{I_t}^{(t+1)} = \bzeta_{I_t}^{(t-1)}\\
    \label{eq:induction3} ||\bzeta^{(t+1)}_{(I_t\cup I_{t+1})^c} - \bzeta^{(t)}_{(I_t \cup I_{t+1})^c}||_\infty \leq \mu C_{\be}\left(\frac{\btheta_k}{1-\delta_k}\right)^{t+1}\sqrt{\frac{m}{k}}
    \end{eqnarray}

To prove this we first note that Eq \ref{eq:induction2} holds, as by definition $\bzeta_{I_t}^{(t-1)} = \bzeta_{I_t}^{(t+1)} = \ba_t$. We then note that
\begin{equation}
\frac{1}{\sqrt{m}}||\bzeta_{I_0^c}^{(0)}||_\infty = \min |\bzeta_{I_0}^{(0)}| \leq \mu C_{\be} \sqrt{\frac{m}{k}}
\end{equation}
We further note that $\bzeta_{I_0}^{(0)} \in \partial p(\bw = \be^{(0)}_{I_0})$. Therefore by lemma \ref{lemma:panahi2017Universal9} we see that
\begin{equation}
    \frac{1}{\sqrt{m}}||\bw_0 - \bv_{I_0}^{(0)}||_2 \leq \frac{||\bzeta_{I_0}^{(0)}||_2}{\sigma_{max}^2(\bX_{I_0})} \leq \frac{\mu C_{\be}}{1-\delta_k}
\end{equation}

We now note that at $t = 0$, by construction
\begin{equation}
    \begin{cases}
    \bzeta_{I_0}^{(0)} = \bm{0}\\
    \bzeta_{I_0}^{(1)} = \bzeta_{I_{0}}^{(0)} + \bX^T_{I_0^c}\bA_{I_0}(\bw_0 - \be_{I_0}^{(0)})
    \end{cases}
\end{equation}
and that $\bp_1 = \bzeta_{I_1}^{(0)}$. From this we see that
\begin{equation}
    \frac{1}{\sqrt{m}}\left\|\bX^T_{I_0^c}\bX_{I_0}(\bw_0 - \be_{I_0}^{(0)}) \right\|_2 \leq \frac{\theta_k}{\sqrt{m}}\left\|\bw_0 - \be^{(0)}_{I_0}\right\|_2 \leq \frac{\theta_k\mu C_{\be}}{1-\delta_k}
\end{equation}
From which we obtain
\begin{equation}
    \left\| \bX^T_{(I_0 \cup I_1)^c}\bX\bX_{I_0}(\bw_0 - \be_{I_0}^{(0)}) \right\|_\infty \leq \min \left|\bX_{I_1}^T\bX_{I_0}(\bw_0 - \be_{I_0}^{(0)}) \right| \leq \frac{\theta_k\mu C_{\be}}{1-\delta_k}\sqrt{\frac{m}{k}}
\end{equation}
Finally noting that for $t = 0$ we have that $||\be^{(1)} - \be^{(0)}||_2 = ||\bw_0 - \be_{I_0}^{(0)}||_2$. From this we see that the base case of the induction is satisfied.

We now assume that equations \ref{eq:induction1} - \ref{eq:induction3} hold for all $t'\leq t$ we now prove that they will hold for $t +1$. We consider the optimization \ref{eq:iterativeProblemMin} at step $t$, we also showed above that $\bzeta^{(t)} \in \partial p(\be^{(t)})$. From this we see that
\begin{equation}
    \bzeta_{I_t}^{(t)} - \ba_t \in \partial p_{t}(\be_{I_t}^{(t)})
\end{equation}
From this we see that
\begin{equation}
    \bX_{I_t}^T\bX_{I_{t-1}}(\bw_{t-1} - \be_{I_{t-1}}^{(t-1)}) \in \partial p_t(\be_{I_t}^{(t)})
\end{equation}
By Lemma \ref{lemma:panahi2017Universal9} below we see that
\begin{eqnarray}
    \frac{1}{\sqrt{m}}||\bw_{t} - \be_{I_t}^{(t)}||_2 \leq \frac{1}{(1-\delta_k)\sqrt{m}}||\bX_{I_t}^T\bX_{I_{t-1}}(\bw_{t-1} - \be_{I_{t-1}}^{(t-1)})||_2\nwl
    \leq \frac{\theta_k}{(1-\delta_k)\sqrt{m}}||\bw_{t-1} - \be_{I_{t-1}}^{(t-1)}||_2 \nwl 
     = \frac{\theta_k}{(1-\delta_k)\sqrt{m}}||\be^{(t)} - \be^{(t-1)}||_2 \nwl
     \leq \frac{\theta_k}{1-\delta_k}\frac{\mu C_{\be}}{1-\delta_k}\left( \frac{\theta_k}{1-\delta_k} \right)^{t-1} = \frac{\mu C_{\be}}{1=\delta_k}\left(\frac{\theta_k}{1-\delta_k} \right)^t
\end{eqnarray}
This proves Eq \ref{eq:induction2}. We also see that
\begin{equation}
   \frac{1}{\sqrt{m}}||  \bX_{I_{t+1}}^T\bX_{I_{t}}(\bw_{t} - \be_{I_{t}}^{(t)}) ||_2 \leq \theta_k ||\bw_t -\be_{I_t}^{(t)}||_2 \leq \frac{\theta_k\mu C_{\be}}{1-\delta_k}\left(\frac{\theta_k}{1-\delta_k} \right)^t
\end{equation}
Therefore
\begin{eqnarray}
    ||\bzeta_{(I_t \cup I_{t+1})^c}^{(t+1)} - \bzeta_{(I_t \cup I_{t+1})^c}^{(t)}||_\infty = ||   \bX^T_{(I_t \cup I_{t+1})^c}\bX_{I_{t}}(\bw_{t} - \be_{I_{t}}^{(t)}) ||_\infty\nwl
    \leq \min |\bX_{I_{t+1}}^T\bX_{I_t}(\bw_t - \be^{(t)}_{I_t})| \leq \sqrt{\frac{1}{k}}|| \bX_{I_{t+1}}^T\bX_{I_t}(\bw_t - \be^{(t)}_{I_t})||_2 \nwl\leq
    \sqrt{\frac{m}{k}}\mu C_{\be}\left( \frac{\theta_k}{1-\delta_k}\right)^{t+1}
\end{eqnarray}
This proves Eq \ref{eq:induction3}.

We now see in eq \ref{eq:induction1} that if $\theta_k + \delta_k < 1$, then the sequence of $\be^{(t)}$ is absolutely convergent. Furthermore, from \ref{eq:induction2} and \ref{eq:induction3}, in addition to the relation 
\begin{equation}
    \frac{1}{\sqrt{m}}||\bzeta_{I_{t+1}}^{(t+1)} - \bzeta_{I_{t+1}}^{t}||_2 = ||\bX_{I_{t+1}}^T\bX_{I_t}(\bw_t - \be_{I_{t}}^{(t)})||_2 \leq \mu C_{\be}\left( \frac{\theta_k}{1-\delta_k}\right)^{t+1}
\end{equation}
From this we obtain 
\begin{eqnarray}
    \frac{1}{\sqrt{m}}||\bzeta^{(t+1)} - \bzeta^{(t)}||_2  = \sqrt{ ||\bzeta_{I_t}^{(t+1)} - \bzeta_{I_t}^{(t)}||_2 + ||\bzeta_{I_{t+1}}^{(t+1)}- \bzeta_{I_{t+1}}^{(t)}||_2^2 + ||\bzeta_{(I_t \cup I_{t+1})^c}^{(t+1)} - \bzeta_{(I_t \cup I_{t+1})^c}^{(t)}||_2^2} \nwl 
    = \sqrt{ ||\bzeta_{I_t}^{(t-1)} - \bzeta_{I_t}^{(t)}||_2 + ||\bzeta_{I_{t+1}}^{(t+1)}- \bzeta_{I_{t+1}}^{(t)}||_2^2 + ||\bzeta_{(I_t \cup I_{t+1})^c}^{(t+1)} - \bzeta_{(I_t \cup I_{t+1})^c}^{(t)}||_2^2} \nwl
    \leq \sqrt{m}\sqrt{ \mu^2 C_{\be}^2\left(\frac{\theta_k}{1-\delta_k}\right)^{2t} + \mu^2 C_{\be}^2\left(\frac{\theta_k}{1-\delta_k} \right)^{2t +2} + \mu^2 C_{\be}^2(\frac{m}{k} - 1)\left( \frac{\theta_k}{1-\delta_k}\right)^{2t +2}  }
\end{eqnarray}
As such we see that the sequence $\bzeta^{(t)}$ is absolutely convergent. We denote the limits of $\bzeta^{(t)}$ and $\bv^{(t)}$ as $\bzeta^{(\infty)}$ and $\be^{(\infty)}$ respectively. 

We have that
\begin{eqnarray}
    \frac{1}{\sqrt{m}}||\be^{(0)} - \be^{(\infty)}||_2 \leq \sum_{t=0}^\infty ||\be^{(t+1)} - \be^{(t)}||_2 \leq \sum_{t=0}^\infty \frac{\mu C_{\be}}{1-\delta_k} \left( \frac{\theta_k}{1-\delta_k}\right)^t\nwl
    - \frac{\mu C_{\be}}{1-\delta_k -\theta_k}
\end{eqnarray}
Finally we show that $||\bzeta^{(\infty)}||_\infty$ is bounded as well. We consider an index $i$ and denote by $t_1 < t_2 < \ldots$ as the iterations of $t$ for which $i \in I_t$. In the case that $i\notin I_0$ by equation \ref{eq:induction2} we see that
\begin{equation}
    \zeta_i^{(\infty)} - \zeta_i^{(0)} = \sum_{t=0}^{\infty} \zeta_i^{(t+1)} - \zeta_i^{(t)} = \sum_{t| i\in (I_t \cup I_{t+1})^c} \zeta_i^{(t+1)} - \zeta_i^{(t)}
\end{equation}
As such we obtain
\begin{equation}
    |\zeta_i^{(\infty)}| \leq \zeta_i^{(0)} + \sum_{t| i\in (I_t \cup I_{t+1})^c} |\zeta_i^{(t+1)} - \zeta_i^{(t)}| \leq \mu C_{\be}\sqrt{\frac{m}{k}} + \mu C_{\be}\sum_{t}^\infty \left(\frac{\theta_k}{1-\delta_k} \right)^t \ leq \mu C_{\be}\left(\sqrt{\frac{m}{k}} + \frac{\theta_k}{1-\delta_k-\theta_k} \right)
\end{equation}
For any $i \in I_0$ we have hta 
\begin{equation}
     \zeta_i^{(\infty)} - \zeta_i^{(1)} = \sum_{t\geq 1| i\in (I_t \cup I_{t+1})^c} \zeta_i^{(t+1)} - \zeta_i^{(t)}
\end{equation}
By recalling that $\zeta_{I_0}^{(1)} = 0$ we obtain
\begin{equation}
    |\zeta_i^{(\infty)}| \leq \mu C_{\be} \sum_{t}^\infty \left(\frac{\theta_k}{1-\delta_k} \right)^t = \frac{\mu C_{\be}\theta_k^2}{(1-\delta_k-\theta_k)(1-\delta_k)}
\end{equation}
Combining the results in total we obtain

\begin{equation}
    ||\zeta^{(\infty)}||_{\infty} \leq \mu C_{\be}\left(\sqrt{\frac{m}{k}} + \frac{\theta_k}{1-\delta_k-\theta_k} \right)
\end{equation}
Finally we note that because for each $t$, $\zeta^{(t)} \in \partial p(\be^{(t)})$ we see that $\zeta^{(\infty)} \in \partial p(\be^{(\infty)})$

\end{proof}

\subsection{Universality of Generalization Error and Test Functions for $\ell_1$ regularization}
We first demonstrate the universality of the Generalization error. We demonstrate that the $2$-norm of the solution vector of the $\ell_1$ regularized case is asymptotically equivalent to the case of the elastic net regularized case for small values of $\ell_2^2$ regularization. We have already demonstrated that the generalization error for the elastic net case is universal, by showing that the $\ell_1$ is asymptotically equivalent we prove universality for that case as well.
\label{app:subsec:pf:l1trainingGenErrorUniv}
\begin{lemma}

Denote by $\hat{\btheta}^{\lambda, \epsilon}$ as the optimal point of 
\begin{equation}
    \label{eq:minmaxlambeps}
    P_{\lambda, \epsilon} = \min_{\btheta} \frac{1}{2n}||\by - \bX\btheta||_2^2 + \lambda||\btheta||_1 + \frac{\epsilon}{2}||\btheta||_2^2
\end{equation}
Under the conditions assumed in theorem \ref{thm:l1Universality}, for each $\eta >0 $, there exists $\epsilon, \rho$ such that for $0 < \epsilon < \eta$ and $|\rho| < \eta$, such that
\begin{equation}
    \mathbb{P}\left(\frac{||\hat{\btheta}^{\lambda +\rho, \epsilon} - \hat{\btheta}^{\lambda, 0}||_2^2}{n} > \eta \right) \rightarrow 0
\end{equation}
\end{lemma}
\begin{proof}
We first note that with a high degree of probability we have that
\begin{equation}
    M_0 + \theta < \frac{l(1-\delta_l(\bX))}{2n}
\end{equation}
in which $M_0$ is the effective sparsity given in equation \ref{eq:eff}, $\theta > 0$ is a fixed number and $l < n$ is natural number such that $\delta_l < 1$. From this we see that $(1-\delta_l) > 2(M_0+\theta)$ and $l/n > 2(M_0 + \theta)$. We let $0 < \alpha < \min(4M_0, 2\theta)$, and let $K =M_0 + \theta - \alpha/2$ and $k = \frac{l}{nK} - 1$. We note that $K > M_0 $ and
\begin{equation}
    k = \frac{l}{n(M_0 + \theta  - \alpha/2)} - 1  > \frac{l}{n(M_0 + \theta) } - 1 > 1
\end{equation}
Furthermore,
\begin{equation}
    K = M_0 + \theta - \frac{\alpha}{2} \leq \frac{l(1-\delta_l(\bX))}{2n} \leq \frac{l}{n}\left[\frac{1-\alpha -\delta_l(\bX)}{2-\alpha} + \frac{\alpha}{2} \right]- \frac{\alpha}{2} \leq \frac{l}{n}\left[ \frac{1 -\alpha - \delta_l(\bX)}{2-\alpha}\right]
\end{equation}
from which we can see that
\begin{equation}
\label{eq:alphabound}
    \alpha \leq \frac{k-1 - (k+1)\delta_l(\bX)}{k}
\end{equation}

We define a function $M_{r,\psi}$ in which $r$ is the regularization function and $\psi$ are both functions given by:
\begin{equation}
    \label{eqn:Mfunctiondef}
    M_{r, \psi}(\beta, q, \xi, r) = \mathbb{E}\left(\psi\left(\hat{\btheta}_r(\beta, q,\xi,r) \right)  \right)
\end{equation}
In which $\hat{\btheta}_r(\beta, q, \xi, t)$ is the optimal value of $\tilde{P}_3$ given in \eqref{eqn:A3problem} with regularization function $r$. We now define
\begin{equation}
    M^{\lambda, \epsilon} = M_{\lambda|\bx| + \frac{\epsilon}{2}x^2, x^2} \qquad N^{\lambda, \epsilon} = M_{\lambda|\bx| + \frac{\epsilon}{2}, |x|}
\end{equation}
Now let $\delta > 0$. We cansee that there exist value $\rho, \epsilon$ such that $0 < \epsilon < \delta,\ |\rho| < \delta$ such that $0 < N^{\lambda+ \rho, \epsilon} - N^{\lambda, 0}<\delta$. Then let $\mu > 0$ be defined such that
\begin{equation}
    2\mu < N^{\lambda + \rho, \epsilon} - N^{\lambda, 0}
\end{equation}
We define $\bh = \hat{\btheta}^{\lambda, 0} - \hat{\btheta}^{\lambda + \rho, \epsilon}$. We denote the objective function in Eq \ref{eq:minmaxlambeps} as $P_{\lambda,\epsilon}(\btheta)$. We have that

\begin{eqnarray}
    P_{\lambda + \rho, \epsilon}(\hat{\btheta}^{\lambda, 0}) = P_{\lambda, 0}(\hat{\btheta}^{\lambda, 0}) + \frac{1}{n}\left(\frac{\epsilon}{2}||\hat{\btheta}^{\lambda,0}||_2^2 - \rho||\hat{\btheta}^{\lambda, 0}||_1\right)\nwl
    \leq P_{\lambda, 0}(\hat{\btheta}^{\lambda + \rho, \epsilon}) + \frac{1}{n}\left(\frac{\epsilon}{2}||\hat{\btheta}^{\lambda,0}||_2^2 - \rho||\hat{\btheta}^{\lambda, 0}||_1\right) \nwl
    = P_{\lambda + \rho, \epsilon}(\hat{\btheta}^{\lambda + \rho, \epsilon}) + \frac{1}{n}\left(\frac{\epsilon}{2}||\hat{\btheta}^{\lambda,0}||_2^2 - \rho||\hat{\btheta}^{\lambda, 0}||_1 - \frac{\epsilon}{2}||\hat{\btheta}^{\lambda +\rho, \epsilon}||_2^2 - \rho_1||\hat{\btheta}^{\lambda + \rho, \epsilon}||_2^2\right)\nwl
    \leq P_{\lambda + \rho, \epsilon}(\hat{\btheta}^{\lambda + \rho, \epsilon}) + \frac{\epsilon}{2}\frac{||\bh||_2^2}{n} + \epsilon\frac{||\bh||_2}{\sqrt{n}} + \frac{\rho}{n}\left(||\hat
    {\btheta}^{\lambda, 0}||_1 - ||\hat{\btheta}^{\lambda + \rho, \epsilon}||_1\right) 
\end{eqnarray}

From theorem \ref{thm:universality:sequence} and part one of Theorem \ref{thm:l1Universality} we know that
\begin{equation}
    \label{eq:MandNconvergence}
    \frac{||\hat{\btheta}^{\lambda + \rho, \epsilon}||_2^2}{n} \xrightarrow[]{p}M^{\lambda + \rho, \epsilon} \quad \frac{||\hat{\btheta}^{\lambda + \rho, \epsilon}||_1}{n} \xrightarrow[]{p}N^{\lambda + \rho, \epsilon} \quad \frac{||\hat{\btheta}^{\lambda, 0}||_1}{n} \xrightarrow[]{p}M^{\lambda,0 } 
\end{equation}
Choosing a value of $M > \sqrt{M^{\lambda + \rho, \epsilon}}$, we obtain
\begin{equation}
P_{\lambda + \rho, \epsilon}(\hat{\btheta}^{\lambda, 0}) \leq P_{\lambda + \rho, \epsilon}(\hat{\btheta}^{\lambda + \rho, \epsilon}) + \frac{\epsilon}{2}\frac{||\bh||_2^2}{n} + M\epsilon\frac{||\bh||_2}{\sqrt{n}} + \rho\delta
\end{equation}

We now define the following index sets
\begin{equation}
    S = \lbrace k|\ |\hat{\btheta}_k^{\lambda + \rho, \epsilon}| \geq \mu \rbrace \qquad L = \lbrace k |\ 0 < |\hat{\btheta}_k^{\lambda + \rho, \epsilon}| < \mu|\rbrace
\end{equation}
We also define
\begin{equation}
    K^{\lambda, \epsilon}_\mu = M_{\lambda|x| + \epsilon x^2/2, \chi_{\mathbb{R}\setminus (-\mu, \mu)}}
\end{equation}
In which $\chi_A$ is the indicator function on the set $A$. By theorem \ref{thm:universality:sequence} we have that
\begin{equation}
    \frac{|S|}{n} \xrightarrow[]{P} K_\mu^{\lambda +\rho, \epsilon}
\end{equation}
we also see that
\begin{equation}
    \lim_{(\mu, \rho, \epsilon)\rightarrow 0} K_\mu^{\lambda + \rho, \epsilon} = M_0 
\end{equation}
Therefore, for small values of $\delta$ we know that $K_\mu^{\lambda + \rho, \epsilon} < K$ and as such with high probability 
\begin{equation}
    \frac{|S|}{n} < K
\end{equation}
We also know from equation \ref{eq:MandNconvergence} that with high probability
\begin{equation}
    \frac{||\hat{\btheta}^{\lambda +\rho, \epsilon}||_1}{n} - \frac{||\hat{\btheta}^{\lambda, 0}||_1}{n} > 2\mu
\end{equation}

This can equivalently be expressed as
\begin{eqnarray}
      \frac{||\hat{\btheta}^{\lambda +\rho, \epsilon}_S||_1}{n} +   \frac{||\hat{\btheta}^{\lambda +\rho, \epsilon}_L||_1}{n} >   \frac{||\hat{\btheta}^{\lambda +\rho, \epsilon}_S +\bh_S||_1}{n} +   \frac{||\hat{\btheta}^{\lambda +\rho, \epsilon}_L + \bh_L||_1}{n} + \frac{||\bh_{S\cup L}^c||_1}{n} + 2\mu \nwl
      \geq   \frac{||\hat{\btheta}^{\lambda +\rho, \epsilon}_s||_1 - ||\bh_S||_1}{n} +   \frac{||\bh_L|| -||\hat{\btheta}^{\lambda +\rho, \epsilon}_L||_1}{n}  +  \frac{||\hat{\btheta}^{\lambda +\rho, \epsilon}_L + \bh_L||_1}{n} + \frac{||\bh_{S\cup L}^c||_1}{n} + 2\mu 
\end{eqnarray}
By definition $||\hat{\btheta}^{\lambda  +\rho, \epsilon}_L||_1 \leq \mu$. As such with high probability we obtain
\begin{equation}
    ||\bh_S||_1 \geq ||\bh_{S^c}||_1
\end{equation}

We now define $\bz = \by - \bX\hat{\btheta}^{\lambda +rho, \epsilon}$. We wish to decompose the vector $\bh_{S^c}$ into block $T_1, T_2, \ldots$. We let $\bh_{T_1}$ be the $k|S|$ elements of $\bh_{S^c}$ with largest absolute value, $\bh_{T_2}$ are the next $k|S|$ largest absolute values and so on. Let $U = S \cup T_1$. With that we have
\begin{equation}
    nP_{\lambda+\rho, \epsilon}(\hat{\btheta}^{\lambda, 0}) = \frac{1}{2}||\bz - \bX\bh||_2^2 + \frac{(\lambda + \rho)}{\gamma}||\hat{\btheta}^{\lambda + \rho, \epsilon} + \bh||_1 + \frac{\epsilon}{2\gamma}||\hat{\btheta}^{\lambda + \rho, \epsilon} + \bh||_2^2
\end{equation}
We note that $\hat{\btheta}^{\lambda, 0} = \hat{\btheta}^{\lambda + \rho, \epsilon} + \bh$ which is the minimal point of the function $P_{\lambda, 0}(\btheta)$. As such we have
\begin{equation}
    \bX^T(\bz - \bX\bh) = \bX^T(\by - \bX\hat{\btheta}^{\lambda, 0}) \in \lambda \partial||\hat{\btheta}^{\lambda, 0}||_1
\end{equation}
Therefore
\begin{equation}
    ||\bX^T_{U^c}(\bz - \bX\bh)||_\infty \leq \lambda \Rightarrow - \bh^T_{U^c}\bX^T_{U^c}(\bz - \bX\bh) \geq -\lambda||\bh_{U^c}||_1
\end{equation}
From which we obtain
\begin{equation}
    - \bh^T_{U^c}\bX^T_{U^c}(\bz - \bX_U\bh_U) \geq -\lambda||\bh_{U^c}||_1 - ||\bX_{U^c}\bh_{U^c}||_2^2
\end{equation}
and as such
\begin{equation}
    \frac{1}{2}||\bz - \bX\bh||_2^2 = \frac{1}{2}||\bz - \bX_{U}\bh_U||_2^2 - \bh^T_{U^c}\bX^T_{U^c} + \frac{1}{2}||\bA_{U^c}\bh_{U^c}||_2^2  \geq   \frac{1}{2}||\bz - \bX_{U}\bh_U||_2^2 -\lambda||\bh_{U^c}||_1 + \frac{1}{2}||\bX_{U^c}\bh_{U^c}||_2^2
\end{equation}
From which we obtain
\begin{eqnarray}
    mP_{\lambda +\rho, \epsilon}(\hat{\btheta}^{\lambda, 0}) \geq  \frac{1}{2}||\bz - \bX_{U}\bh_U||_2^2 -\lambda||\bh_{U^c}||_1 + \frac{1}{2}||\bX_{U^c}\bh_{U^c}||_2^2 + (\lambda + \rho)||\hat{\btheta}^{\lambda+\rho, \epsilon}_U + \bh_U||_1 + \nwl (\lambda + \rho)||\hat{\btheta}^{\lambda+\rho, \epsilon}_{U^c} + \bh_{U^c}||_1  + \frac{\epsilon}{2}||\hat{\btheta}^{\lambda+\rho, \epsilon}_U + h_U||_2^2 + \frac{\epsilon}{2}||\hat{\btheta}^{\lambda+\rho, \epsilon}_{U^c}+ h_{U^c}||_2^2
\end{eqnarray}
We note that $\bw = 0$ is the minimum point of the function
\begin{equation}
    \frac{1}{2}||\bz - \bX_{U}\bw||_2^2  + (\lambda + \rho)||\hat{\btheta}^{\lambda + \rho, \epsilon} +\bw||_1 + \frac{\epsilon}{2}||\hat{\btheta}^{\lambda+\rho, \epsilon}_U + \bw||_2^2
\end{equation}
Therefore from lemma \ref{lemma:panahi2017Universal10} we get that
\begin{eqnarray}
\frac{1}{2}||\bz - \bX_{U}\bh_U||_2^2  + (\lambda + \rho)||\hat{\btheta}^{\lambda + \rho, \epsilon} +\bh_U||_1 + \frac{\epsilon}{2}||\hat{\btheta}^{\lambda+\rho, \epsilon}_U + \bh_U||_2^2\nwl
\geq \frac{\sigma_{max}^2(\bX_U)}{2}||\bh_U||_2^2 + \frac{1}{2}||\bz||_2^2 + (\lambda + \rho)||\hat{\btheta}^{\lambda + \rho, \epsilon}||_1 + \frac{\epsilon}{2}||\hat{\btheta}^{\lambda + \rho, \epsilon}||_2^2 
\end{eqnarray}

Substituing this in above we get that
\begin{eqnarray}
    nP_{\lambda + \rho, \epsilon}(\hat{\btheta}^{\lambda, 0}) -   nP_{\lambda + \rho, \epsilon}(\hat{\btheta}^{\lambda +\rho, \epsilon}) \nwl
    \geq  \frac{\sigma_{max}^2(\bX_U)}{2}||\bh_U||_2^2 -\lambda||\bh_{U^c}||_1 - \frac{1}{2}||\bA_{U^c}\bh_{U^c}||_2^2 + (\lambda + \rho)||\hat{\btheta}^{\lambda + \rho, \epsilon}_{U^c} + \bh_{U^c}||_1 - (\lambda + \rho)||\hat{\btheta}^{\lambda + \rho, \epsilon}_{U^c}||_1 \nwl- \frac{\epsilon}{2}||\hat{\btheta}^{\lambda + \rho, \epsilon}_{U^c}||_2^2 + \frac{\epsilon}{2}||\hat{\btheta}^{\lambda + \rho, \epsilon}_{U^c} + \bh_{U^c}||_2^2  \nwl
    \geq \frac{\sigma_{max}^2(\bX_U)}{2}||\bh_U||_2^2 + \rho||\bh_{U^c}||_1 - \frac{1}{2}||\bX_{U^c}\bh_{U^c}||_2^2 - 2(\lambda + \rho)||\hat{\btheta}^{\lambda+\rho, \epsilon}_{U^c}||_1 - 2||\hat{\btheta}^{\lambda + \rho, \epsilon}_{U^c}||_2||\bh_{U^c}||_2\nwl
    \geq \frac{\sigma_{max}^2(\bX_U)}{2}||\bh_U||_2^2 + \delta\sqrt{n}||\bh_U||_2 - \frac{1}{2}||\bX_{U^c}\bh_{U^c}||_2^2 - 2(\lambda + \rho)n\mu - 2\sqrt{n}\mu||\bh_{U^c}||_2
\end{eqnarray}
Where we have made use of the fact that
\begin{equation}
    \rho||\bh_{U^c}||_1 \geq -\delta||\bh_{U^c}||_1 \geq -\delta||\bh_U||_1 \geq -\delta\sqrt{n}||\bh_U||_2
\end{equation}
and in \citep{candes2006stable}(equation 11) it is proven that
\begin{equation}
    ||\bh_{U^c}||_1 \leq \frac{|S|}{|L|}||\bh_U||_2^2 = \frac{1}{k}||\bh_U||_2^2
\end{equation}
Also in \citep{candes2006stable} (equation 12) it is shown that
\begin{equation}
    ||\bX_{U^c}\bh_{U^c}||_2 \leq \sqrt{1+\delta_{k|S|}(\bX)}\sqrt{\frac{|S|}{|T|}}||\bh_{U}||_2 = \sqrt{\frac{1 + \delta_{k|S|}(\bX)}{k}} ||\bh_U||_2
\end{equation}
As such
\begin{eqnarray}
    nP_{\lambda+\rho, \epsilon}(\hat{\btheta}^{\lambda, 0}) - nP_{\lambda+\rho, \epsilon}(\hat{\btheta}^{\lambda + \rho, \epsilon}) \nwl
    \geq \frac{1-\delta_{(1+k)|S|}(\bX) - \frac{1-\delta_{k|S|}(\bX)}{k}}{2}||\bh_U||_2^2 - (1-\frac{1}{\sqrt{k}})\delta\sqrt{n}||\bh_U||_2 - (\lambda + \rho)n\delta
\end{eqnarray}
Noting that $|S| < Kn$. By equation \ref{eq:alphabound},
\begin{equation}
    \alpha_1 = 1 - \delta_{(1+k)|S|}(\bX) - \frac{1+\delta_{k|S|}(\bX)}{k} \geq 1 - \delta_{n(1+k)K}(\bX) - \frac{1 + \delta_{nkK}(\bX)}{k} \geq 1 - \delta_{l}(\bX) - \frac{1 + \delta_{l}(\bX)}{k} \geq \alpha 
\end{equation}
from which we obtain
\begin{equation}
    nP_{\lambda+\rho, \epsilon}(\hat{\btheta}^{\lambda, 0}) - nP_{\lambda+\rho, \epsilon}(\hat{\btheta}^{\lambda + \rho, \epsilon}) \geq \frac{\alpha}{2}||\bh_U||_2^2 - (1-\frac{1}{\sqrt{k}})\delta\sqrt{n}||\bh_{U}||_2 - (\lambda +\delta)n\delta
\end{equation}
combining this with equation X above
\begin{equation}
   \frac{\alpha}{2}||\bh_U||_2^2 - (1-\frac{1}{\sqrt{k}})\delta\sqrt{n}||\bh_{U}||_2 - (\lambda +\rho)n\delta \leq \frac{\delta}{2}||\bh_U||_2^2 + M\delta\sqrt{n}||\bh||_2 + n\delta^2
\end{equation}
We see that
\begin{equation}
    ||\bh||_2^2 \leq \left(1 + \frac{1}{k} \right)||\bh_U||_2^2
\end{equation}
From this we see that
\begin{equation}
    \label{eq:finalalphacomparison}
    \frac{\alpha}{2(1+ \frac{1}{k})}||\bh||_2^2 - \frac{1+\frac{1}{\sqrt{k}}}{\sqrt{1+\frac{1}{k}}}\delta\sqrt{n}||\bh||_2 - (\lambda + \delta)n\delta \leq \frac{\delta}{2}||\bh||_2 + M\delta\sqrt{n}||\bh||_2 + n\delta^2
\end{equation}
Since we know that $k > 1$, we can see that for any choice of $\eta>0$ the value of $\delta$ can be made sufficently small to ensure that equation \ref{eq:finalalphacomparison} implies that the lemma holds. 
\end{proof}

We can now show the universality of the generalization error, we note that for the two cases, the term: 
\begin{eqnarray}
(\btheta^{\lambda, 0} - \btheta^*)^T\bR(\btheta^{\lambda, 0 } - \btheta^*) &= (\btheta^{\lambda, 0} - \btheta^{\lambda + \rho, \epsilon} + \btheta^{\lambda + \rho, \epsilon}  - \btheta^*)^T\bR(\btheta^{\lambda, 0 }- \btheta^{\lambda + \rho, \epsilon} + \btheta^{\lambda + \rho, \epsilon} - \btheta^*) \nwl
&= (\btheta^{\lambda, 0} - \btheta^{\lambda + \rho, \epsilon})^T\bR(\btheta^{\lambda, 0} - \btheta^{\lambda + \rho, \epsilon}) + 2(\btheta^{\lambda, 0} - \btheta^{\lambda + \rho, \epsilon})\bR(\btheta^{\lambda+\rho, \epsilon} - \btheta^*) \nwl&+ (\btheta^{\lambda+\rho, \epsilon} - \btheta^*)^T\bR(\btheta^{\lambda+\rho, \epsilon} - \btheta^*)\nwl
&\leq ||\btheta^{\lambda, 0} - \btheta^{\lambda + \rho, \epsilon}||_2^2||\bR||_2 + 2\|(\btheta^{\lambda, 0} - \btheta^{\lambda + \rho, \epsilon})^T\|_2\|\bR\|_2\|\btheta^{\lambda+\rho, \epsilon} - \btheta^*\|_2 \nwl&+ (\btheta^{\lambda+\rho, \epsilon} - \btheta^*)^T\bR(\btheta^{\lambda+\rho, \epsilon} - \btheta^*)\nwl
&= (\btheta^{\lambda+\rho, \epsilon} - \btheta^*)^T\bR(\btheta^{\lambda+\rho, \epsilon} - \btheta^*)
\end{eqnarray}
Where the final step is by the lemma above showing the asymptotic equivalence of the two norm. By symmetry the argument may be repeated to show that
\begin{eqnarray}
(\btheta^{\lambda+\rho, \epsilon} - \btheta^*)^T\bR(\btheta^{\lambda+\rho, \epsilon} - \btheta^*) \leq (\btheta^{\lambda, 0} - \btheta^*)^T\bR(\btheta^{\lambda, 0 } - \btheta^*)
\end{eqnarray}
In the asymptotic limit. This fact, in conjunction with Theorem \ref{sup:thm:universalityGenerror} proves the universality of the generalization error. Finally we show that the universality of the test functions $h(\btheta)$.
\begin{lemma}
For a function $h$ have that with high probability that
\begin{equation}
   \lim_{n\rightarrow \infty} \left|\frac{h(\hat{\btheta}^{\lambda, 0})}{n} - M_{\lambda|x|,h} \right|  = 0
\end{equation}
in which $M$ is the function given in \eqref{eqn:Mfunctiondef}.
\end{lemma}
\begin{proof}
We denote $\hat{\btheta}^{\lambda, \epsilon}$ as the minimal solution of 
\begin{equation}
    P_{1\lambda, \epsilon} = \frac{1}{2n}||\by - \bX\btheta||_2^2 +\lambda||\btheta||_1 + \frac{\epsilon}{2}||\btheta||_2^2
\end{equation}
Because of the results of the lemma \ref{app:subsec:pf:l1trainingGenErrorUniv} above we know that

\begin{eqnarray}
    \frac{h(\hat{\btheta^{\lambda, 0})}}{n} - M_{\lambda|x|, h} = \frac{\sum_i^n h(\hat{\btheta_i}^{\lambda, 0})}{n} - M_{\lambda|x|, h} = \left(\frac{\sum_i^n h(\hat{\theta}^{\lambda, 0})}{n} - \frac{\sum_i^n h(\hat{\theta}^{\lambda + \rho, \epsilon})}{n} \right) \nwl
    + \left(\frac{\sum_i^n h(\hat{\theta}^{\lambda + \rho, \epsilon})}{n} - M_{(\lambda + \rho)|x| + \epsilon x^2/2,h} \right) + (M_{(\lambda + \rho)|x| + \epsilon x^2/2,g} - M_{\lambda|x|,h})
\end{eqnarray}
Letting $\hat{\btheta}^{\lambda, 0} - \bp = \hat{\btheta}^{\lambda +\rho, \epsilon}$. Then a taylor expansion gives us
\begin{equation}
\frac{\sum_i^n h(\hat{\theta}^{\lambda, 0})}{n} - \frac{\sum_i^n h(\hat{\theta}^{\lambda+\rho, \epsilon})}{n} = \frac{\sum_i^n h'(\hat{\theta}^{\lambda+\rho, \epsilon})p_i + h''(\eta_i)p_i^2/2}{n}
\end{equation}
for some $\bmeta$. Using the Cauchy-Schwartz inequality and using the fact that $h'' < L$ for some value of $L$ we get that
\begin{equation}
\left| \frac{\sum_i^n h(\hat{\theta}^{\lambda, 0})}{n} - \frac{\sum_i^n h(\hat{\theta}^{\lambda+\rho, \epsilon})}{n}\right| \leq \sqrt{\frac{\sum_i^n (h')^2(\hat{\theta}_i^{\lambda+\rho,\epsilon})}{n}}\sqrt{\frac{\sum_i^n p_i^2}{n}} + \frac{L}{2}\frac{\sum_i^n p_i^2}{n}
\end{equation}
As $h'' < L$ we note that $|h'(x)| < L|x| + C$ for some constant $C_2$. As such
\begin{equation}
    \frac{\sum_i^n (h')^2(\hat{\theta}_i^{\lambda+\rho,\epsilon})}{n} \leq 2C^2\frac{\sum_i^n (\hat{\theta}^{\lambda +\rho,\epsilon})^2}{n} + 2C_2^2
\end{equation}
Where $C_2$ is another positive constant.
From theorem \ref{thm:universality:sequence} the term $\frac{\sum_i^n (\hat{\theta}^{\lambda +\rho,\epsilon})^2}{n}$ converges in probability to some value. As such there exists a constant $R > 0$ such that
\begin{equation}
    \mathbb{P}\left(\frac{\sum_i^n (h')^2(\hat{\theta}_i^{\lambda+\rho,\epsilon})}{n} \geq R^2 \right) \rightarrow 0
\end{equation}

For an arbitrary choice of $\delta > 0$. We choose $\eta_1 > 0$ such that $R\sqrt{\eta_1} + c_1\eta_1/2 < \delta/3$. Furthermore we can verify that we can choose an $\eta_2$ such that for every $0<\epsilon<\eta_2, |\rho|<\eta_2$ that
\begin{equation}
    |M_{(\lambda + \rho)|x| + \epsilon x^2/2, h} - M_{\lambda|x|, h}| \leq \frac{\delta}{3}
\end{equation}

letting $\eta = \min(\eta_1, \eta_2)$. Assume that Lemma \ref{lemma:panahi2017Universal10} holds with a proper choice of $\epsilon$ and $\rho$ for this $\eta$. This leads to the following holding true with high probability
\begin{equation}
    \frac{\sum_i^n p_i^2}{n} < \eta \leq \eta_1
\end{equation}
From which we find
\begin{equation}
    \left|\frac{\sum_i^n h(\hat{\theta}^{\lambda, 0})}{n} - \frac{\sum_i^n h(\hat{\theta}^{\lambda + \rho, \epsilon})}{n} \right| \leq R\sqrt{\eta_1} + c_1\eta_1/2 < \frac{\delta}{3}
\end{equation}
Finally we note from theorem \ref{thm:universality:sequence} that 
\begin{equation}
    \mathbb{P}\left(\left|\frac{\sum_i^n h(\hat{\theta}^{\lambda + \rho, \epsilon})}{n} - M_{(\lambda + \rho)|x| + \epsilon x^2/2,h}  \right| > \frac{\delta}{3} \right) \rightarrow 0
\end{equation}
Combining all of the bounds we get with high probability that
\begin{equation}
    \left|\frac{h(\hat{\btheta}^{\lambda, 0})}{n} - M_{\lambda|x|, h}\right|\leq \delta
\end{equation}
Since we can choose delta to be arbitrarily small this leads to the desired results.
\end{proof}

\subsection{Auxiliary lemmas for proving Theorem \ref{thm:l1Universality}}

\begin{lemma} [\cite{panahi2017universal} lemma 9]
    \label{lemma:panahi2017Universal9}
    Consider the function $\rho(\be) = \frac{1}{2}||\bh + \bA\be||_2^2 + \lambda||\be + \btheta^*||_1 + \ba^T\be$ and suppose that it is minimized at $\be^*$. At an arbitary point $\be$ and $q\in \partial \rho(\be)$, then
    \begin{equation}
        ||\be - \be^*||_2 \leq \frac{1}{\sigma^2_{min}(\bA)}||\bq||_2
    \end{equation}
\end{lemma}

\begin{lemma}[\cite{panahi2017universal} lemma 10]
\label{lemma:panahi2017Universal10}
Consider the function $\rho(\be) = \frac{1}{2}||\bh + \bP\be||_2^2 + \lambda||\btheta^* + \be||_1 + \frac{\epsilon}{2}||\be||_2^2$ and suppose that it is minimized at $\be^*$. Let $\be$ be an arbitrary point, then
\begin{equation}
    \rho(\be) - \rho(\be^*) \geq \frac{\sigma_{min}(\bP)}{2}||\be - \be^*||_2^2
\end{equation}
\end{lemma}
\begin{proof}
Let $\bw = \frac{\be - \be^*}{||\be - \be^*||_2}$ and $f(\nu) = \rho(\be^* + \nu\bw)$. Notice that $\rho(\be) = f(||\be - \be^*||_2)$ and $f$ is minimized at $0$. A direct calculation shows that $f$ can be written as $f = \frac{1}{2}\alpha \nu^2 + g(\nu)$, where $g$ is convex and $\alpha = ||\bP\bw||_2^2 + \epsilon/2 \geq \sigma_{min}(\bP)^2$. Then by lemma \ref{lem:convexfunction} this reuslts in
\begin{equation}
    \rho(\be) - \rho(\be^*) = f(||\be - be^*||_2) - f(0) \geq \frac{\alpha}{2}||\be - \be^*||_2^2 \geq \frac{\sigma_{min}(\bP)}{2}||\be - \be^*||_2^2
\end{equation}
\end{proof}
\begin{lemma}[\cite{panahi2017universal} lemma 11]
    \label{lem:convexfunction}
    Suppose $g(\nu)$ is a convex function on $\mathbb{R}$ and $\nu^*$ is a minimum point of the function $f(\nu) = \frac{\alpha}{2}\nu^2 + g(\nu)$. Then for any $\nu\in \mathbb{R}$,
    \begin{equation}
        f(\nu) - f(\nu^*) \geq \frac{\alpha}{2}(\nu -\nu^*)^2
    \end{equation}
\end{lemma}
\begin{proof}
    From the optimality of $\nu^*$, we have that $-\alpha\nu^*\in \partial g(\nu^*)$. Therefore,
    \begin{equation}
        g(\nu) \geq g(\nu^*) - \alpha\nu^*(\nu - \nu^*)
    \end{equation}
    Hence,
    \begin{equation*}
        f(\nu) - f(\nu^*) = \alpha \nu^*(\nu - \nu^*) + \frac{\alpha}{2}(\nu - \nu^*)^2 + g(\nu) - g(\nu^*) \geq \frac{\alpha}{2}(\nu - \nu^*)^2
    \end{equation*}
\end{proof}

\section{Example Case : Elastic Net Regularization}
\label{app:sec:ElasticsNetExampleCase}
We consider the case of Elastic Net Regularization, the case that
\begin{eqnarray}
r(\btheta) = \lambda\|\btheta\|_1 + \frac{\alpha}{2}\|\btheta\|_2^2
\end{eqnarray}
A simple computation gives us that
\begin{eqnarray}
    \label{eqn:hatbthetadef}
    \hat{\btheta}_{i} = (\mathrm{prox}_{\frac{1}{2c_1}r}(\btheta^* - \frac{c_2\sqrt{\gamma}}{2c_1}\bphi))_i = \begin{cases}
    \frac{2c_1\theta^*_i}{2c_1+\alpha} - \frac{c_2\sqrt{\gamma}}{2c_1+ \alpha}\phi_i - \frac{\lambda}{2c_1+\alpha} &  \frac{2c_1\btheta^*_i}{2c_1+\alpha} + \frac{c_2\sqrt{\gamma}}{2c_1+ \alpha}\phi_i > \frac{\lambda}{2c_1+\alpha}\\
    \frac{2c_1\theta^*_i}{2c_1+\alpha} - \frac{c_2\sqrt{\gamma}}{2c_1+ \alpha}\phi_i + \frac{\lambda}{2c_1+\alpha} & \frac{2c_1\btheta^*_i}{2c_1+\alpha} + \frac{c_2\sqrt{\gamma}}{2c_1+ \alpha}\phi_i < -\frac{\lambda}{2c_1+\alpha}\\
    0& || \frac{2c_1\btheta^*_i}{2c_1+\alpha} + \frac{c_2\sqrt{\gamma}}{2c_1+ \alpha}\phi_i|| \leq \frac{\lambda}{2c_1+\alpha}\\
    \end{cases}
\end{eqnarray}
We note that this can equivalently be expressed in the form of a soft thresholding operator
\begin{eqnarray}
\label{eqn:softthresholdingOperatordef}
(\hat{\btheta}_{3})_i = \mathcal{T}_{\frac{\lambda}{2c_1+\alpha}}\left(\frac{2c_1\theta^*_i}{2c_1+\alpha} - \frac{c_2\sqrt{\gamma}}{2c_1+ \alpha}\phi_i \right)
\end{eqnarray}

Substituting the value of the proximal operator in to the Moreau envelope we find that

\begin{eqnarray}
   \frac{1}{m}\sum_i^m\mathcal{M}_{\frac{1}{2c_1}r}(\btheta^* - \frac{c_2\sqrt{\gamma}}{2c_1}\bphi) =\nwl \frac{1}{2c_1 + \alpha}\left(\alpha c_1(\theta^*_i)^2 +2c_1\lambda\theta_i^* - \frac{1}{2}\lambda^2 + \frac{c_2^2\gamma}{4c_1}\phi_i^2 + \alpha c_2\sqrt{\gamma}\phi\theta^*_i + \lambda c_2\sqrt{\gamma}\phi\right)\mathbf{1}_{\lbrace \phi_i < -\zeta_{1i} \rbrace}\nwl
    \frac{1}{2c_1 +\alpha}\left(\alpha c_1(\theta_i^*)^2 - 2c_1\lambda\theta_i^* - \frac{1}{2}\lambda^2 + \frac{\alpha c_2^2\gamma}{4c_1}\phi_i^2 - \alpha c_2\sqrt{\gamma}\phi\theta_i^* +\lambda c_2\sqrt{\gamma}\phi\right)\mathbf{1}_{\lbrace \phi_i > \zeta_{2i} \rbrace}\nwl
    \left(c_1(\theta_i^*)^2 + \frac{c_2^2\gamma}{4c_1}\phi_i^2 -c_2\sqrt{\gamma}\btheta_{i}^*\phi_i\right)\mathbf{1}_{\lbrace \zeta_{1i} \leq \phi_i \leq \zeta_{2i} \rbrace}
\end{eqnarray}
in which $\mathbf{1}_{A}$ is the characteristic function on the set $A$, and 
\begin{eqnarray}
\zeta_{1i} = \frac{(\lambda - 2\hat{c}_1\theta^*_i)}{\sqrt{\gamma}\hat{c}_2} \quad \zeta_{2i} = \frac{(\lambda + 2\hat{c}_1\theta^*_i)}{\sqrt{\gamma}\hat{c}_2}
\end{eqnarray}

Taking the expectation of the envelope with respect to $\phi$ and making use of Steins lemma one can obtain

\begin{eqnarray}
\label{eqn:FullExpectedMoreauEnvelope}
\frac{1}{m}\mathbb{E}\sum_i^m\mathcal{M}_{\frac{1}{2c_1}r}(\btheta^* - \frac{c_2\sqrt{\gamma}}{2c_1}\bphi) =\nwl
\frac{1}{2c_1 + \alpha}\left(\alpha c_1(\theta^*_i)^2  + \frac{\alpha c_2^2\gamma}{4c_1} +2c_1\lambda\theta_i^* - \frac{1}{2}\lambda^2\right)Q(\zeta_{1i}) + 
\frac{1}{\sqrt{2\pi }(2c_1 +\alpha)}\left(\frac{c_2^2\gamma\zeta_{1i}}{4c_1} + \alpha c_2\sqrt{\gamma}\theta^*_i + \lambda c_2\sqrt{\gamma} \right)e^{-\zeta_{1i}^2/2} \nwl
\frac{1}{2c_1 +\alpha}\left(\alpha c_1(\theta_i^*)^2 + \frac{\alpha c_2^2\gamma}{4c_1} - 2c_1\lambda\theta_i^* - \frac{1}{2}\lambda^2 \right)Q(\zeta_{2i})
+\frac{1}{\sqrt{2\pi}(2c_1 + \alpha)}\left(\frac{\alpha c_2^2\gamma\zeta_{2i}}{4c_1} - \alpha c_2\sqrt{\gamma}\theta_i^* +\lambda c_2\sqrt{\gamma}\right)e^{-\zeta_{2i}^2/2}\nwl
\left(c_1(\theta_i^*)^2 + \frac{c_2^2\gamma}{4c_1} \right)(1-Q(\zeta_{1i}) - Q(\zeta_{2i})) - \frac{c_2\sqrt{\gamma}\theta_i^*}{\sqrt{2\pi}}\left(e^{-\zeta_{1i}^2/2} - e^{-\zeta_{2i}^2/2} \right)
+ \frac{c_2^2\gamma}{4c_1\sqrt{2\pi}}\left(-\zeta_{1i}e^{-\zeta_{1i}^2/2} - \zeta_{2i}e^{-\zeta_{2i}^2/2} \right)
\end{eqnarray}
In which $Q(\cdot)$ is the Q-function. Defined to be
\begin{eqnarray}
Q(x) = \frac{1}{\sqrt{2\pi}}\int_{x}^\infty e^{-\frac{u^2}{2}}\mathrm{d}u
\end{eqnarray}
This expression may be implemented in code and simply evaluated for any choice of the parameters.

\subsection{Sparsity}
The effect of the $\ell_1$ regularization term is to promote sparsity in the solution vector. Let $s$ denote the number of elements of $\hat{\btheta}$ that are non-zero. We see that

\begin{eqnarray}
s = \mathbb{E}\sum_{i}^m \mathbf{1}_{\hat{\btheta}_i \neq 0} = \sum_i^m 1 - \mathbb{P}(\hat{\btheta}_i = 0) = \sum_i^m 1 - \mathbb{P}(-\zeta_{1i}\leq \phi_i \leq \zeta_{2i}) 
= \sum_i^m Q(\zeta_{1i}) + Q(\zeta_{2i})
\end{eqnarray}

We further consider the term $\frac{1}{m}\hat{\be}^T\bphi$ and consider what this concentrates on

\begin{eqnarray}
\frac{1}{m}\mathbb{E}[\hat{\be}^T\bphi] = \frac{1}{m}\sum_i^m \mathbb{E}[\hat{e}_i\phi_i] = \frac{1}{m}\sum_i^m\mathbb{E}[\hat{\theta}_i\phi_i] = -\frac{c_2\sqrt{\gamma}}{2c_1 + \alpha} \frac{1}{m}\mathbb{E}\left[\left(\mathcal{T}'_{\frac{\lambda}{2c_1 +\alpha}}\left(\left(\frac{2c_1\theta^*_i}{2c_1+\alpha} + \frac{c_2\sqrt{\gamma}}{2c_1+ \alpha}\phi_i \right)\right)\right)_i\right]
\end{eqnarray}
where in the last inequality we have made use of steins lemma, and $\mathcal{T}'$ is the derivative of the soft thresholding operator defined in \eqref{eqn:softthresholdingOperatordef}. We note that the derivative of the soft thresholding operator is the value of $s$ that we are looking for. In symbols
\begin{eqnarray}
\mathbf{1}_{\lbrace\hat{\btheta}_i \neq 0\rbrace}\mathcal{T} = \mathcal{T}'
\end{eqnarray}
From this we note that
\begin{eqnarray}
\frac{1}{m}\mathbb{E}[\be^T\bphi] = -\frac{c_2\sqrt{\gamma}}{2c_1 + \alpha}\frac{s}{m}
\end{eqnarray}
The value of $\frac{1}{m}\mathbb{E}[\be^T\bphi]$ may also be computed directly from definition of $\hat{\btheta}$ (equation \eqref{eqn:hatbthetadef}). Combining these expressions the value of $s$ may be computed.

\section{Numerical Simulation Detail}\label{App:simulationDetails}

We implement the optimization problem $P_3$ \eqref{eq:4doptproblem} by making use of the explicitly computed Moreau envelope for the case of elastic net \eqref{eqn:FullExpectedMoreauEnvelope}. The optimization is solved using a standard iterative approach in which the inner optimizations are solved at constant values of the outer optimizations. This is repeated iteratively until all parameters are determined. Zeroth order gradient methods were attempted, but were highly dependent on the starting choices of the parameters $\beta, q, \xi, t$, and frequently failed to converge.

The experimental verification was completed using synthetic data, in which the data points $\bz_i$ and the weight matrix $\bW$ was drawn from standard normal distributions. The elastic net optimization was solved using the python package \texttt{cvxpy}. The values of $n$ and $m$ were chosen such that $n+m = 1000$ and that $m/n\approx \gamma$, for a chosen ratio $\gamma$. Each sample was averaged $100$ times to account for the randomness in both the input data $\bz$ and the weights $\bW$.

\subsection{Effective Sparsity}\label{App:simulationDetails:Sparsity}

In this section,  we plot the effective sparsity $s$ for elastic net as a function of the regularization strength $\lambda$ for a number of values of $\gamma = \frac{m}{n}$. Recall that $s$ gives the number of nonzero elements in the solution vector $\hat{\btheta}$. The plots for the ratios $\frac{s}{m}$ and $\frac{s}{n}$ may be seen in figure \ref{fig:sratios}. The $\ell_2^2$ regularization strength was fixed with parameter $\alpha = 0.001$. The solid lines are the theoretical predictions while the dots are determined experimentally. For the experimental values the solution vector $\hat{\btheta}$ was determined using a solver, then each element of the solution vector, it was determined to be ``zero" (i.e. sparse) if its value was less than $\frac{0.01}{\sqrt{m}}$. 

We can see from the figures that for all values of $\gamma$ the sparsity is similar at both large and small values of regularization. As the number of model parameters increases relative to the number of data points, i.e. as $\gamma$ grows larger, the regularization strength required to induce a sparse solution drops. Recalling that true solution was half zeros, the value of regularization strength at which $\frac{s}{m} = 0.5$ matches well with the regularization strength that minimizes the generalization error in figure \ref{subfig:genErrorLam}.

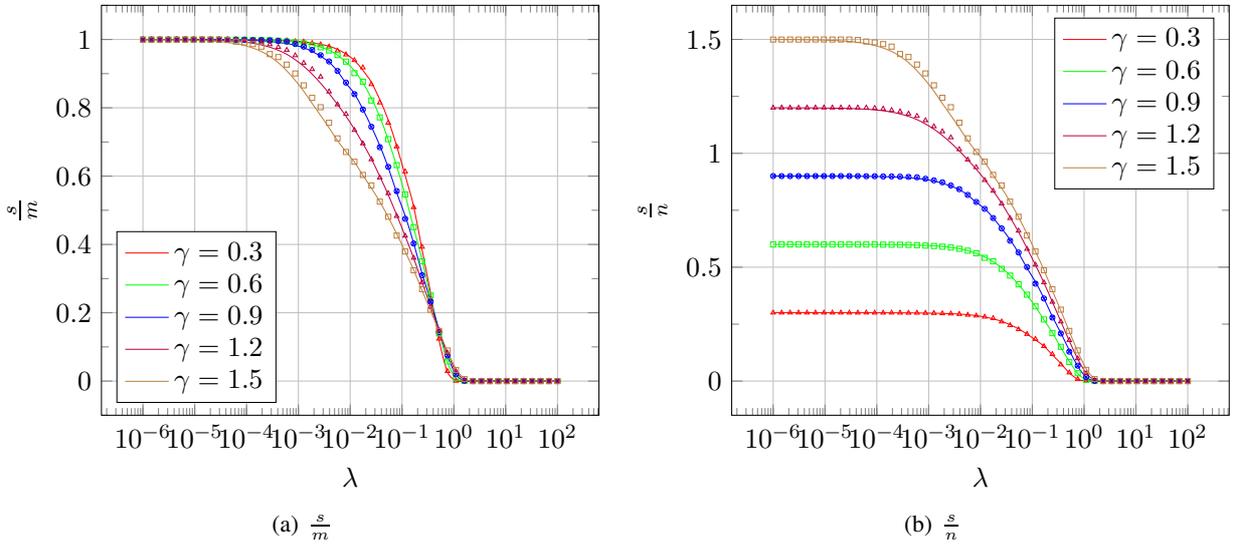
\begin{figure}[!hb]
    \centering
    
    \subfigure [$\frac{s}{m}$] {
    \label{fig:smratio}
    \resizebox{0.45\textwidth}{!}{%
        \begin{tikzpicture}
        \begin{axis}[
          xlabel={$\lambda$},
          ylabel=$\frac{s}{m}$,
          legend pos = south west,
          grid ,
            minor tick num = 1,
            major grid style = {lightgray},
            minor grid style = {lightgray!25},
        xmode=log,
          ]
        \addplot[ smooth, thin, red] table[y=sm0.3, x=lam]{Data/datatheory3.dat};
        \addlegendentry{$\gamma = 0.3$}
        \addplot[ smooth, thin, green] table[y=sm0.6, x=lam]{Data/datatheory3.dat};
        \addlegendentry{$\gamma = 0.6$}
        \addplot[ smooth, thin, blue] table[y=sm0.9, x=lam]{Data/datatheory3.dat};
        \addlegendentry{$\gamma = 0.9$}
        \addplot[ smooth, thin, purple] table[y=sm1.2, x=lam]{Data/datatheory3.dat};
        \addlegendentry{$\gamma = 1.2$}
        \addplot[ smooth, thin, brown] table[y=sm1.5, x=lam]{Data/datatheory3.dat};
        \addlegendentry{$\gamma = 1.5$}
         \addplot[color = red, mark = triangle, mark size = 1pt, only marks] table[ y=sm0.3, x=gamma]{Data/dataTheory4.dat};
        \addplot[color = green, mark = square, mark size = 1pt, only marks] table[ y=sm0.6, x=gamma]{Data/dataTheory4.dat};
        \addplot[blue, mark = otimes, mark size = 1pt, only marks] table[y=sm0.9, x=gamma]{Data/dataTheory4.dat};
        \addplot[purple, mark = triangle, mark size = 1pt, only marks] table[y=sm1.2, x=gamma]{Data/dataTheory4.dat};
         \addplot[color = brown, mark = square, mark size = 1pt, only marks] table[ y=sm1.5, x=gamma]{Data/dataTheory4.dat};
        \end{axis}
    \end{tikzpicture}
        }
    }
\subfigure [$\frac{s}{n}$] {
    \label{fig:snratio}
    \resizebox{0.45\textwidth}{!}{%
        \begin{tikzpicture}
        \begin{axis}[
          xlabel={$\lambda$},
          ylabel=$\frac{s}{n}$,
          legend pos = north east,
          grid ,
            minor tick num = 1,
            major grid style = {lightgray},
            minor grid style = {lightgray!25},
        xmode=log,
          ]
        \addplot[ smooth, thin, red] table[y=sn0.3, x=lam]{Data/datatheory3.dat};
        \addlegendentry{$\gamma = 0.3$}
        \addplot[ smooth, thin, green] table[y=sn0.6, x=lam]{Data/datatheory3.dat};
        \addlegendentry{$\gamma = 0.6$}
        \addplot[ smooth, thin, blue] table[y=sn0.9, x=lam]{Data/datatheory3.dat};
        \addlegendentry{$\gamma = 0.9$}
        \addplot[ smooth, thin, purple] table[y=sn1.2, x=lam]{Data/datatheory3.dat};
        \addlegendentry{$\gamma = 1.2$}
        \addplot[ smooth, thin, brown] table[y=sn1.5, x=lam]{Data/datatheory3.dat};
        \addlegendentry{$\gamma = 1.5$}
         \addplot[color = red, mark = triangle, mark size = 1pt, only marks] table[ y=sn0.3, x=gamma]{Data/dataTheory4.dat};
        \addplot[color = green, mark = square, mark size = 1pt, only marks] table[ y=sn0.6, x=gamma]{Data/dataTheory4.dat};
        \addplot[blue, mark = otimes, mark size = 1pt, only marks] table[y=sn0.9, x=gamma]{Data/dataTheory4.dat};
        \addplot[purple, mark = triangle, mark size = 1pt, only marks] table[y=sn1.2, x=gamma]{Data/dataTheory4.dat};
         \addplot[color = brown, mark = square, mark size = 1pt, only marks] table[ y=sn1.5, x=gamma]{Data/dataTheory4.dat};
        \end{axis}
    \end{tikzpicture}
        } 
        }
       \caption{The effective sparsity $s$ as a ratio to the number of model parameters $m$ or the number of data points $n$ for elastic net regularization for varying strengths of the regularization parameter $\lambda$. The $\ell_2$ regularization term was fixed to $0.001$. Multiple values of $\gamma = \frac{m}{n}$ are considered. Solid line is the theoretical prediction, and the dots are experimental values.}
    \label{fig:sratios}
\end{figure}

\end{document}